\documentclass{article}

\usepackage[T1]{fontenc}
\usepackage{microtype}
\usepackage{graphicx}
\usepackage{subfigure}
\usepackage{amsmath}
\usepackage{amssymb}
\usepackage{booktabs} 
\usepackage{xcolor}

\usepackage{hyperref}

\usepackage[accepted]{icml2020}
\usepackage{macros}
\usepackage{amsmath,amsthm,amssymb, amsfonts, enumerate}
\usepackage{thmtools} 
\usepackage{thm-restate}
\usepackage{ourStyle}
\usepackage{mathtools}
\usepackage{bigints}

\newcommand{\Tr}{\operatorname{Tr}}
\newcommand{\expect}{\mathbb{E}}

\newcommand{\vx}{{\vec{x}}}
\newcommand{\vw}{{\vec{w}}}
\newcommand{\vy}{{\vec{y}}}
\newcommand{\vxi}{{\vec{x}_i}}
\newcommand{\vxj}{{\vec{x}_j}}
\newcommand{\vwi}{{\vec{w}_i}}
\newcommand{\vW}{{\vec{W}}}
\newcommand{\vWi}{{\vec{W}_i}}

\allowdisplaybreaks

\icmltitlerunning{The Impact of Neural Network Overparameterization on Gradient Confusion and Stochastic Gradient Descent}
\begin{document}

\twocolumn[
\icmltitle{The Impact of Neural Network Overparameterization on \\Gradient Confusion and Stochastic Gradient Descent}
\icmlsetsymbol{equal}{*}

\begin{icmlauthorlist}
\icmlauthor{Karthik A. Sankararaman*}{fb,umd}
\icmlauthor{Soham De*}{dm}
\icmlauthor{Zheng Xu}{umd}
\icmlauthor{W. Ronny Huang}{umd}
\icmlauthor{Tom Goldstein}{umd}

\end{icmlauthorlist}

\icmlaffiliation{fb}{Facebook.}
\icmlaffiliation{dm}{DeepMind, London.}
\icmlaffiliation{umd}{University of Maryland, College Park.}

\icmlcorrespondingauthor{Karthik A. Sankararaman}{karthikabinavs@gmail.com}
\icmlcorrespondingauthor{Soham De}{sohamde@google.com}

\icmlkeywords{optimization, initialization}

\vskip 0.3in
]
\printAffiliationsAndNotice{\icmlEqualContribution}

\begin{abstract}
This paper studies how neural network architecture affects the speed of training. 
We introduce a simple concept called {\em gradient confusion} to help formally analyze this.  When gradient confusion is high, stochastic gradients produced by different data samples may be negatively correlated, slowing down convergence.
But when gradient confusion is low, data samples interact harmoniously, and training proceeds quickly.
Through theoretical and experimental results, we demonstrate how the neural network architecture affects gradient confusion, and thus the efficiency of training.
Our results show that, for popular initialization techniques, increasing the \emph{width} of neural networks leads to \emph{lower} gradient confusion, and thus faster model training. On the other hand, increasing the \emph{depth} of neural networks has the opposite effect. Our results indicate that alternate initialization techniques or networks using both batch normalization and skip connections help reduce the training burden of very deep networks.




\end{abstract}

\section{Introduction}

Stochastic gradient descent (SGD) \citep{robbins1951stochastic} and its variants with momentum have become the standard optimization routine for neural networks due to their fast convergence and good generalization properties \citep{wilson2017marginal, sutskever2013importance, smith2020generalization}. 
%
Yet the convergence behavior of SGD on neural networks still eludes full theoretical understanding. 
Furthermore, it is not well understood how design choices on neural network architecture affect training performance. In this paper, we make progress on these open questions.


Classical stochastic optimization theory predicts that the learning rate of SGD needs to decrease over time for convergence to be guaranteed to the minimizer of a convex function \citep{shamir2013stochastic, bertsekas2011incremental}.  
For strongly convex functions for example, such results show that a decreasing learning rate schedule of $O(1/k)$ is required to guarantee convergence to within $\epsilon$-accuracy of the minimizer in $O(1/\epsilon)$ iterations, where $k$ denotes the iteration number. Such decay schemes, however, typically lead to poor performance on standard neural network problems.
%

Neural networks operate in a regime where the number of parameters is much larger than the number of training data.  In this ``over-parameterized'' regime, SGD seems to converge quickly with constant learning rates. Most neural network practitioners use a constant learning rate for the majority of training (with exponential decay only towards the end of training) without seeing the method stall \citep{krizhevsky2012imagenet, simonyan2014very, he2016deep, zagoruyko2016wide}.
With constant learning rates, theoretical guarantees show that SGD converges quickly to a neighborhood of the minimizer,
 but then reaches a \emph{noise floor} beyond which it stops converging; this noise floor depends on the learning rate and the variance of the gradients \citep{moulines2011non, needell2014stochastic}. 
Recent results show that convergence without a noise floor is possible without decaying the learning rate, provided the model is strongly convex and overfitting occurs \citep{schmidt2013fast, ma2017power, vaswani2018fast}.

While these results do give important insights,
 they do not fully explain the dynamics of SGD on neural networks, and how they relate to over-parameterization.
%
%
%
%
%
Furthermore, training performance is strongly influenced by network architecture. It is common knowledge among practitioners that, under standard Gaussian initialization techniques \citep{glorot2010understanding, he2015delving}, deeper networks train slower \citep{bengio1994learning,saxe2013exact}. This has led to several innovations over the years to get deeper nets to train more easily, such as careful initialization strategies \citep{xiao2018dynamical}, residual connections \citep{he2016deep}, and normalization schemes like batch normalization \citep{ioffe2015batch}. Furthermore, there is evidence to indicate that wider networks are faster to train \citep{zagoruyko2016wide, nguyen2017loss}, and recent theoretical results suggest that the dynamics of SGD simplify considerably for very wide networks \citep{jacot2018neural, lee2019wide}. 
In this paper, we make progress on theoretically understanding these empirical observations and unifying existing theoretical results. To this end, we identify and analyze a condition that enables us to establish direct relationships between layer width, network depth, problem dimensionality, initialization schemes, and trainability and SGD dynamics for over-parameterized networks.
 
\textbf{Our contributions.} 
Typical neural networks are \emph{over-parameterized} (\emph{i.e.,} the number of parameters exceed the number of training points). In this paper, we ask how this over-parameterization, and more specifically the network architecture, affects the trainability of neural networks and the dynamics of SGD. Through extensive theoretical and experimental studies, we show how layer width, network depth, initialization schemes, and other architecture choices affect the dynamics. 
The following are our main contributions.\footnote{To keep the main text of the paper concise, all proofs and several additional experimental results are delegated to the appendix.}
%
%
	\begin{itemize}
	\item We identify a condition, termed \emph{gradient confusion}, that impacts the convergence properties of SGD on over-parameterized models. We prove that high gradient confusion may lead to slower convergence, while convergence is accelerated (and could be faster than predicted by existing theory) if confusion is low, indicating a regime where constant learning rates work well in practice (sections \ref{sec:prelim} and \ref{sec:convergence}). We use the gradient confusion condition to study the effect of various architecture choices on trainability and convergence.
	\item We study the effect of neural network architecture on gradient confusion at standard Gaussian initialization schemes (section \ref{sec:initialization}), and prove (a) gradient confusion increases as the network depth increases, and (b) wider networks have lower gradient confusion. These indicate that deeper networks are more difficult to train and wider networks can improve trainability of networks. Directly analyzing the gradient confusion bound enables us to derive results on the effect of depth and width, without requiring restrictive assumptions like large layer widths \citep{du2018gradient, allen2018convergence}. Our results hold for a large class of neural networks with different non-linear activations and loss-functions. In section \ref{sec:general_depth}, we present a more general result on the effect of depth on the trainability of networks without assuming the network is at initialization.
	
	\item We prove that for linear neural networks, gradient confusion is \emph{independent of depth} when using orthogonal initialization schemes (section \ref{sec:orthInit}) \citep{saxe2013exact, schoenholz2016deep}. This indicates a way forward in developing techniques for training deeper models.
	
	\item We test our theoretical predictions using extensive experiments on wide residual networks (WRNs) \citep{zagoruyko2016wide}, convolutional networks (CNNs) and multi-layer perceptrons (MLPs) for image classification tasks on CIFAR-10, CIFAR-100 and MNIST (section \ref{sec:experiments} and appendix~\ref{app:extra_results}). We find that our theoretical results consistently hold across all our experiments. We further show that the combination of batch normalization and skip connections in residual networks help lower gradient confusion, thus indicating why SGD can efficiently train deep neural networks that employ such techniques.
	\end{itemize}

\section{Gradient confusion}
\label{sec:prelim}

\xhdr{Notations.} We denote vectors in bold lower-case and matrices  in bold upper-case. We use $(\vec{W})_{i, j}$ to indicate the $(i, j)$ cell in matrix $\vec{W}$ and $(\vec{W})_i$ for the $i^{\text{th}}$ row of matrix $\vec{W}$. $\|\vec{W}\|$ denotes the operator norm of $\vec{W}$. $[N]$ denotes $\{1, 2, \ldots, N\}$ and $[N]_{0}$ denotes  $\{0, 1, \ldots, N\}$.

	\textbf{Preliminaries.} Given $N$ training points (specified by the corresponding loss functions $\{f_i\}_{i \in [N]}$), we use SGD to solve empirical risk minimization problems of the form,
		\begin{align}
		\textstyle
		\min_{\vec{w} \in \mathbb{R}^d} F(\vec{w}) := \min_{\vec{w} \in \mathbb{R}^d} \frac{1}{N} \sum_{i=1}^N f_i(\vec{w}), \label{eq:obj_fn}
		\end{align}
using the following iterative update rule for $T$ rounds: 		
\begin{align}
\vec{w}_{k+1} = \vec{w}_k - \alpha \nabla \tilde f_k (\vec{w}_k). \label{eq:sgd}
\end{align}
Here $\alpha$ is the learning rate and $\tilde f_k$ is a function chosen uniformly at random from $\{f_i\}_{i \in [N]}$ at iteration $k \in [T]$. 
$\vec{w}\opt = \arg\min_\vec{w} F(\vec{w})$ denotes the optimal solution.

\textbf{Gradient confusion.}  SGD works by iteratively selecting a random function $\tilde f_k$, and modifying the parameters to move in the direction of the negative gradient of $ \tilde f_k$. 
%
  It may happen that the selected gradient $\nabla \tilde f_k$ is negatively correlated with the gradient of another term $\nabla f_j.$ 
  When the gradients of different mini-batches are negatively correlated, the objective terms disagree on which direction the parameters should move, and we say that there is {\em gradient confusion}.\footnote{Gradient confusion is related to both gradient variance and gradient diversity \citep{yin2017gradient}, but with important differences, which we discuss in section \ref{sec:related}. We also discuss alternate definitions of the gradient confusion condition in section \ref{sec:altdef}.}
 %
%
%
\begin{definition}
\label{ass:cosine}
A set of objective functions $\{f_i\}_{i \in [N]}$ has gradient confusion bound $\eta \ge 0$ if the pair-wise inner products between gradients satisfy, for a fixed $\vec{w} \in \mathbb{R}^d$,
\begin{align}  \label{confusion}
\langle \nabla f_i (\vec{w}) , \nabla f_j (\vec{w}) \rangle  \ge - \eta,  \, \, \forall i \neq j \in [N].
\end{align}
\end{definition} 

%
\begin{figure}[t]
\centering
\includegraphics[width=0.45\textwidth]{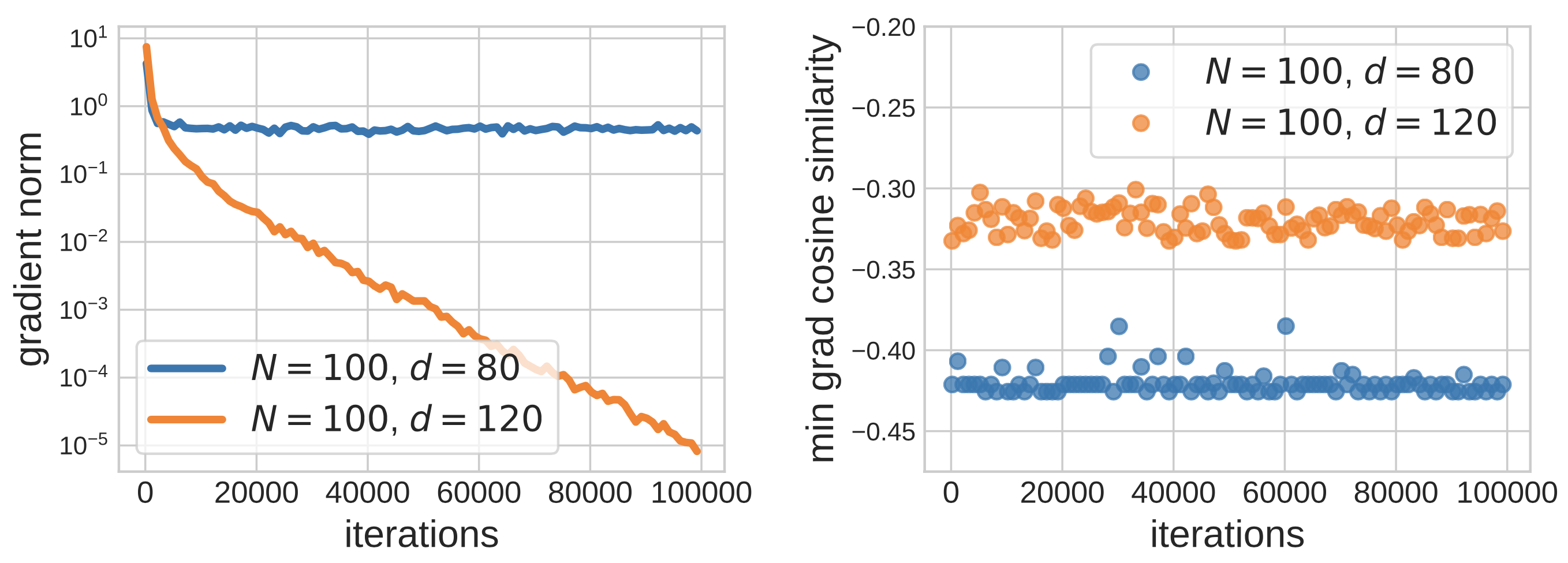}
\caption{Linear regression on an over-parameterized ($d = 120$) and under-parameterized ($d = 80$) model with $N = 100$ samples generated randomly from a Gaussian, trained using SGD with minibatch size 1. Plots are averaged over 3 independent runs. Gradient cosine similarities were calculated over all pairs of gradients. 
}
\label{fig:linreg_simulation}
\end{figure}
\textbf{Observations in simplified settings.}
SGD converges fast when gradient confusion is low along its path.  To see why, consider the case of training a logistic regression model on a dataset with \emph{orthogonal} vectors. We have $f_i(\vec{w})=\mathcal{L}(y_i \vec{x}_i^\top \vec{w}),$ where $\mathcal{L}:\reals\to\reals$ is the logistic loss, $\{\vec{x}_i\}_{i \in [N]}$ is a set of orthogonal training vectors, and $y_i\in\{-1,1\}$ is the label for $\vec{x}_i$. We then have $\nabla f_i(\vec{w}) = \zeta_i \vec{x}_i,$ where $\zeta_i=y_i  \mathcal{L}'(y_i\cdot  \vec{x}_i^\top  \vec{w}).$  Note that the gradient confusion is 0 since $\la \nabla f_i(\vec{w}), \nabla f_j(\vec{w})  \ra = \zeta_i\zeta_j\la  \vec{x}_i , \vec{x}_j  \ra = 0$, $\forall i, j \in [N]$ and $i \neq j$.  Thus, an update in the gradient direction $f_i$ has {\em no} effect on the loss value of $f_j$ for $i\neq j$.  In this case, SGD decouples into (deterministic) gradient descent on each objective term separately, and we can expect to see the fast convergence rates attained by gradient descent. 

Can we expect a problem to have low gradient confusion in practice?  
From the logistic regression problem, we have:
$
| \la \nabla f_i(\vec{w}) , \nabla f_j(\vec{w}) \ra | =  |\la \vec{x}_i, \vec{x}_j \ra|  \cdot | \zeta_i\zeta_j |.
$
This inner product is expected to be small for all $\vec{w}$; the logistic loss satisfies $|\zeta_i \zeta_j| <1$,  and for fixed $N$ the quantity $\max_{ij} |\la \vec{x}_j, \vec{x}_i\ra| $  is $O(1/\sqrt{d})$ whenever $\{\vec{x}_i\}$ are randomly sampled from a sphere (see lemma \ref{lem:orthovec} for the formal statement).\footnote{Generally, this is true whenever $\vec{x}_i = \frac{1}{\sqrt{d}} \vec{y}_i,$ where $\vec{y}_i$ is an isotropic random vector \citep{vershynin2016high}.}
Thus, we would expect a random linear model to have nearly orthogonal gradients, when the number of parameters is "large" and the number of training data is "small", i.e., when the model is over-parameterized. This is further evidenced by a toy example in figure \ref{fig:linreg_simulation}, where we show a slightly over-parameterized linear regression model can have much faster convergence rates, as well as lower gradient confusion.
One can prove a similar result for problems that have random and low-rank Hessians, which suggests that one might expect gradient to be small near the minimizer for many standard neural nets (see appendix \ref{app:hessians} for more discussion).

The above arguments are a bit simplistic, considering toy scenarios and ignoring issues like the effect of network structure. In the following sections, we rigorously analyze the effect of gradient confusion on the speed of convergence on non-convex problems, and the effect of width and depth of the neural network architecture on the gradient confusion.



%

\section{SGD is fast when gradient confusion is low}
\label{sec:convergence}

 Several prior papers have analyzed the convergence rates of constant learning rate SGD \citep{nedic2001convergence, moulines2011non, needell2014stochastic}. These results show that for strongly convex and Lipschitz smooth functions, SGD with a constant learning rate $\alpha$ converges \emph{linearly} to a neighborhood of the minimizer. The noise floor it converges to depends on the learning rate $\alpha$ and the variance of the gradients at the minimizer, i.e., $\expect_i \| \nabla f_i(\vw\opt) \|^2$. To guarantee convergence to $\epsilon$-accuracy in such a setting, the learning rate needs to be small, i.e., $\alpha = O(\epsilon)$, and the method requires $T = O(1/\epsilon)$ iterations. Some more recent results show convergence of constant learning rate SGD without a noise floor and without small step sizes for models that can completely fit the data \citep{schmidt2013fast, ma2017power, vaswani2018fast}. 

Gradient confusion is related to these results. Cauchy-Schwarz inequality implies that if $\expect_i \| \nabla f_i(\vw\opt) \|^2 = O(\epsilon)$, then $\expect_{i, j} |\la \nabla f_i(\vw\opt), \nabla f_j(\vw\opt)\ra| = O(\epsilon)$, $\forall i, j$. Thus the gradient confusion at the minimizer is small when the variance of the gradients at the minimizer is small. Further note that when the variance of the gradients at the minimizer is $O(\epsilon)$, a direct application of the results in \citet{moulines2011non} and \citet{needell2014stochastic} shows that constant learning rate SGD has fast convergence to $\epsilon$-accuracy in $T=O(\log(1/\epsilon))$ iterations, without the learning rate needing to be small.
Generally however, bounded gradient confusion does not provide a bound on the variance of the gradients (see section \ref{sec:related}). Thus, it is instructive to derive convergence bounds of SGD explicitly in terms of the gradient confusion to properly understand its effect.

We first consider functions satisfying the Polyak-Lojasiewicz (PL) inequality \citep{lojasiewicz1965ensembles}, a condition related to, but weaker than, strong convexity, and used in recent work \cite{karimi2016linear, de2017automated}. We provide bounds on the rate of convergence in terms of the optimality gap.  
We start with two standard assumptions. 
\begin{itemize}
\item[\textbf{(A1)}]
$\{f_i\}_{i \in [N]}$ are \emph{Lipschitz smooth}:\\ $f_i(\vec{w}') \leq f_i(\vec{w}) + \nabla f_i(\vec{w})^\top(\vec{w}'-\vec{w}) + \frac{L}{2} \|\vec{w}'-\vec{w}\|^2.$ 
\item[\textbf{(A2)}]
$\{f_i\}_{i \in [N]}$ satisfy the \emph{PL inequality}:\\
$ \frac{1}{2} \| \nabla f_i (\vec{w}) \|^2 \geq \mu (f_i (\vec{w}) - f_i\opt),$
 $f_i\opt = \min_\vec{w} f_i(\vec{w})$. 
\end{itemize}
%
%
%
We now state a convergence result of constant learning rate SGD in terms of the gradient confusion. 
\begin{restatable}{theorem}{linearConvergence}
\label{thm:cosine}
If the objective function satisfies (A1) 
and (A2), 
 and has gradient confusion $\eta$,
 SGD converges linearly to a neighborhood of the minima of problem \eqref{eq:obj_fn} as: 
 $$\textstyle \expect[F(\vec{w}_T) - F\opt] \le \rho^T (F(\vec{w}_0) - F\opt) + \frac{\alpha \eta}{1 - \rho},$$
  where $\alpha < \frac{2}{NL}$, $\rho = 1 - \frac{2\mu}{N} \big(\alpha - \frac{NL\alpha^2}{2} \big)$, $F\opt = \min_{\vw} F(\vw)$ and $\vec{w}_0$ is the initialized weights.
\end{restatable}
%
%
%
This result shows that SGD converges {\em linearly} to a neighborhood of a minimizer, and the size of this neighborhood depends on the level of gradient confusion.  
When the gradient confusion is small, i.e., $\eta = O(\epsilon)$, SGD has fast convergence to $O(\epsilon)$-accuracy in $T=O(\log(1/\epsilon))$ iterations, without requiring the learning rate to be vanishingly small. We now extend this to general smooth functions. 

 \begin{restatable}{theorem}{boundedVariance}
\label{thm:nonconvex}
If the objective satisfies (A1)
and has gradient confusion $\eta$, then SGD converges to a neighborhood of a stationary point of problem \eqref{eq:obj_fn} as: 
$$\textstyle \min_{k=1, \dots, T} \expect \| \nabla F(\vec{w}_k)  \|^2  \le   \frac{ \rho(F(\vec{w}_1) - F\opt) }{T} + \rho \eta,$$
 for $\alpha < \frac{2}{NL}$, $\rho = \frac{2N}{2 - NL\alpha }$, and $F\opt = \min_{\vw} F(\vw)$.
\end{restatable}


Thus, as long as $\eta = O(1/T)$, SGD has fast $O(1/T)$ convergence on smooth non-convex functions.
Theorems \ref{thm:cosine} and \ref{thm:nonconvex} predict an initial phase of optimization with fast convergence to the neighborhood of a minimizer or a stationary point.  This behavior is often observed when optimizing neural nets \citep{darken1992towards, sutskever2013importance},
where a constant learning rate reaches a high level of accuracy on the model. As we show in subsequent sections, this is expected since for neural networks typically used, the gradient confusion is expected to be low.
See section \ref{sec:related} for more discussion on the above results and how they relate to previous work. 
%
%
We stress that our goal is not to study convergence rates per se, nor is it to prove state-of-the-art rate bounds for this class of problems. Rather, we show the direct effect that the gradient confusion bound has on the convergence rate and the noise floor for constant learning rate SGD. As we show in the following sections, this new perspective in terms of the gradient confusion helps us more directly understand how neural network architecture design affects SGD dynamics and why.

\section{Effect of neural network architecture at Gaussian initializations}
\label{sec:initialization}
%
%
To draw a connection between neural network architecture and training performance, we analyze gradient confusion for generic (i.e., random) model problems using methods from high-dimensional probability. In this section, we analyze the effect of neural network architecture at the beginning of training, when using standard Gaussian initialization techniques. Analyzing these models at initialization is important to understand which architectures are more easily trainable than others. Our results cover a wide range of scenarios compared to prior work, require minimal additional assumptions, and hold for a large family of neural networks with different non-linear activation functions and loss-functions. In particular, our results hold for fully connected networks (and can be extended to convolutional networks) with the square-loss and logistic-loss functions, and commonly used non-linear activations such as sigmoid, tanh and ReLU. We consider both the case where the input data is arbitrary but bounded (theorem~\ref{thm:fixedData}, part 1), as well as where the input data is randomly drawn from the surface of a unit sphere (theorem~\ref{thm:fixedData}, part 2).

\xhdr{Setting.} We consider training data $\mathcal{D} = \{ (\vec{x}_i, \cC(\vec{x}_i))\}_{i \in [N]},$ with labeling function $\cC: \mathbb{R}^d \rightarrow [-1, 1]$. For some of our results, we consider that the data points $\{\vec{x}_i\}$ are drawn uniformly at random from the surface of a $d$-dimensional unit sphere. The labeling function satisfies $|\cC(\vec{x})| \le 1$  and $\|\nabla_{\vec{x}} \cC(\vec{x})\|_2 \le 1$ for $\|\vec x\| \le 1.$ Note that this automatically holds for every model considered in this paper where the labeling function is \emph{realizable} (i.e., where the model can express the labeling function using its parameters). More generally, this assumes a Lipschitz condition on the labels (i.e., the labels don't change too quickly with the inputs). 

We consider two loss-functions: square-loss for regression and logistic loss for classification. The square-loss function is defined as $f_i(\vec{w}) = \frac{1}{2} (\cC(\vec{x}_i) - g_{\vec w}(\vec{x}_i))^2$ and the logistic function is defined as $f_i(\vec{w}) = \log(1+\exp(-\cC(\vec{x}_i)g_{\vec w}(\vec{x}_i)))$. 
Here, $g_{\vec w}: \mathbb{R}^d \rightarrow  \mathbb{R}$ denotes the parameterized function we fit to the training data and $f_i(\vec{w})$ denotes the loss-function of hypothesis $g_{\vec w}$ on data point $\vec{x}_i$.
	


	%
	

Let $\vec{W}_0 \in \mathbb{R}^{\ell_1 \times d}$ and $\{ \vec{W}_p \}_{p \in [\beta]}$ where $\vec{W}_p \in \mathbb{R}^{\ell_p \times \ell_{p-1}}$ are weight matrices. Let $\vec{W}$ denote the tuple $(\vec{W}_p)_{p \in [\beta]_0}$. Define $\ell := \max_{p \in [\beta]} \ell_p$ to be the \emph{width} and $\beta$ to be the \emph{depth} of the network. Then, the model $g_{\vec{W}}$ is defined as 
		\begin{equation}
			\label{eq:NNmodel}
			g_{\vec{W}}(\vec{x}) := \sigma(\vec{W}_{\beta} \sigma(\vec{W}_{\beta-1} \ldots \sigma(\vec{W}_1 \sigma(\vec{W}_0 \vec{x}))\ldots )), \nonumber
		\end{equation}
		where $\sigma$ denotes the non-linear activation function applied point-wise to its arguments. We assume that the activation is given by a function $\sigma(x)$ with the following properties. 
		\begin{itemize}
			\item \textbf{(P1) Boundedness:} $|\sigma(x)| \leq 1$ for $x \in [-1, 1]$.
			\item \textbf{(P2) Bounded differentials:}  Let $\sigma'(x)$ and $\sigma''(x)$ denote the first and second sub-differentials respectively. Then, $|\sigma'(x)| \leq 1$ and $ |\sigma''(x)| \leq 1$ for all $x \in [-1, 1]$.
		\end{itemize}
		When $\|\vec{x}\| \leq 1$, activation functions such as \emph{sigmoid}, \emph{tanh}, \emph{softmax} and \emph{ReLU} satisfy these requirements.

Furthermore, in this section, we consider the following Gaussian weight initialization strategy. 
	\begin{strategy}\label{strat:weights}
		 $\vec{W}_0 \in \mathbb{R}^{\ell \times d}$ has independent $\mathcal{N}(0, \frac{1}{d})$ entries. For every $p \in [\beta]$, the weights $\vec{W}_p \in \mathbb{R}^{\ell_p \times \ell_{p-1}}$ have independent $\mathcal{N}\left( 0, \frac{1}{\kappa \ell_{p-1}} \right)$ entries for some constant $\kappa > 0$. 
	\end{strategy}
 This initialization strategy with different settings of $\kappa$ are used almost universally for neural networks \citep{glorot2010understanding,lecun2012efficient,he2015delving}. For instance, typically $\kappa = \frac{1}{2}$ when ReLU activations are used, and $\kappa = 1$ when tanh activations are used.

\xhdr{Main result.}
The following theorem shows how the width $\ell := \max_{p \in [\beta]} \ell_p$ and the depth $\beta$ affect the gradient confusion condition at standard initializations. We show that \emph{as width increases (for fixed depth) or depth decreases (for fixed width) the probability that the gradient confusion bound (equation \ref{confusion}) holds increases}. Thus, as the depth increases (with fixed width), training a model becomes harder, while as the width increases (with fixed depth), training a model becomes easier. Furthermore, note that this result also implies that training very deep \emph{linear} neural networks (with identity activation functions) with standard Gaussian initializations is hard. Throughout the paper, we define the parameter $\zeta_0:= 2 \sqrt{\beta}$. See the appendix (Lemma~\ref{lem:lossPropNeural}) for a more careful definition of this quantity. 
			\begin{restatable}{theorem}{NeuralNetsFixedData}
				\label{thm:fixedData}
				Let $\vec{W}_0, \vec{W}_1, \ldots, \vec{W}_{\beta}$ be weight matrices chosen according to strategy~\ref{strat:weights}. There exists fixed constants $c_1, c_2 >0$ such that we have the following.
				\begin{enumerate}
					\item Consider a fixed but arbitrary dataset $\vec{x}_1, \vec{x}_2, \ldots, \vec{x}_N$ with $\| \vec{x}_i \| \leq 1$ for every $i \in [N]$. For $\eta > 4$, the gradient confusion bound in equation \ref{confusion} holds with probability at least $$\textstyle 1- \beta \exp\left( -c_1 \kappa^2 \ell^2 \right) - N^2 \exp\left( \frac{-c \ell^2 \beta (\eta-4)^2}{64 \zeta_0^4 (\beta+2)^4} \right).$$
					\item If the dataset $\{ \vec{x}_i \} _{i \in [N]}$ is such that each $\vec{x}_i$ is an i.i.d. sample from the surface of $d$-dimensional unit sphere, then for every $\eta > 0$ the gradient confusion bound in equation \ref{confusion} holds with probability at least $$\textstyle 1- \beta \exp\left( -c_1 \kappa^2 \ell^2 \right) - N^2 \exp\left( \frac{-c_2 (\ell d + \ell^2 \beta) \eta^2}{16 \zeta_0^4 (\beta+2)^4} \right).$$
				\end{enumerate}
			\end{restatable}

Theorem \ref{thm:fixedData} shows that under popular Gaussian initializations used, training becomes harder as networks get deeper. The result however also shows a way forward: layer width improves the trainability of deep networks. 
Other related work supports this showing that when the layers are infinitely wide, the learning dynamics of gradient descent simplifies considerably \citep{jacot2018neural, lee2019wide}. \citet{hanin2018start} also suggest that the width should increase linearly with depth in a neural network to help dynamics at the beginning of training. In section \ref{sec:experiments} and appendix \ref{app:extra_results}, we show substantial empirical evidence that, given a sufficiently deep network, increasing the layer width often helps in lowering gradient confusion and speeding up convergence for a range of models.

\section{A more general result on the effect of depth}
\label{sec:general_depth}

%
%
While our results in section \ref{sec:initialization} hold at standard initialization schemes, in this section we derive a more general version of the result. In particular, we assume the setting where the data is drawn uniformly at random from a unit sphere and the weights lie in a ball around a local minimizer. Our results hold for both fully connected networks and convolutional networks with the square-loss and logistic-loss functions, and commonly-used non-linear activations such as sigmoid, tanh, softmax and ReLU.



We consider the same setup as in the previous section, and assume additionally that the data points $\{\vec{x}_i\}$ are drawn uniformly from the surface of a $d$-dimensional unit sphere. Additionally, instead of studying the network at initialization, we make the following assumption on the weights. 

			\begin{assumption}[Small Weights]
				\label{ass:small_weight}
					We assume that the operator norm of the weight matrices $\{\vec{W}_i\}_{i \in [\beta]_{0}}$ are bounded above by $1$, i.e., for every $i \in [\beta]_{0}$ we have $\|\vec{W}_i\| \leq 1$.
				\end{assumption}

%
   		
The operator norm of the weight matrices $\| \vW \|$ being close to 1 is important for the trainability of neural networks, as it ensures that the input signal is passed through the network without exploding or shrinking across layers \citep{glorot2010understanding}. Proving non-vacuous bounds in case of such blow-ups in magnitude of the signal or the gradient is not possible in general, and thus, we consider this restricted class of weights. Most standard neural networks are trained using \emph{weight decay} regularizers of the form $\sum_i \|W_i\|_F^2$. This biases the weights to be small when training neural networks in practice. See appendix \ref{app:small_weights} for further discussion on the small weights assumption.

	 We now present a more general version of theorem \ref{thm:fixedData}. 
	\begin{restatable}{theorem}{NeuralNet}
    	\label{thm:arbitraryNN}
    	Let $\vec{W}_0, \vec{W}_1, \ldots, \vec{W}_{\beta}$ satisfy assumption~\ref{ass:small_weight}. For some fixed constant $c >0$, the gradient confusion bound (equation \ref{confusion}) holds with probability at least 
	$$\textstyle 1- N^2 \exp\left( \frac{-c d \eta^2}{16 \zeta_0^4 (\beta+2)^4} \right).$$
    	\end{restatable}
	Theorem \ref{thm:arbitraryNN} shows that (for fixed dimension $d$ and number of samples $N$) when the  depth $\beta$ decreases, the probability that the gradient confusion bound in equation \ref{confusion} holds increases, and vice versa. Thus, our results indicate that in the general case when the weights are small, increasing the network depth will typically lead to slower model training.
	
	Note that on assuming $\|\vW \| \le 1$ for each weight matrix $\vW$, the dependence of gradient confusion on the width goes away. To see why this, consider an example where each weight matrix in the neural network has exactly one non-zero element, which is set to 1. The operator norm of each such weight matrix is 1, but the forward or backward propagated signals would not depend on the width.
	
	Note that the convergence rate results of SGD in section \ref{sec:convergence} 
	assume that the gradient confusion bound holds at every point along the path of SGD. On the other hand, theorem \ref{thm:arbitraryNN} shows concentration bounds for the gradient confusion at a fixed weight $\vW$. Thus, to make the above result more relevant for the convergence of SGD on neural networks, we now make the concentration bound in theorem \ref{thm:arbitraryNN} \emph{uniform over all weights inside a ball $\mathcal{B}_r$ of radius $r$}. 
    	
    	\begin{restatable}{corollary}{NeuralNetUniform}
    		\label{thm:uniformNN}	
    	Select a point $\vec{W} = (\vec{W}_0, \vec{W}_1, \ldots, \vec{W}_\beta)$, satisfying assumption \ref{ass:small_weight}. Consider a ball $\mathcal{B}_r$ centered at $\vec{W}$ of radius $r > 0$. If the data $\{\vxi\}_{i \in [N]}$ are sampled uniformly from a unit sphere, then the gradient confusion bound in equation \ref{confusion} holds uniformly at all points $\vec{W}' \in \mathcal{B}_r$ with probability at least
    	\begin{align*}
			& \textstyle 1- N^2 \exp\left(-\frac{cd\eta^2}{64\zeta_0^4 (\beta+2)^4} \right), \hspace{19mm}\text{ if } r\le \eta/4\zeta_0^2, \\ 
&  \textstyle 1- N^2 \exp\left(-\frac{cd\eta^2}{64\zeta_0^4 (\beta+2)^4} + \frac{8d\zeta_0^2 r}{\eta}\right), \quad\quad\text{ otherwise}.
\end{align*}
    \end{restatable}

Corollary \ref{thm:uniformNN} shows that the probability that the gradient confusion bound holds decreases with increasing depth, for all weights in a ball around the minimizer.\footnote{The above results automatically hold for convolutional networks, since a convolution operation on $\vx$ can be represented as a matrix multiplication $\vec{U}  \vec{x}$ for an appropriate Toeplitz matrix $\vec{U}$.}
This explains why, in the general case, training very deep models might always be hard. 
This raises the question why most deep neural networks used in practice are so efficiently trained using SGD. While careful Gaussian initialization strategies prevent vanishing or exploding gradients, these strategies still suffer from high gradient confusion for very deep networks unless the width is also increased with the depth, as we show in section \ref{sec:initialization}. Practitioners over the years, however, have achieved state-of-the-art results by making networks deeper, without necessarily making networks wider.
Thus, in section \ref{sec:experiments}, we empirically study how popular techniques used in these models like skip connections and batch normalization affect gradient confusion. We find that these techniques drastically lower gradient confusion, making deep networks significantly easier to train.
Furthermore, in the next section, we show how deep linear nets are trainable when used with orthogonal initialization techniques, indicating a way forward for training deeper models.

\section{Gradient confusion is independent of depth for orthogonal initializations}
    \label{sec:orthInit}
    In this section, we show that for deep linear neural networks, gradient confusion is independent of depth when the weight matrices are initialized as orthogonal matrices.\footnote{An orthogonal matrix $\vec A$ satisfies $\vec{A}^T \cdot \vec{A} = \vec{A} \cdot \vec{A}^T = \vec{I}$.} Consider the following linear neural network:
    \begin{equation}
        \label{eq:deepLinNet}
        g_{\vec{W}}(\vec{x}) := \gamma \vec{W}_{\beta} \cdot \vec{W}_{\beta-1} \cdot \ldots \cdot \vec{W}_1 \cdot \vec{x},
    \end{equation}
    where the rescaling parameter $\gamma = \frac{1}{\sqrt{2 \beta}}$, and assume we use the squared loss function. Then we have the following.
    \begin{restatable}{theorem}{OrthInit}
    	\label{thm:OrthInit}
    	Let $\{ \vec{W}_i \}_{i \in [\beta]}$ be arbitrary orthogonal matrices that satisfy assumption~\ref{ass:small_weight}. Let the dataset $\{ \vec{x}_i \} _{i \in [N]}$ be such that each $\vec{x}_i$ is an i.i.d. sample from the surface of $d$-dimensional unit sphere. Consider the linear neural network in equation \ref{eq:deepLinNet} that minimizes the empirical square loss function. For some fixed constant $c >0$, the gradient confusion bound (equation \ref{confusion}) holds with probability at least $$\textstyle 1- N^2 \exp\left( -c d \eta^2 \right).$$
    	\end{restatable}
    	From Theorem~\ref{thm:OrthInit}, we see that the probability does not depend on the depth $\beta$ or maximum width $\ell$. Thus, trainability does not get worse with depth when using orthogonal initializations. This result matches previous theoretical and empirical results showing the efficiency of orthogonal initialization techniques for training very deep linear or tanh networks \citep{saxe2013exact, schoenholz2016deep, xiao2018dynamical}. However, orthogonal initializations are not compatable with non-linear activation functions like sigmoids or ReLUs, which limit their use in practice. Nonetheless, this result suggests a promising direction in developing techniques for training deeper models.
    	
    

\begin{figure*}[t]
\centering
\subfigure[]{\includegraphics[width=0.32\textwidth]{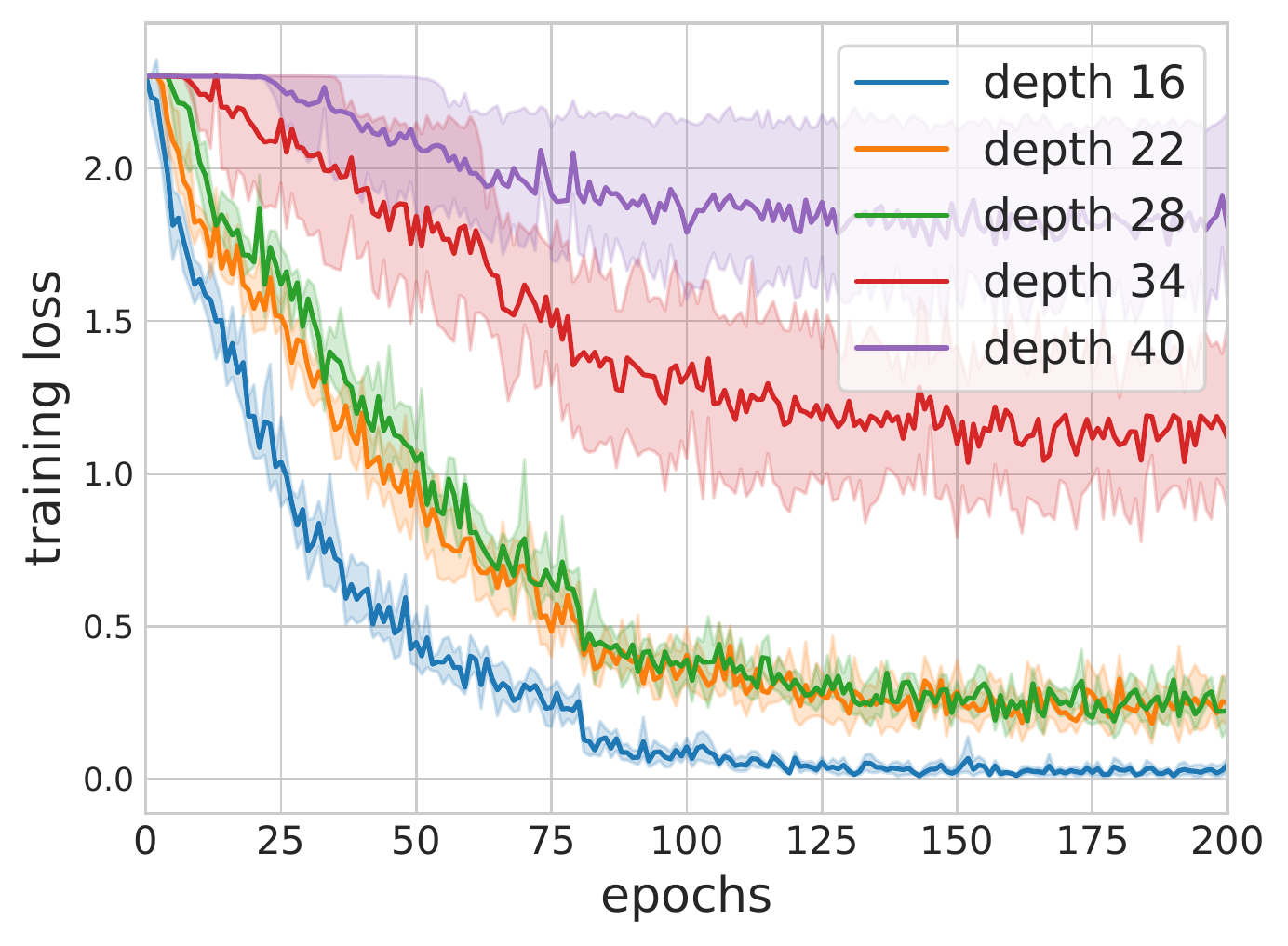}}
\subfigure[]{\includegraphics[width=0.32\textwidth]{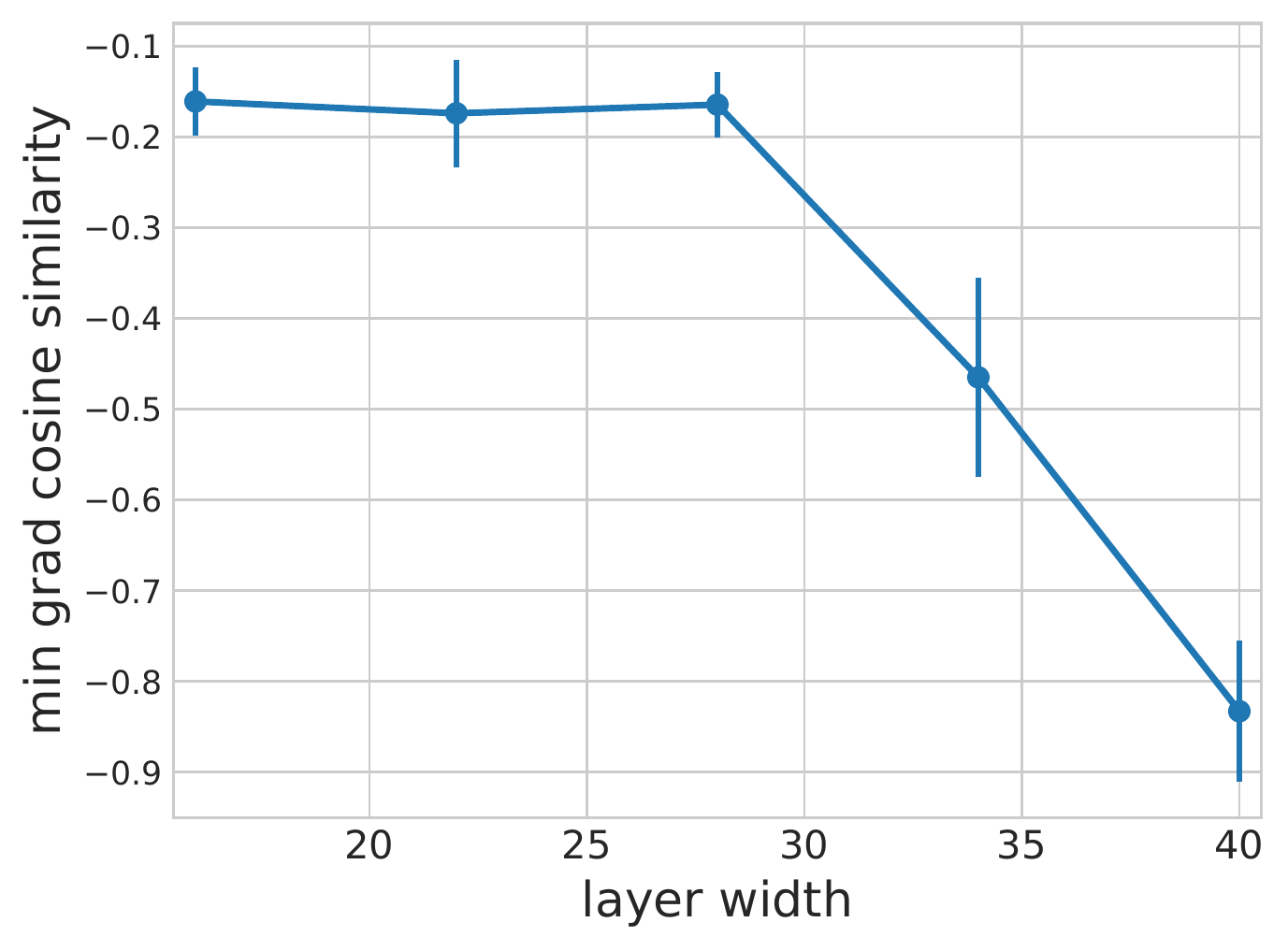}}
\subfigure[]{\includegraphics[width=0.32\textwidth]{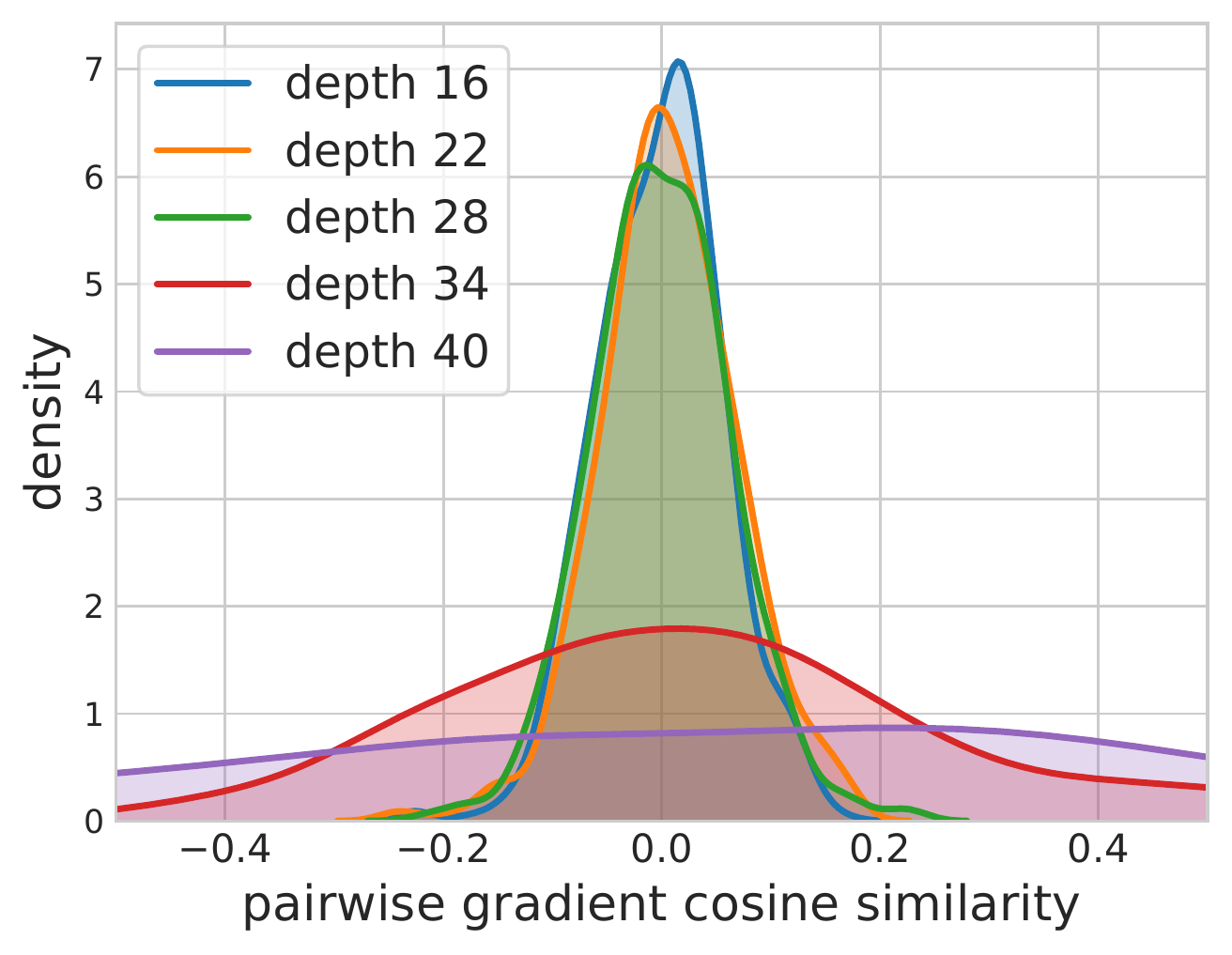}}
      \caption{\small{The effect of network depth with CNN-$\beta$-2 on CIFAR-10 for depths $\beta = $ 16, 22, 28, 34 and 40. Plots show the (a) convergence curves for SGD, (b) minimum of pairwise gradient cosine similarities at the end of training, and the (c) kernel density estimate of the pairwise gradient cosine similarities at the end of training (over all independent runs).
       }}
       \label{fig:c10_cnn_depth}
\end{figure*}

\section{Experimental results}
\label{sec:experiments}

To test our theoretical results and to probe why standard neural networks are efficiently trained with SGD, we now present experimental results showing the effect of the neural network architecture on the convergence of SGD and gradient confusion. It is worth noting that theorems \ref{thm:cosine} and \ref{thm:nonconvex} indicate that 
we would expect the effect of gradient confusion to be most prominent closer to the end of training.

We performed experiments on wide residual networks (WRNs) \citep{zagoruyko2016wide}, convolutional networks (CNNs) and multi-layer perceptrons (MLPs) for image classification tasks on CIFAR-10, CIFAR-100 and MNIST. We present results for CNNs on CIFAR-10 in this section, and present all other results in appendix \ref{app:extra_results}. We use CNN-$\beta$-$\ell$ to denote WRNs that have no skip connections or batch normalization, with a depth $\beta$ and width factor $\ell$.\footnote{The width factor denotes the number of filters relative to the original ResNet model \citep{zagoruyko2016wide}.} We turned off dropout and weight decay for all our experiments. We used SGD as the optimizer without any momentum. Following \citet{zagoruyko2016wide}, we ran all experiments for 200 epochs with minibatches of size 128, and reduced the initial learning rate by a factor of 10 at epochs 80 and 160. We used the MSRA initializer \citep{he2015delving} for the weights as is standard for this model, and used the same preprocessing steps for the CIFAR-10 images as described in \citet{zagoruyko2016wide}. We ran each experiment 5 times, and we show the standard deviation across runs in our plots. We tuned the optimal initial learning rate for each model over a logarithmically-spaced grid and selected the run that achieved the lowest training loss value. To measure gradient confusion, at the end of every training epoch, we sampled 100 pairs of mini-batches each of size 128 (the same size as the training batch). We calculated gradients on each mini-batch, and then computed pairwise cosine similarities. 
See appendix \ref{sec:supp_exp} for more details on the experimental setup and architectures used.

\begin{figure*}[t]
\centering
\subfigure[]{\includegraphics[width=0.32\textwidth]{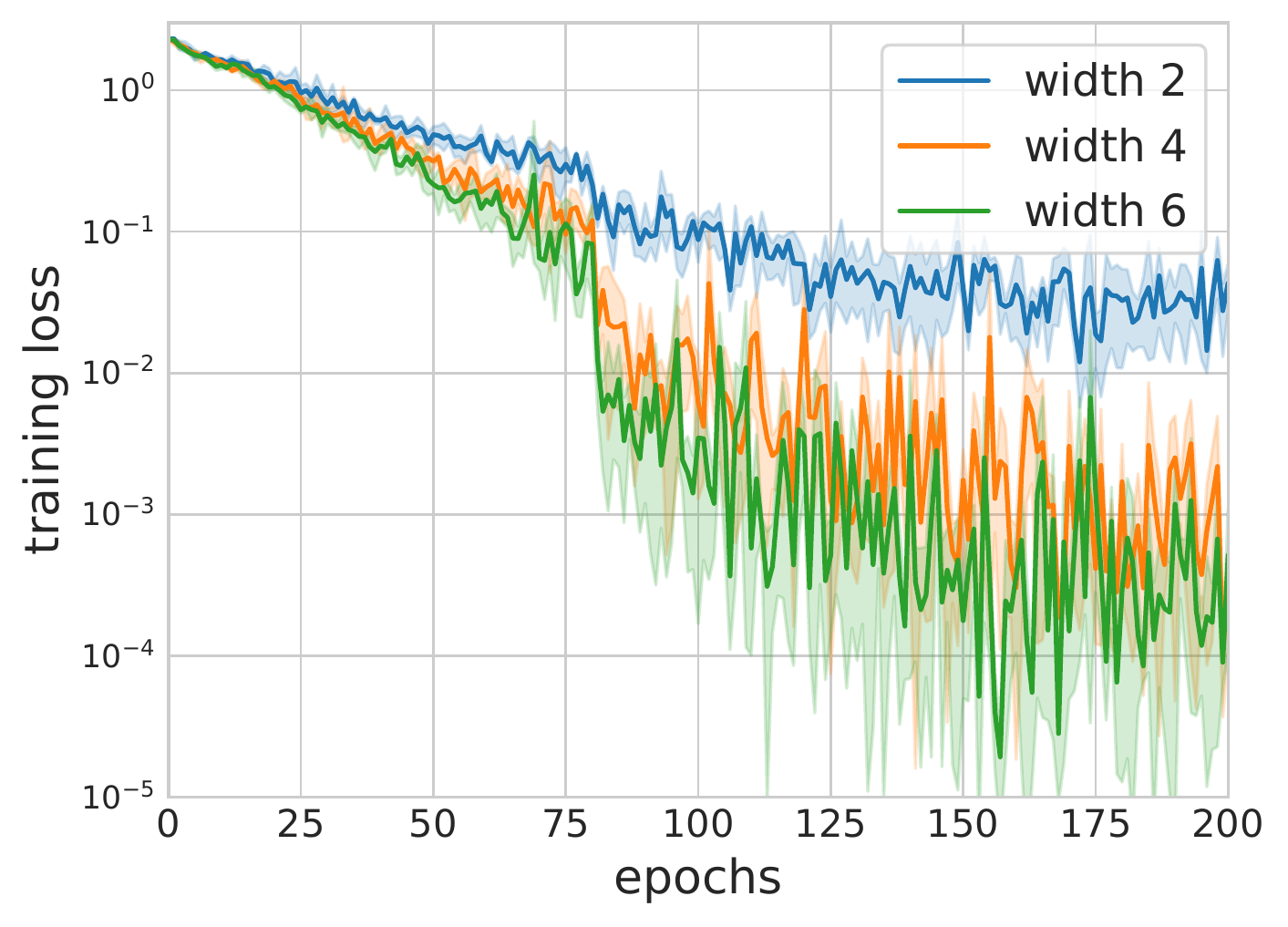}}
\subfigure[]{\includegraphics[width=0.32\textwidth]{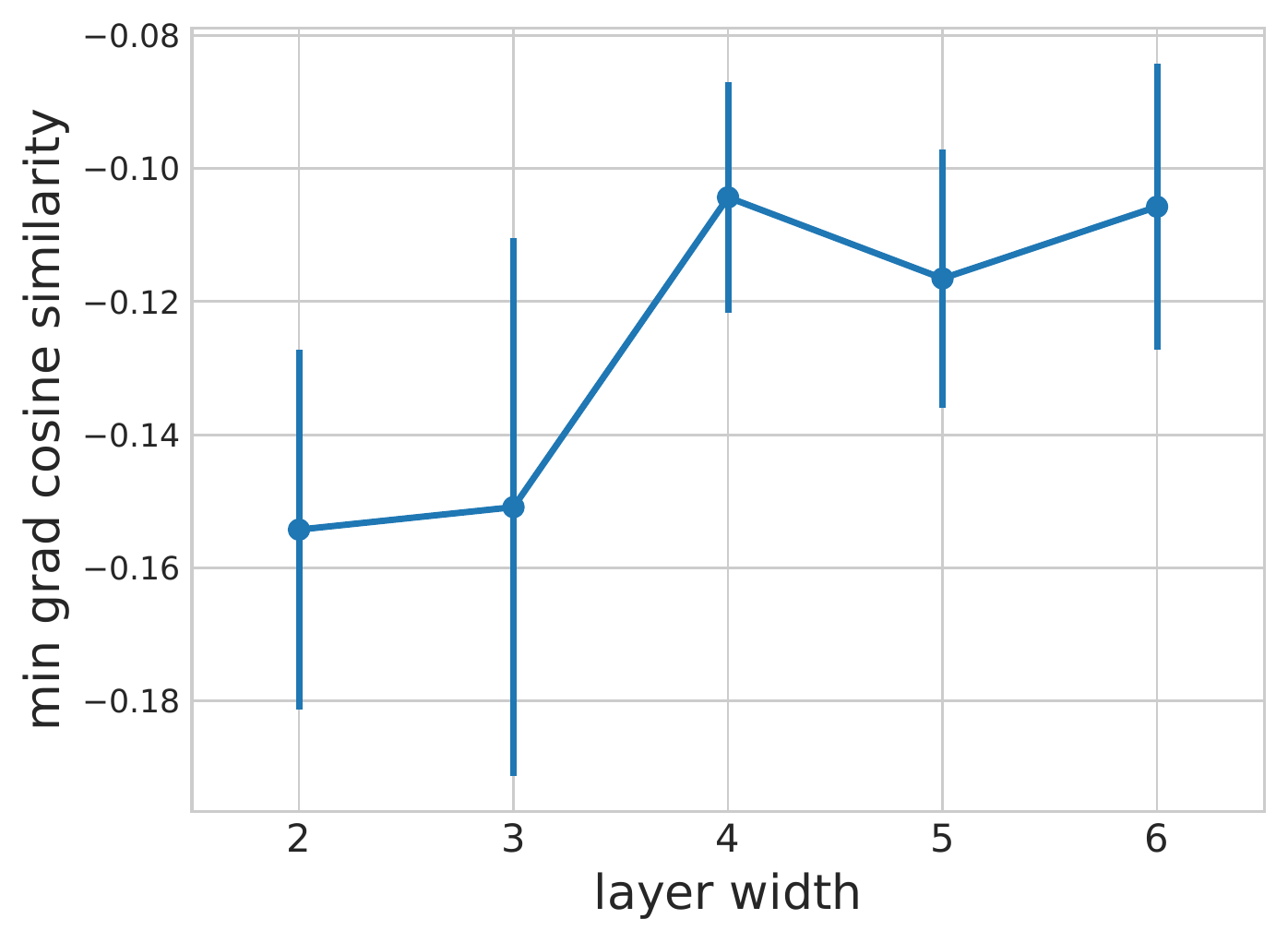}}
\subfigure[]{\includegraphics[width=0.32\textwidth]{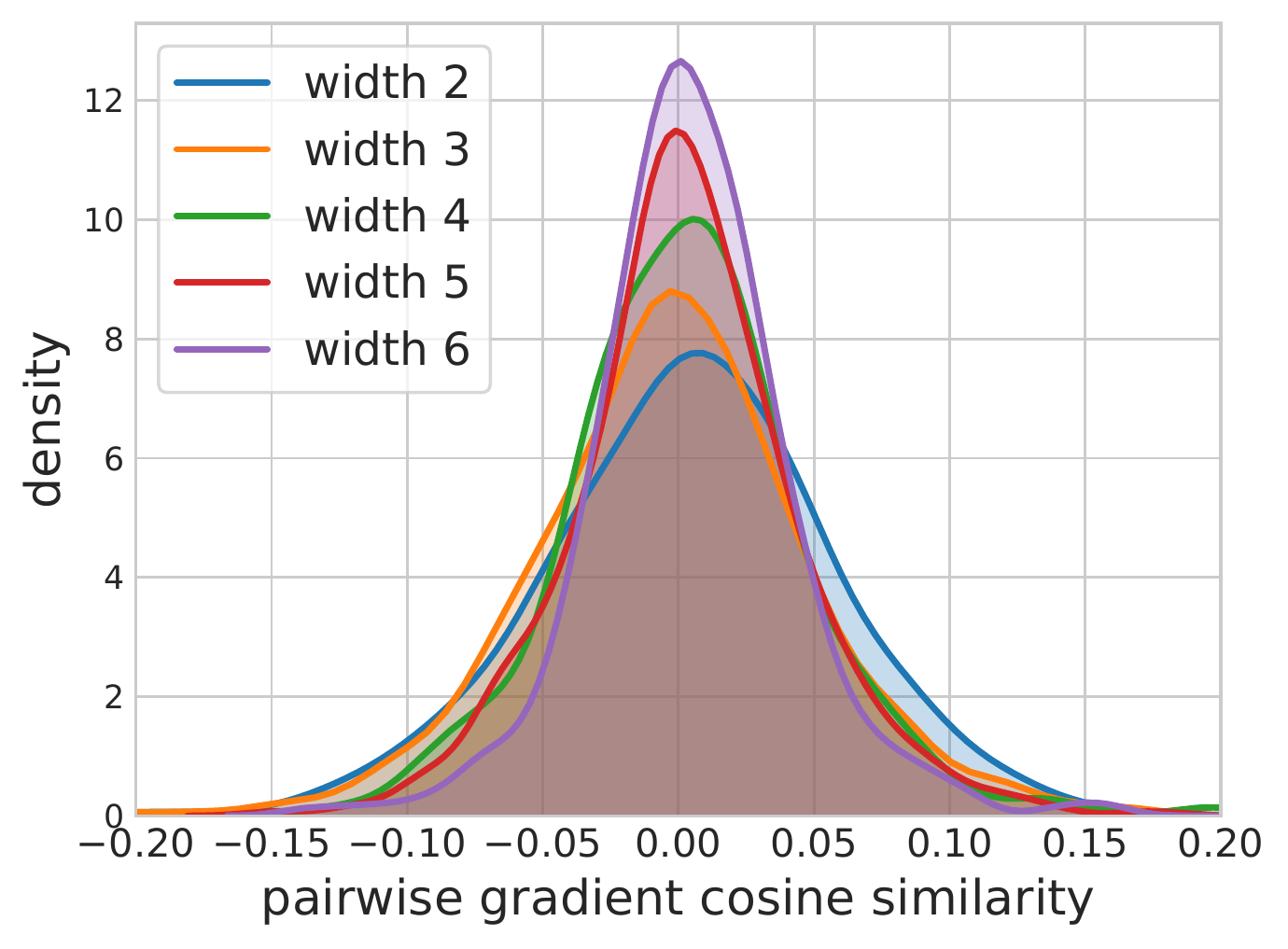}}
\caption{The effect of width with CNN-16-$\ell$ on CIFAR-10 for width factors $\ell = $ 2, 3, 4, 5 and 6. Plots show the (a) convergence curves of SGD (for cleaner figures, we plot results for width factors 2, 4 and 6 here), (b) minimum of pairwise gradient cosine similarities at the end of training, and the (c) kernel density estimate of the pairwise gradient cosine similarities at the end of training (over all independent runs).}
\label{fig:c10_cnn_width}
\end{figure*}
      
\begin{figure*}[t]
\centering
\subfigure[]{\includegraphics[width=0.32\textwidth]{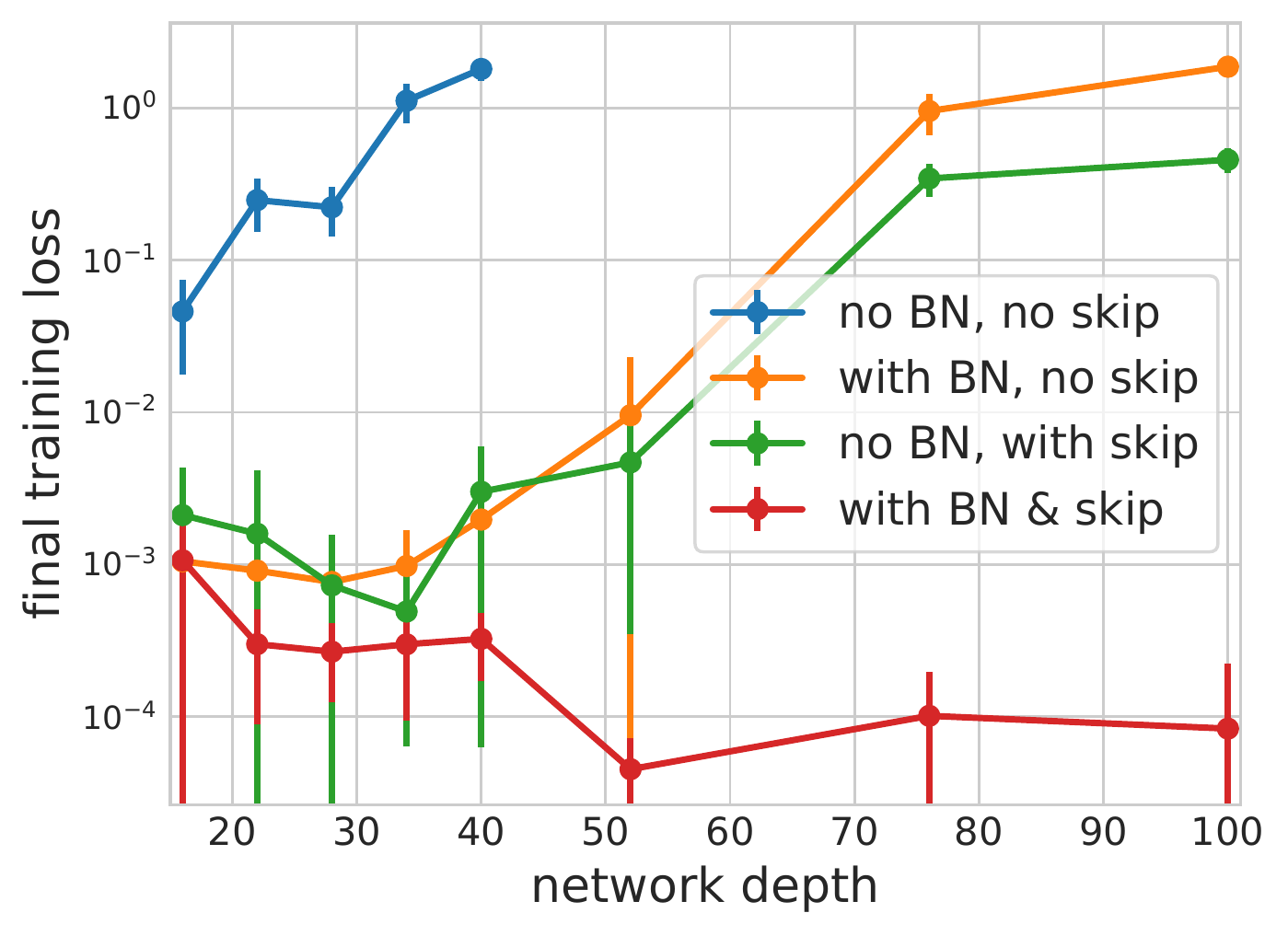}}
\subfigure[]{\includegraphics[width=0.32\textwidth]{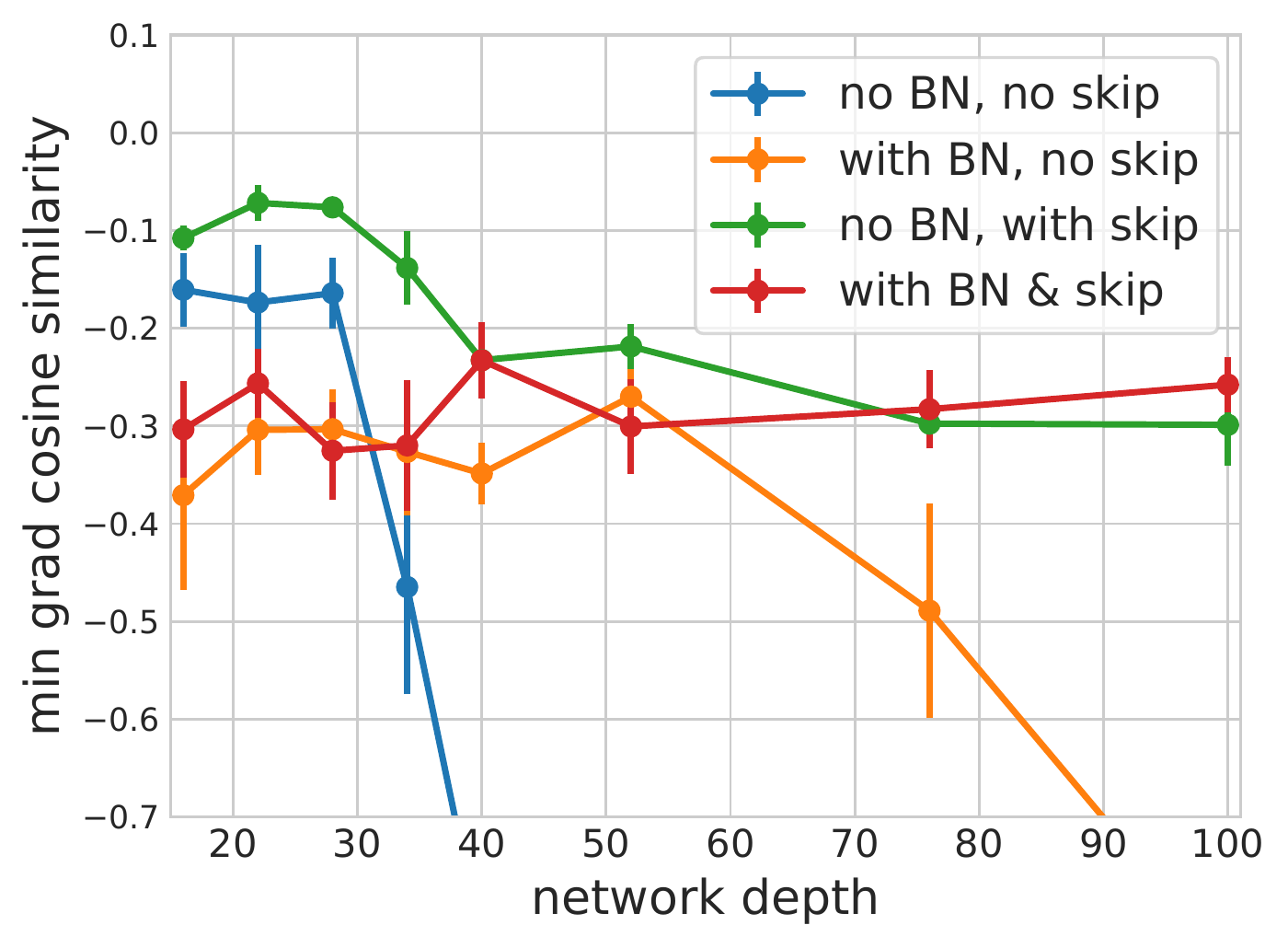}}
\subfigure[]{\includegraphics[width=0.32\textwidth]{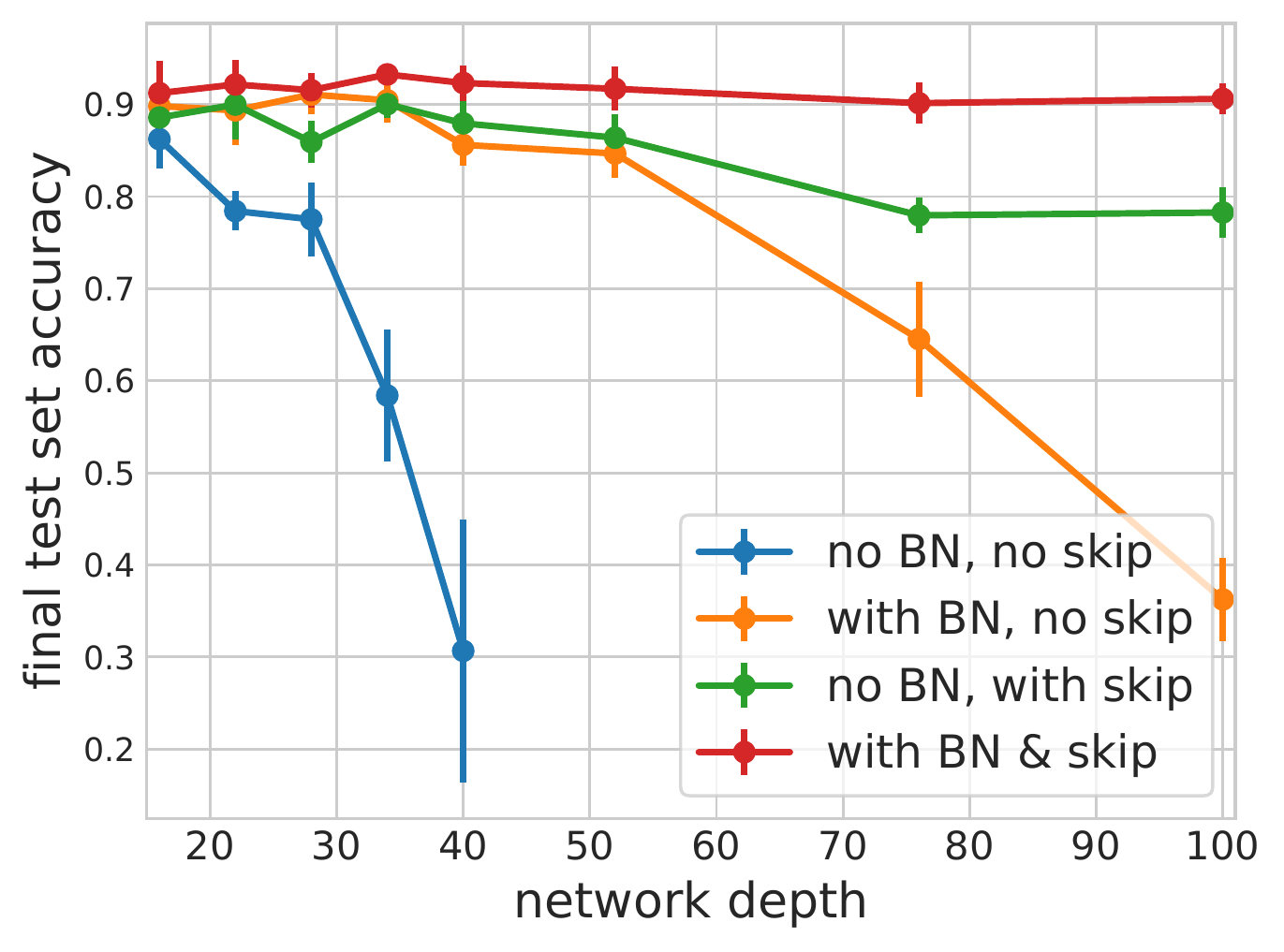}}
\caption{The effect of adding skip connections and batch normalization to CNN-$\beta$-2 on CIFAR-10 for depths $\beta = $ 16, 22, 28, 34, 40, 52, 76 and 100. Plots show the (a) optimal training losses, (b) minimum pairwise gradient cosine similarities, and the (c) test set accuracies at the end of training.}
\label{fig:c10_cnn_bnskip}
\end{figure*}


\textbf{Effect of depth. }
To test our theoretical results, we consider CNNs with a fixed width factor of 2 and varying network depth. From figure~\ref{fig:c10_cnn_depth}, we see that our theoretical results are backed by the experiments: increasing depth slows down convergence, and increases gradient confusion. We also notice that with increasing depth, the density of pairwise gradient cosine similarities concentrates less sharply around 0, which makes the network harder to train. 

\textbf{Effect of width. }
We now consider CNNs with a fixed depth of 16 and varying width factors. 
From figure~\ref{fig:c10_cnn_width}, we see that increasing width results in faster convergence and lower gradient confusion. We further see that gradient cosine similarities concentrate around 0 with growing width, indicating that SGD decouples across the training samples with growing width.
Note that the smallest network considered (CNN-16-2) is still over-parameterized and achieves a high level of performance (see appendix \ref{app:additional_plots}). 


\textbf{Effect of batch normalization and skip connections. }
Almost all state-of-the-art neural networks currently contain both skip connections and normalization layers. To help understand why such neural networks are so efficiently trained using SGD with constant learning rates, we test the effect of adding skip connections and batch normalization to CNNs of fixed width and varying depth. Figure~\ref{fig:c10_cnn_bnskip} shows that adding skip connections or batch normalization individually help in training deeper models, but these models still suffer from worsening results and increasing gradient confusion as the network gets deeper. When these techniques are used together, the model has relatively low gradient confusion even for very deep networks, significantly improving trainability of deep models. Note that our observations are consistent with prior work \citep{de2020batch, yang2019mean}.

\section{Alternate definitions of gradient confusion}
\label{sec:altdef}

Note that the gradient confusion bound $\eta$ in equation \ref{confusion} is defined for the worst-case gradient inner product. However, all the results in this paper can be trivially extended to using a bound on the average gradient inner product of the form: $$\textstyle \sum_{i, j = 1}^N \langle \nabla f_i (\vec{w}) , \nabla f_j (\vec{w}) \rangle / N^2  \ge - \eta.$$ In this case, all theoretical results would remain the same up to constants.
We can also define a normalized variant of the gradient confusion condition: $$ \langle \nabla f_i (\vec{w}) , \nabla f_j (\vec{w}) \rangle / (\| \nabla f_i (\vec{w}) \| \| \nabla f_j (\vec{w}) \|)  \ge - \eta.$$
This condition inherently makes an additional assumption that the norm of the stochastic gradients, $\| \nabla f_i (\vec{w}) \|$, is bounded, and thus the gradient variance is also bounded (see discussion in section \ref{sec:related}).
Thus, while all our theoretical results would qualitatively remain the same under this  condition, we can prove tighter versions of our current results.

Finally, note that gradient confusion condition in equation \ref{confusion} is applicable even when the stochastic gradients are averaged over minibatches of size $B$. The variance of the gradient inner product scales down as $1/B^2$ in this case, and thus $\eta$ is expected to decrease as $B$ grows.

\section{Related work}
\label{sec:related}

The gradient confusion bound and our theoretical results have interesting connections to prior work. In this section, we briefly discuss some of these connections.

\textbf{Connections to the gradient variance}:
If we assume bounded gradient variance $\expect_i \| \nabla f_i(\vec{w}) - \nabla F(\vec{w}) \|^2 \le \sigma^2$, we can bound the gradient confusion parameter $\eta$ in terms of other quantities. For example, suppose the true gradient $\nabla F(\vec{w}) = \nabla f_1(\vec{w})/2 + \nabla f_2(\vec{w})/2$. Then we can write: $| \langle \nabla f_1(\vec{w}), \nabla f_2(\vec{w}) \rangle | \leq \sigma^2 + \| \nabla F(\vec{w}) \|^2$.
However, in general one cannot bound the gradient variance in terms of the gradient confusion parameter.  As a counter-example, consider a problem with the following distribution on the gradients:  $\frac{1}{1-p}$ samples with gradient $\frac{1}{\epsilon}$ and $\frac{1}{p}$ samples with gradient $\epsilon$, where $p = \epsilon \to 0$. In this case, the gradients are positive, so gradient confusion $\eta = 0$. The mean of the gradients is given by $1+\epsilon (1-\epsilon)$, which remains bounded as $\epsilon\to0$. On the other hand, the variance (and thus the squared norm of the stochastic gradients) is unbounded ($O(1/\epsilon)$ as $\epsilon \rightarrow 0$). 
A consequence of this is that in theorems \ref{thm:cosine} and \ref{thm:nonconvex}, the "noise term" (i.e., the second term in the RHS of the convergence bounds) does not depend on the learning rate in the general case. If gradients have unbounded variance, lowering the learning rate does not reduce the variance of the SGD updates, and thus does not reduce the noise term. 



\textbf{Connections to gradient diversity:}
Gradient diversity \citep{yin2017gradient} also measures the degree to which individual gradients at different data samples are different from each other. However, the gradient diversity measure gets larger as the individual gradients become orthogonal to each other, and further increases as the gradients start pointing in opposite directions. 
On the other hand, gradient confusion between two individual gradients is zero unless the inner product between them is negative. 
As we show in this paper, this has important implications when we study the convergence of SGD in the over-parameterized setting: increased width makes gradients more orthogonal to each other improving trainability, while increased depth result in gradients pointing in opposite directions making networks harder to train.
Thus, we view our papers to be complementary, providing insights about different issues (large batch distributed training vs.~small batch convergence).

\textbf{Related work on the impact of network architecture:}
%
\citet{balduzzi2017shattered} studied neural networks with ReLU activations at Gaussian initializations, and showed that gradients become increasingly negatively correlated with depth.
\citet{hanin2018neural} showed that the variance of gradients in fully connected networks with ReLU activations is exponential in the sum of the reciprocals of the hidden layer widths at Gaussian initializations. 
In a follow-up work, \citet{hanin2018start} showed that this sum of the reciprocals of the hidden layer widths determines the variance of the sizes of the activations at each layer. When this sum of reciprocals is too large, early training dynamics are very slow, suggesting the difficulties of starting training on deeper networks, as well as the benefits of increased width.

\textbf{Other work on SGD convergence:}
%
There has recently been a lot of interest in analyzing conditions under which SGD converges to global minimizers of over-parameterized linear and non-linear neural networks. \citet{arora2018convergence} shows SGD converges linearly to global minimizers for linear neural networks under certain conditions. \citet{du2018gradient, allen2018convergence, zou2018stochastic, brutzkus2017sgd} also show convergence to global minimizers of SGD for non-linear networks. 
This paper complements these recent results by studying how low gradient confusion contributes to SGD's success on over-parameterized neural networks used in practice.

\section{Discussion}
\label{sec:discussion}

In this paper, we study how neural network architecture affects the trainability of networks and the dynamics of SGD. To rigorously analyze this, we introduce a concept called gradient confusion, and show that when gradient confusion is low, SGD has fast convergence. We show at standard Gaussian initializations, increasing layer width leads to lower gradient confusion, making the model easier to train. In contrast, increasing depth results in higher gradient confusion, making models harder to train. These results indicate that increasing the layer width with the network depth is important to maintain trainability of the neural network. This is supported by other recent work that suggest that the width should increase linearly with depth in a Gaussian-initialized neural network to help dynamics early in training \citep{hanin2018neural, hanin2018start}.

Many previous results have shown how deeper models are more efficient at modeling higher complexity function classes than wider models, and thus depth is essential for the success of neural networks \citep{eldan2016power, telgarsky2016benefits}. Indeed, practitioners over the years have achieved state-of-the-art results on various tasks by making networks deeper, without necessarily making networks wider. We thus study techniques that enable us to train deep models without requiring us to increase the width with depth. Most state-of-the-art neural networks currently contain both skip connections and normalization layers. We thus, empirically study the effect of introducing batch normalization and skip connections to a neural network. We show that the combination of batch normalization and skip connections lower gradient confusion and help train very deep models, explaining why many neural networks used in practice are so efficiently trained. Furthermore, we show how orthogonal initialization techniques provide a promising direction for improving the trainability of very deep networks.


Our results provide a number of important insights that can be used for neural network model design. We demonstrate that the gradient confusion condition could be useful as a measure of trainability of networks, and thus could potentially be used to develop algorithms for more efficient training. 
Additionally, the correlation between gradient confusion and the test set accuracies shown in appendix \ref{app:extra_results} suggest that an interesting topic for future work would be to investigate the connection between gradient confusion and generalization \citep{fort2019stiffness}. 
Our results also suggest the importance of further work on orthogonal initialization schemes for neural networks with non-linear activations that make training very deep models possible.

\section*{Acknowledgements}
The authors thank Brendan O’Donoghue, Aleksandar Botev, James Martens, Sudha Rao, and Samuel L.~Smith for helpful discussions and for reviewing earlier versions of this manuscript. This paper was supported by the ONR MURI program, AFOSR MURI Program, and the National Science Foundation DMS directorate.

\nocite{lecun2012efficient, nesterov2018lectures, martens2016second, sagun2017empirical, cooper2018loss, chaudhari2016entropy, wu2017towards, ghorbani2019investigation, neyshabur2018towards, dziugaite2017computing, nagarajan2019generalization, zou2018stochastic, allen2018convergence, du2018gradient, oymak2018overparameterized, vershynin2016high, tao2012topics, boucheron2013concentration, sedghi2018singular, milman1986asymptotic}

\bibliographystyle{icml2020}
\bibliography{references}
\newpage
\appendix

\onecolumn

\section*{Appendix}
We first briefly outline the different sections in the appendix.
\begin{itemize}
\item In appendix \ref{app:extra_results}, we provide details of our experimental setup, and provide additional empirical results on fully connected networks, convolutional networks and residual networks with the MNIST, CIFAR-10 and CIFAR-100 datasets.

\item In appendix \ref{app:orthovec}, we state and prove a lemma on the near orthogonality of random vectors, which we refer to in the main text. This result is often attributed to \citet{milman1986asymptotic}.

\item In appendix \ref{app:hessians}, we provide some intuition on why many standard over-parameterized neural networks with low-rank Hessians might have low gradient confusion for a large set of weights near the minimizer.

\item In appendix \ref{app:missing_proofs}, we provide the proofs of the theorems presented in the main section. In appendix \ref{appsec:rate}, we provide proofs of theorems \ref{thm:cosine} and \ref{thm:nonconvex}. In appendix \ref{appsec:helperlemma}, we provide the proof of lemma \ref{lem:lossPropNeural}, which we refer to in the main text. In appendix \ref{appsec:nnconc}, we provide proofs of theorem \ref{thm:arbitraryNN} and corollary \ref{thm:uniformNN}. In appendix \ref{appsec:randomw}, we provide the proof of theorem \ref{thm:fixedData}. In appendix \ref{appsec:orthoinit}, we provide the proof of theorem~\ref{thm:OrthInit}.

\item In appendix \ref{app:technical_lemmas}, we briefly describe a few lemmas that we require in our analysis.

\item In appendix \ref{app:small_weights}, we discuss the small weights assumption (assumption \ref{ass:small_weight}), which is required for theorem \ref{thm:arbitraryNN}, corollary \ref{thm:uniformNN} and theorem \ref{thm:OrthInit} in the main text.
\end{itemize}

\section{Additional experimental results}
\label{app:extra_results}

In this section, we present more details about our experimental setup, as well as, additional experimental results on a range of models (MLPs, CNNs and Wide ResNets) and a range of datasets (MNIST, CIFAR-10, CIFAR-100).

\subsection{MLPs on MNIST}

To further test the main claims in the paper, we performed additional experiments on an image classification problem on the MNIST dataset using fully connected neural networks. We iterated over neural networks of varying depth and width, and considered both the identity activation function (i.e., linear neural networks) and the tanh activation function. We also considered two different weight initializations that are popularly used and appropriate for these activation functions:
\begin{itemize}
\item The Glorot normal initializer \citep{glorot2010understanding} with weights initialized by sampling from the distribution $\mathcal{N}\big(0, 2 / (\text{fan-in} + \text{fan-out})\big)$, where $\text{fan-in}$ denotes the number of input units in the weight matrix, and $\text{fan-out}$ denotes the number of output units in the weight matrix.
\item The LeCun normal initializer \citep{lecun2012efficient} with weights initialized by sampling from the distribution $\mathcal{N}\big(0, 1 / \text{fan-in}\big)$.
\end{itemize}
We considered the simplified case where all hidden layers have the same width $\ell$. Thus, the first weight matrix $\vec{W}_0 \in \mathbb{R}^{\ell \times d}$, where $d = 784$ for the $28 \times 28$-sized images of MNIST; all intermediate weight matrices $\{ \vec{W}_p \}_{p \in [\beta-1]} \in \mathbb{R}^{\ell \times \ell}$; and the final layer $\vec{W}_\beta \in \mathbb{R}^{10 \times \ell}$ for the 10 image classes in MNIST. We added biases to each layer, which we initialized to 0. We used softmax cross entropy as the loss function. We use MLP-$\beta$-$\ell$ to denote this fully connected network of depth $\beta$ and width $\ell$. We used the standard train-valid-test splits of 40000-10000-10000 for MNIST.

This relatively simple model gave us the ability to iterate over a large number of combinations of network architectures of varying width and depth, and different activation functions and weight initializations. Linear neural networks are an efficient way to directly understand the effect of changing depth and width without increasing model complexity over linear regression. Thus, we considered both linear and non-linear neural networks in our experiments.

We used SGD with constant learning rates for training with a mini-batch size of 128 and trained each model for 40000 iterations (more than 100 epochs). The constant learning rate $\alpha$ was tuned over a logarithmically-spaced grid: $$\alpha \in \{10^0, 10^{-1}, 10^{-2}, 10^{-3}, 10^{-4}, 10^{-5}, 10^{-6}\}.$$ We ran each experiment 10 times (making sure at least 8 of them ran till completion), and picked the learning rate that achieved the lowest training loss value on average at the end of training. Our grid search was such that the optimal learning rate never occurred at one of the extreme values tested.

To measure gradient confusion at the end training, we sampled 1000 pairs of mini-batches each of size 128 (the same size as the training batch size). We calculated gradients on each of these pairs of mini-batches, and then calculated the cosine similarity between them. To measure the worse-case gradient confusion, we computed the lowest gradient cosine similarity among all pairs. We explored the effect of changing depth and changing width on the different activation functions and weight initializations. We plot the final training loss achieved for each model and the minimum gradient cosine similarities calculated over the 1000 pairs of gradients at the end of training. For each point, we plot both the mean and the standard deviation over the 10 independent runs.

\addtocontents{toc}{\protect\setcounter{tocdepth}{1}}


\textbf{The effect of depth. }
We first present results showing the effect of network depth. We considered a fixed width of $\ell = 100$, and varied the depth of the neural network, on the log scale, as: $$\beta \in \{3, 10, 30, 100, 300, 1000\}.$$ Figure \ref{fig:depth_mnist} shows results on neural networks with identity and tanh activation functions for the two weight initializations considered (Glorot normal and LeCun normal).
Similar to the experimental results in section \ref{sec:experiments}, and matching our theoretical results in sections \ref{sec:initialization} and \ref{sec:general_depth}, we notice the consistent trend of gradient confusion increasing with increasing depth. This makes the networks harder to train with increasing depth, and this is evidenced by an increase in the final training loss value. By depth $\beta = 1000$, the increased gradient confusion effectively makes the network untrainable when using tanh non-linearities. 


\textbf{The effect of width. }
We explored the effect of width by varying the width of the neural network while keeping the depth fixed at $\beta = 300$. We chose a very deep model, which is essentially untrainable for small widths (with standard initialization techniques) and helps better illustrate the effects of increasing width. We varied the width of the network, again on the log scale, as: 
$$\ell \in \{10, 30, 100, 300, 1000\}.$$ Crucially, note that the smallest network considered here, MLP-300-10, still has more than 50000 parameters (i.e., more than the number of training samples), and the network with width $\ell = 30$ has almost three times the number of parameters as the high-performing MLP-3-100 network considered in the previous section. Figure \ref{fig:width_mnist} show results on linear neural networks and neural networks with tanh activations for both the Glorot normal and LeCun normal initializations.
As in the experimental results of section \ref{sec:experiments}, we see the consistent trend of gradient confusion decreasing with increasing width. Thus, wider networks become easier to train and improve the final training loss value. We further see that when the width is too small ($\ell = 30$), the gradient confusion becomes drastically high and the network becomes completely untrainable. 

\begin{figure*}[ht]
\centering
\subfigure[Linear NN, Glorot init]{\includegraphics[width=0.24\textwidth]{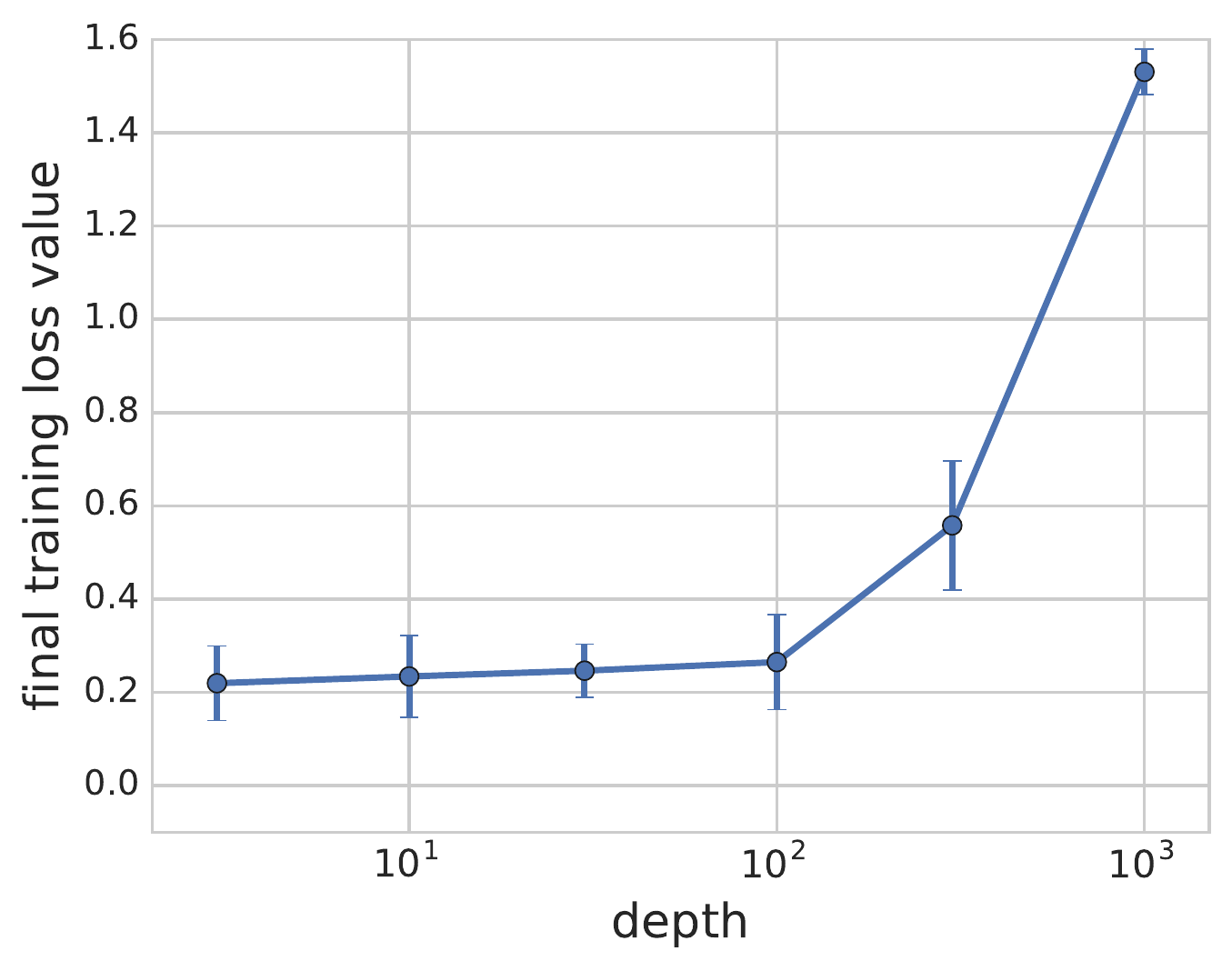}}
\subfigure[Linear NN, Glorot init]{\includegraphics[width=0.24\textwidth]{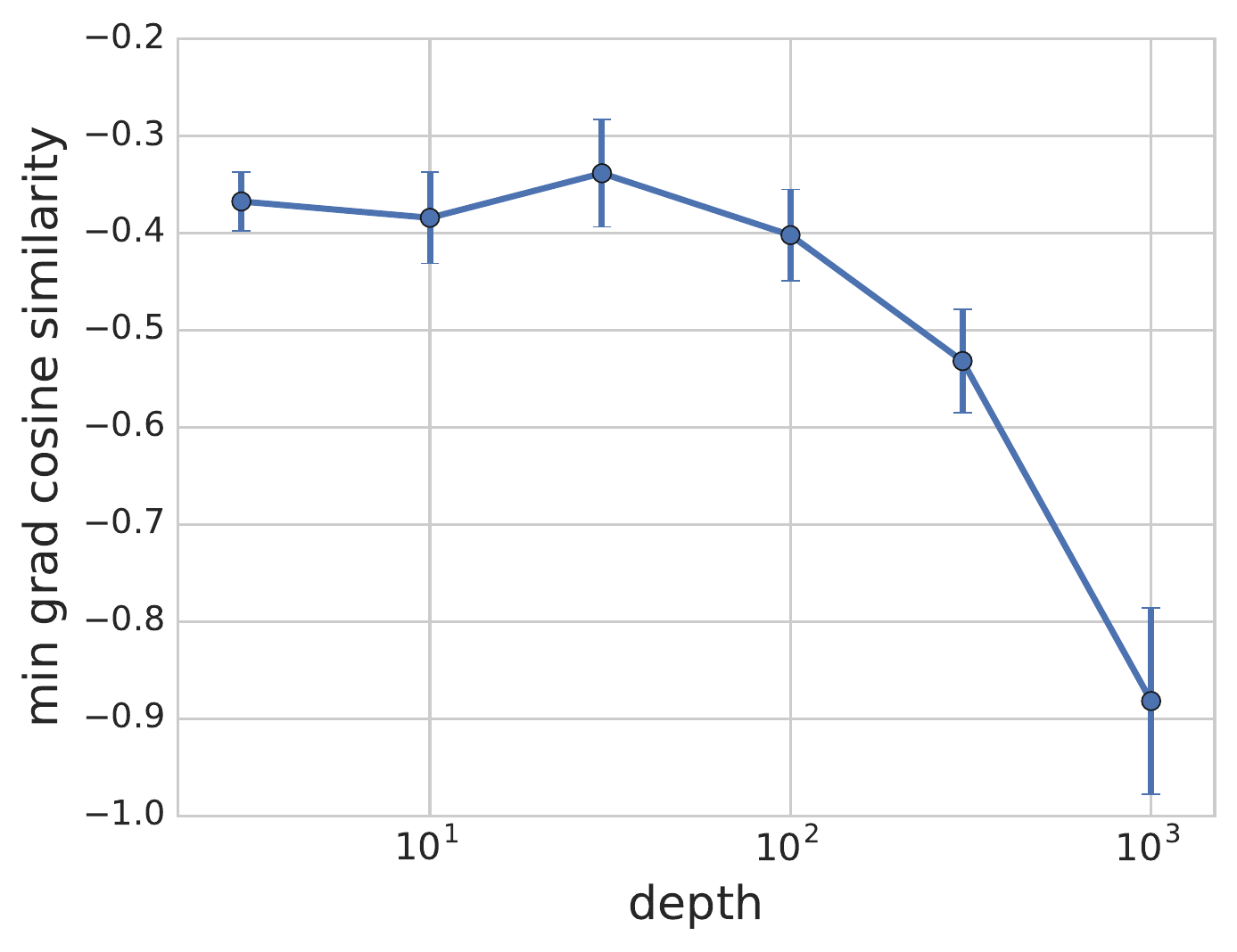}}
\subfigure[Linear NN, LeCun init]{\includegraphics[width=0.24\textwidth]{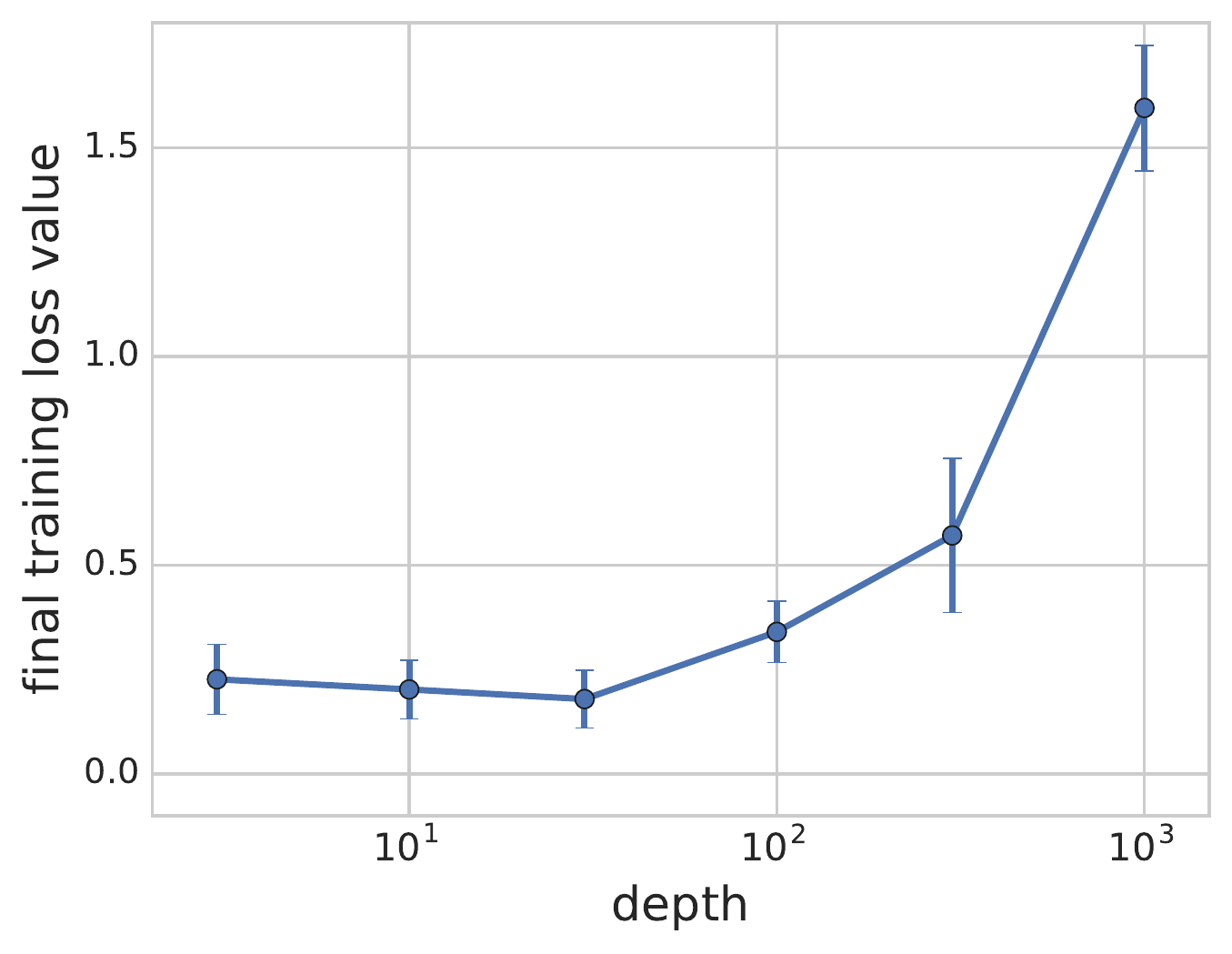}}
\subfigure[Linear NN, LeCun init]{\includegraphics[width=0.24\textwidth]{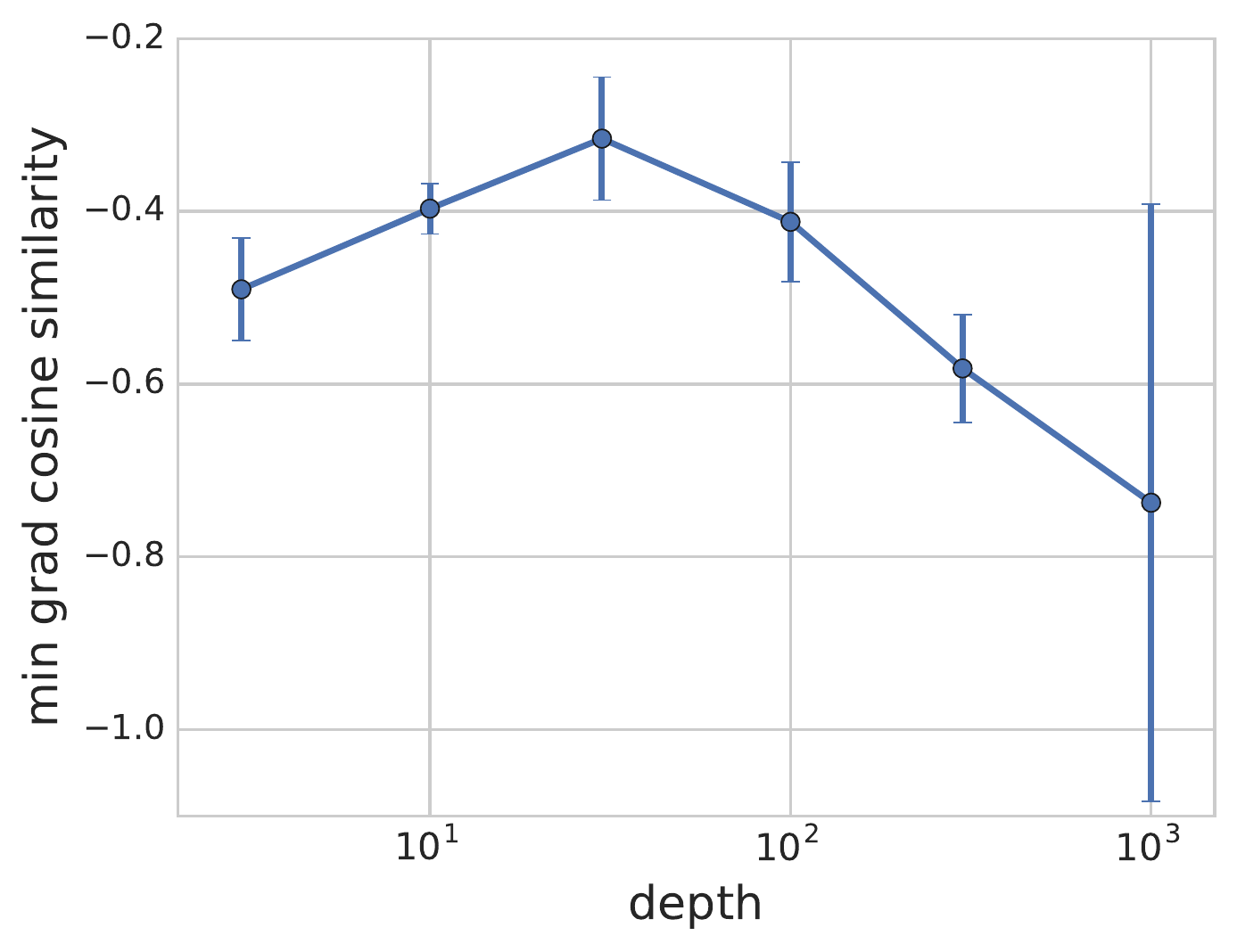}}
\subfigure[Tanh NN, Glorot init]{\includegraphics[width=0.24\textwidth]{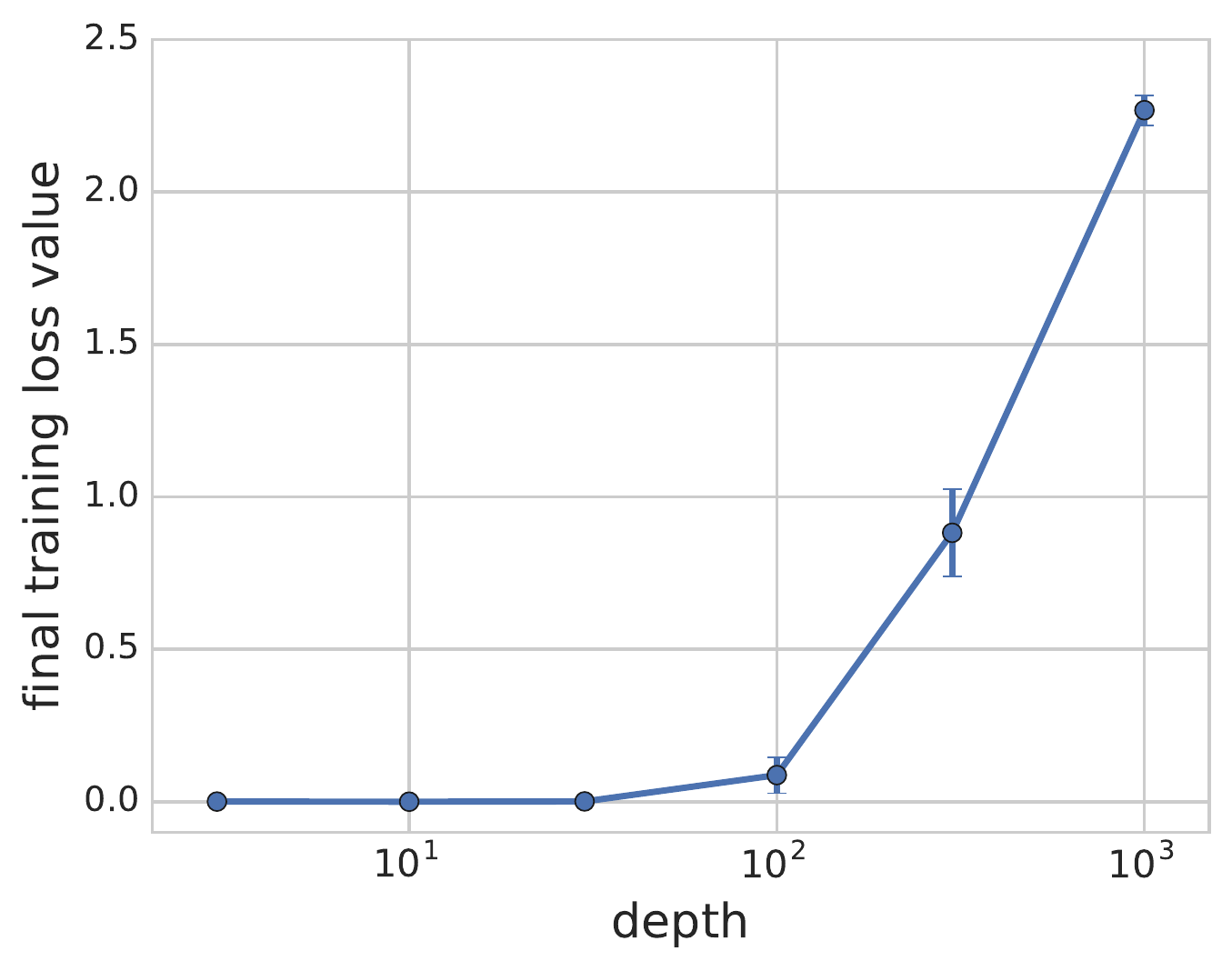}}
\subfigure[Tanh NN, Glorot init]{\includegraphics[width=0.24\textwidth]{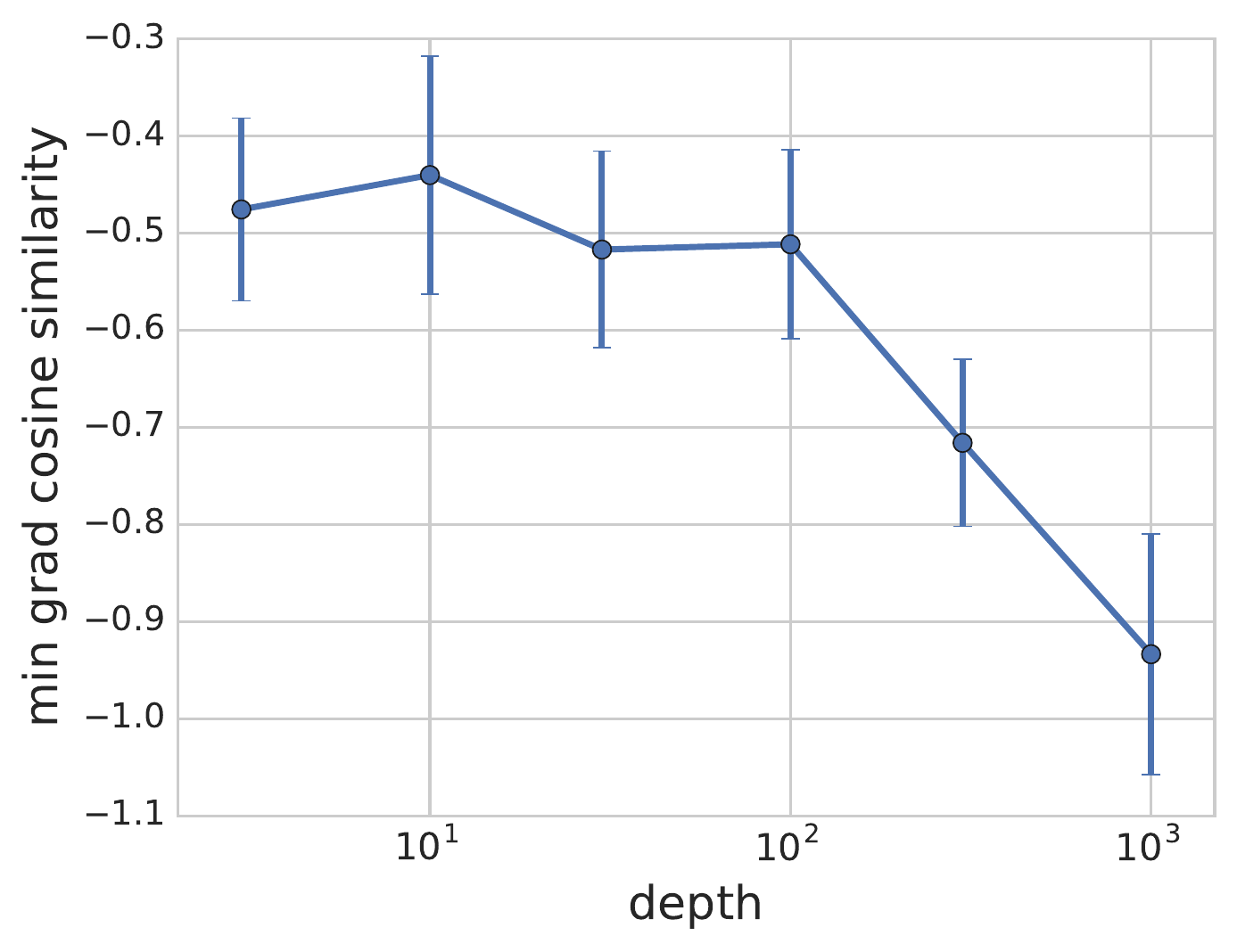}}
\subfigure[Tanh NN, LeCun init]{\includegraphics[width=0.24\textwidth]{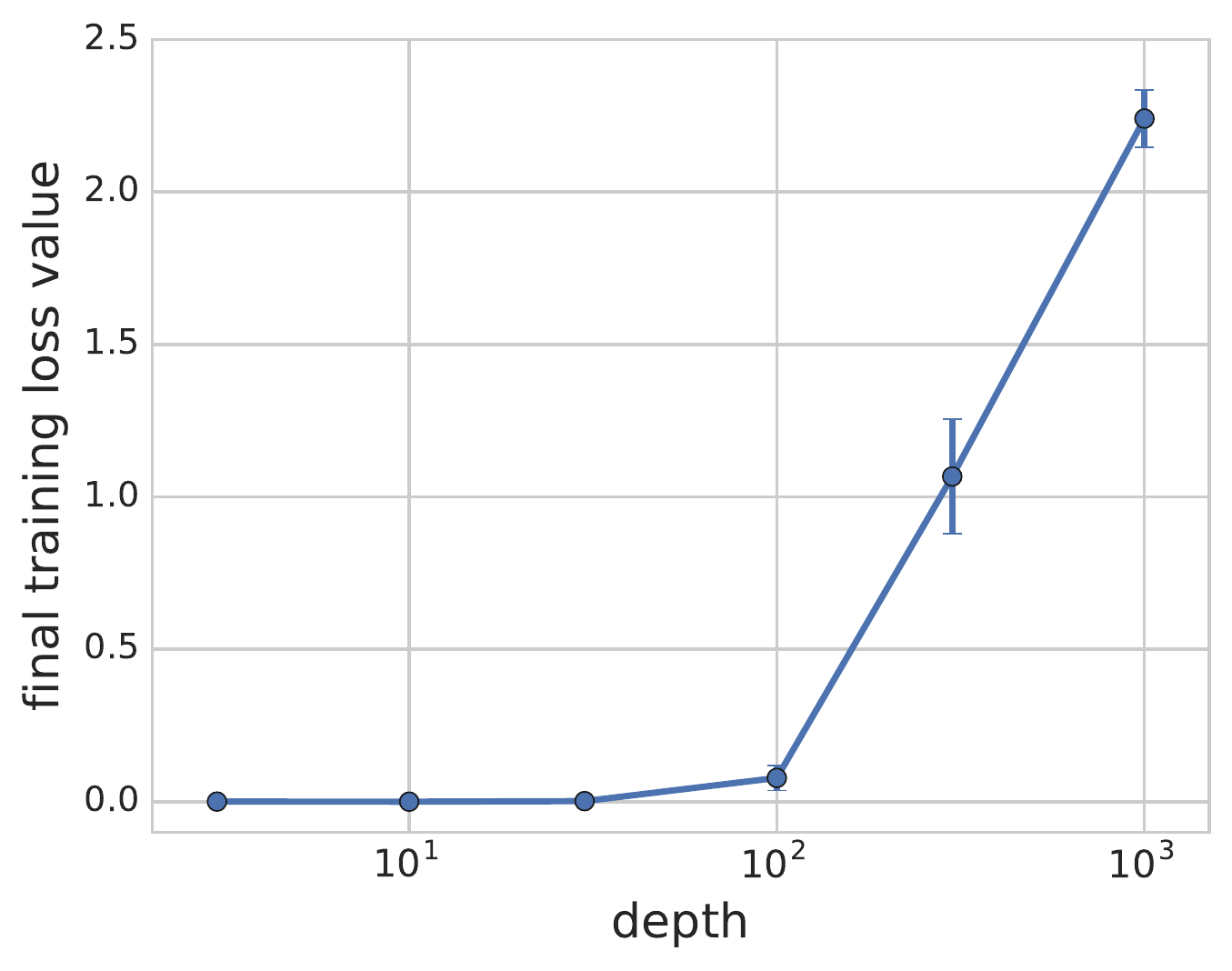}}
\subfigure[Tanh NN, LeCun init]{\includegraphics[width=0.24\textwidth]{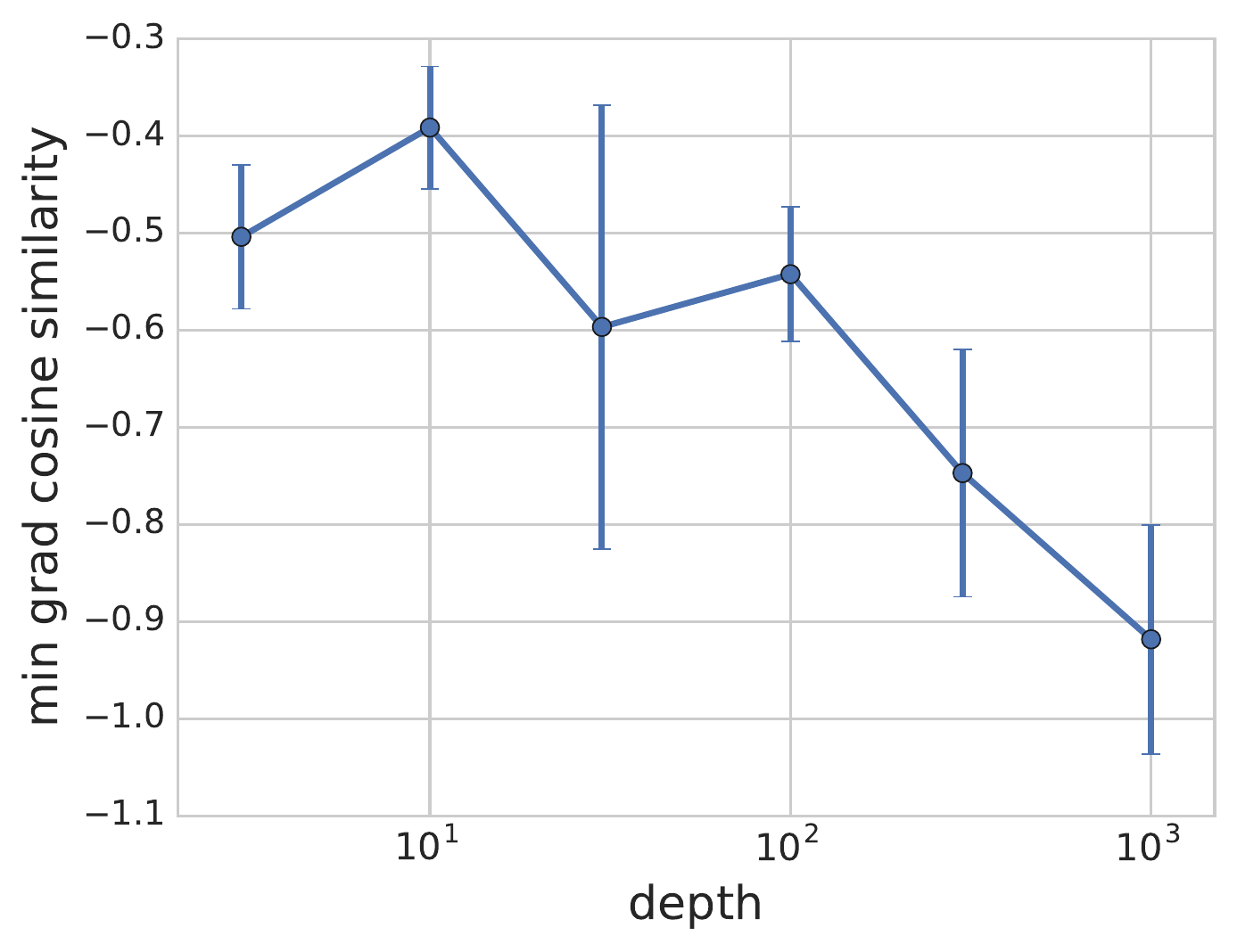}}
\caption{\small{Effect of varying depth on MLP-$\beta$-100.}}
\label{fig:depth_mnist}
\end{figure*}

\begin{figure*}[!h]
\centering
\subfigure[Linear NN, Glorot init]{\includegraphics[width=0.24\textwidth]{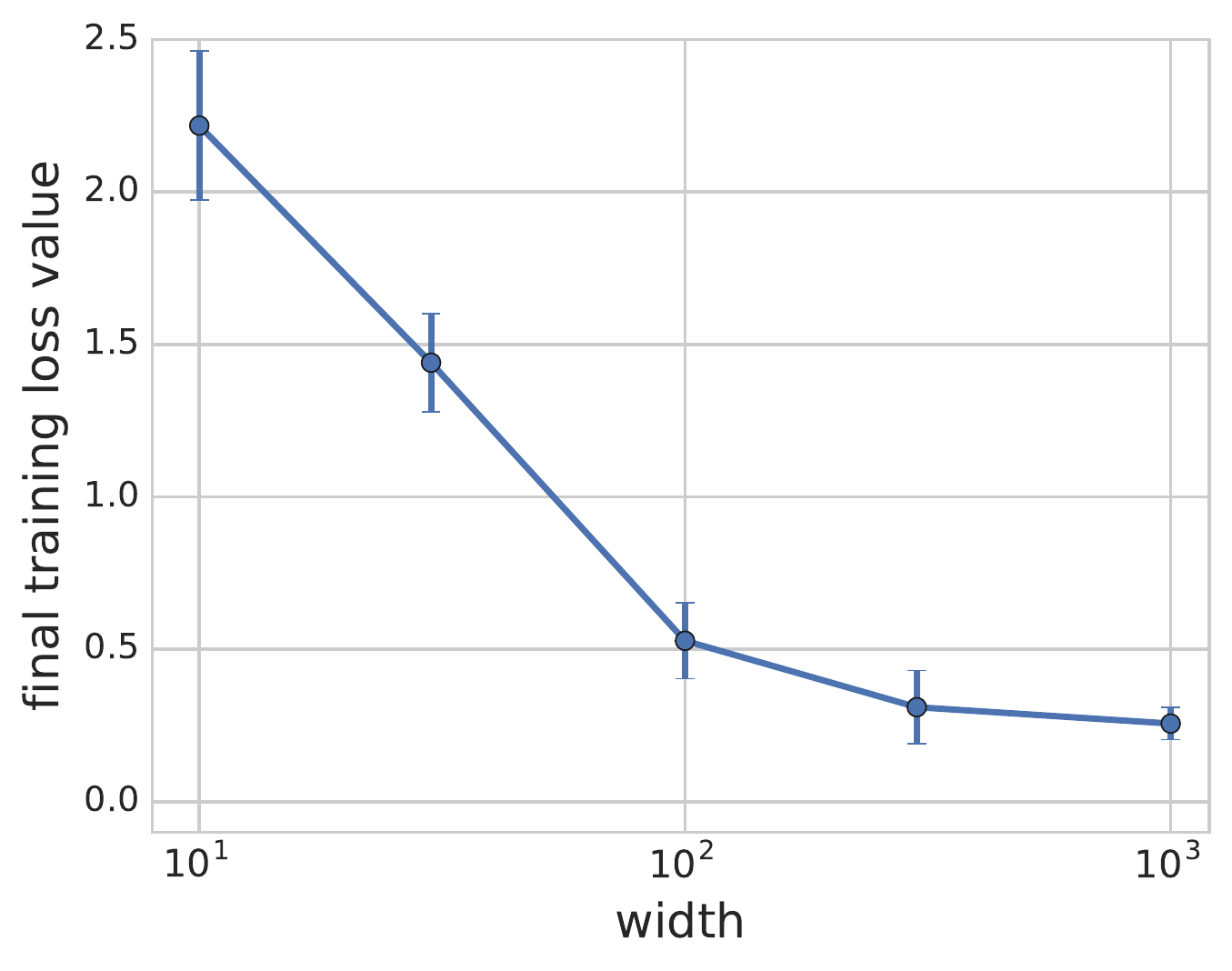}}
\subfigure[Linear NN, Glorot init]{\includegraphics[width=0.24\textwidth]{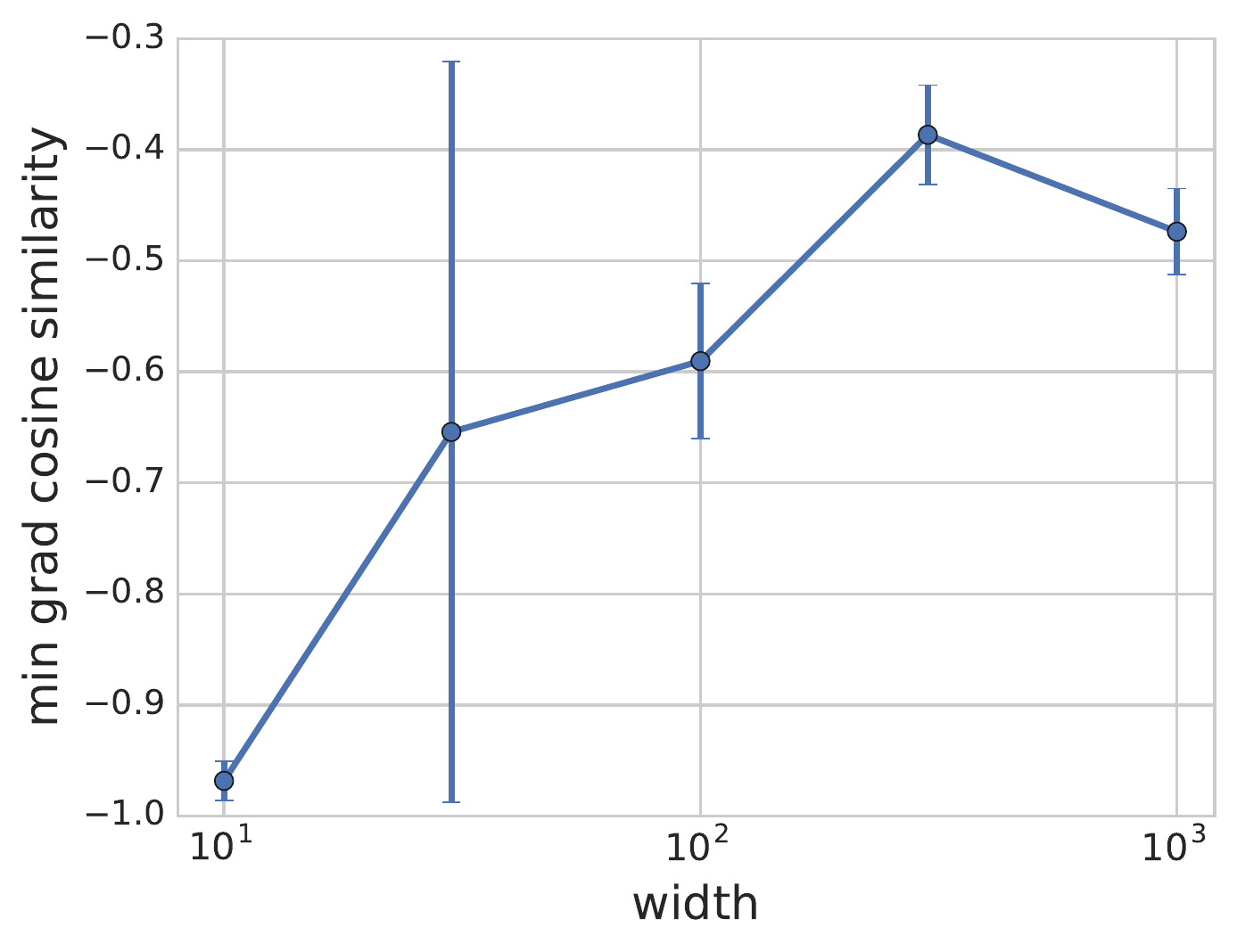}}
\subfigure[Linear NN, LeCun init]{\includegraphics[width=0.24\textwidth]{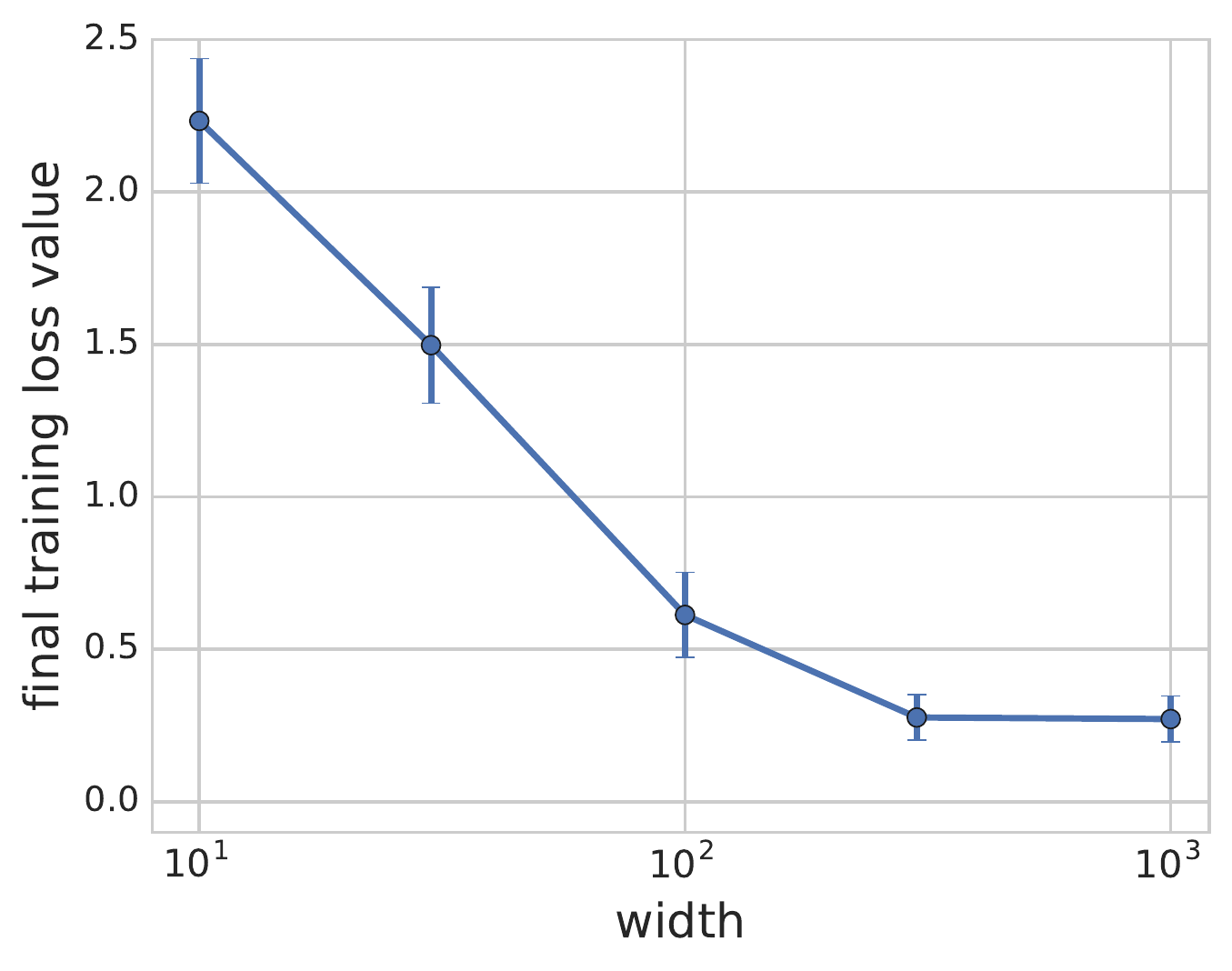}}
\subfigure[Linear NN, LeCun init]{\includegraphics[width=0.24\textwidth]{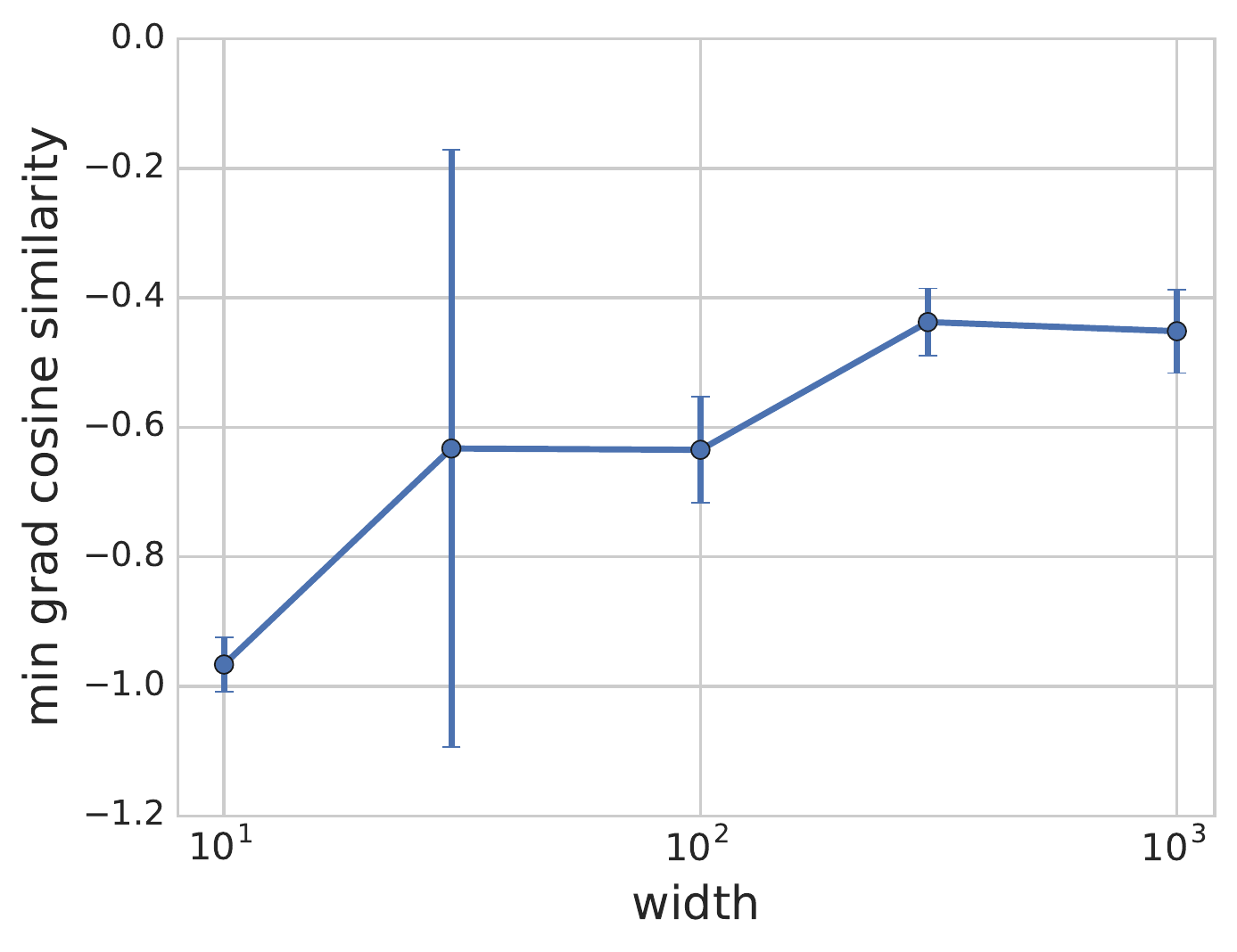}}
\subfigure[Tanh NN, Glorot init]{\includegraphics[width=0.24\textwidth]{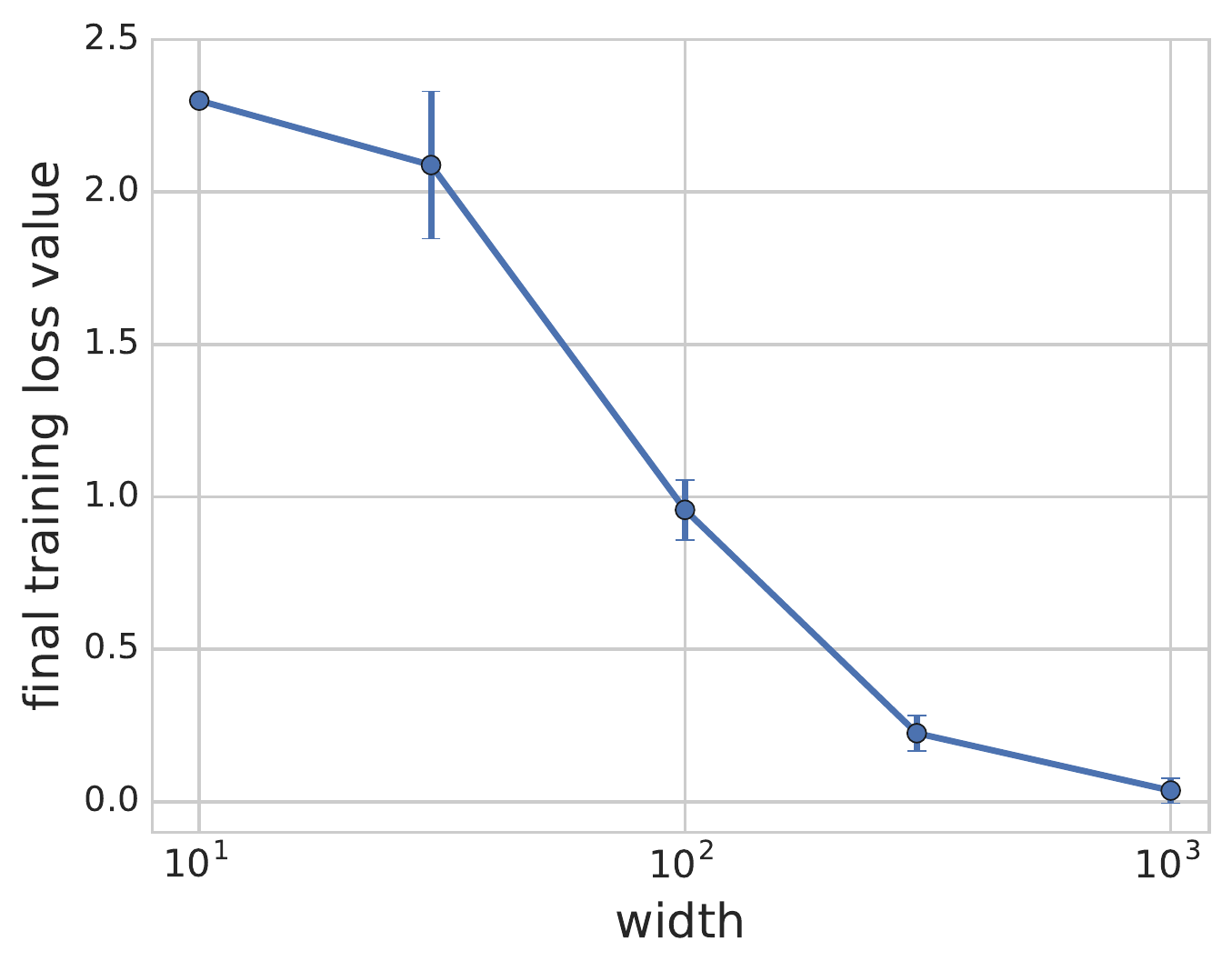}}
\subfigure[Tanh NN, Glorot init]{\includegraphics[width=0.24\textwidth]{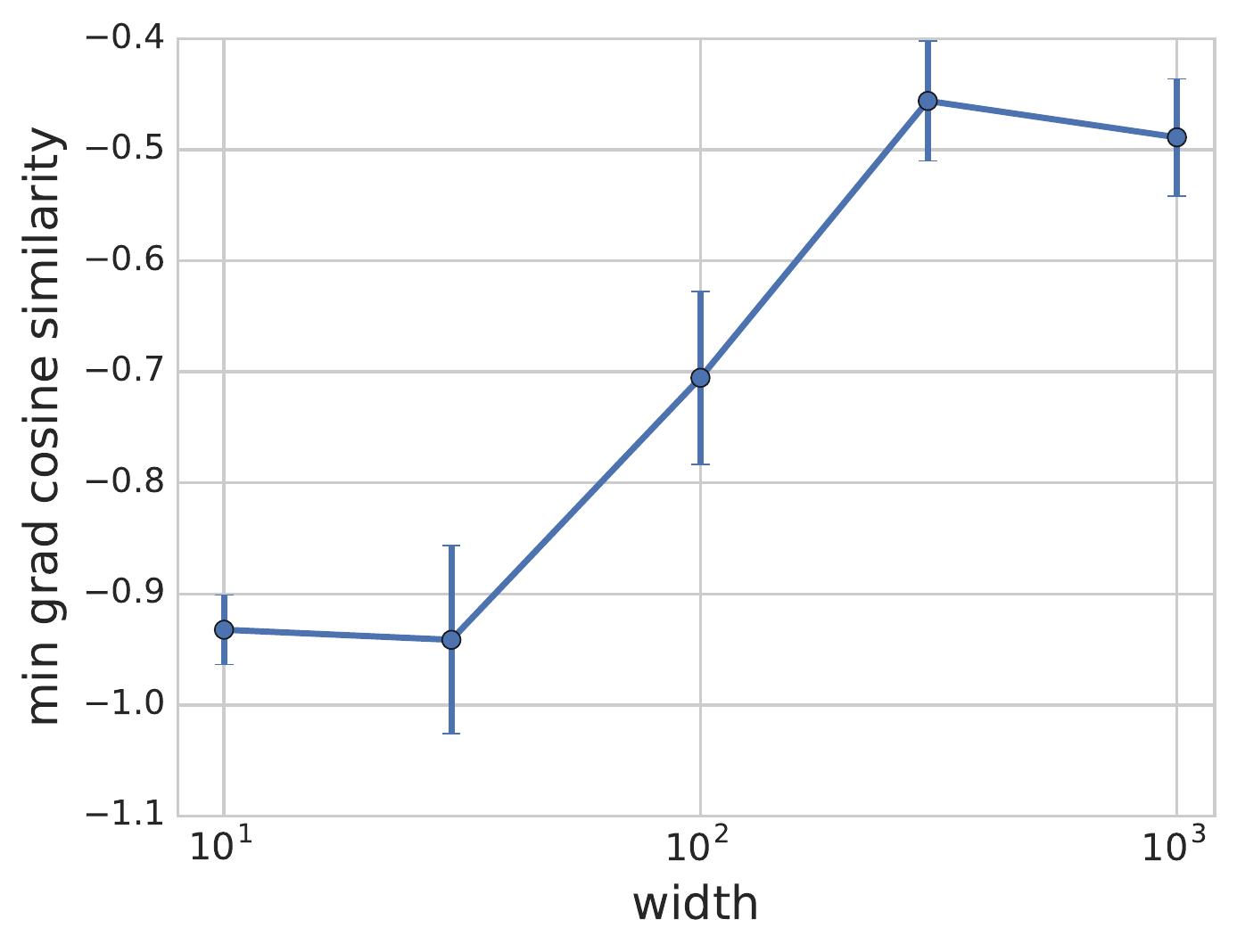}}
\subfigure[Tanh NN, LeCun init]{\includegraphics[width=0.24\textwidth]{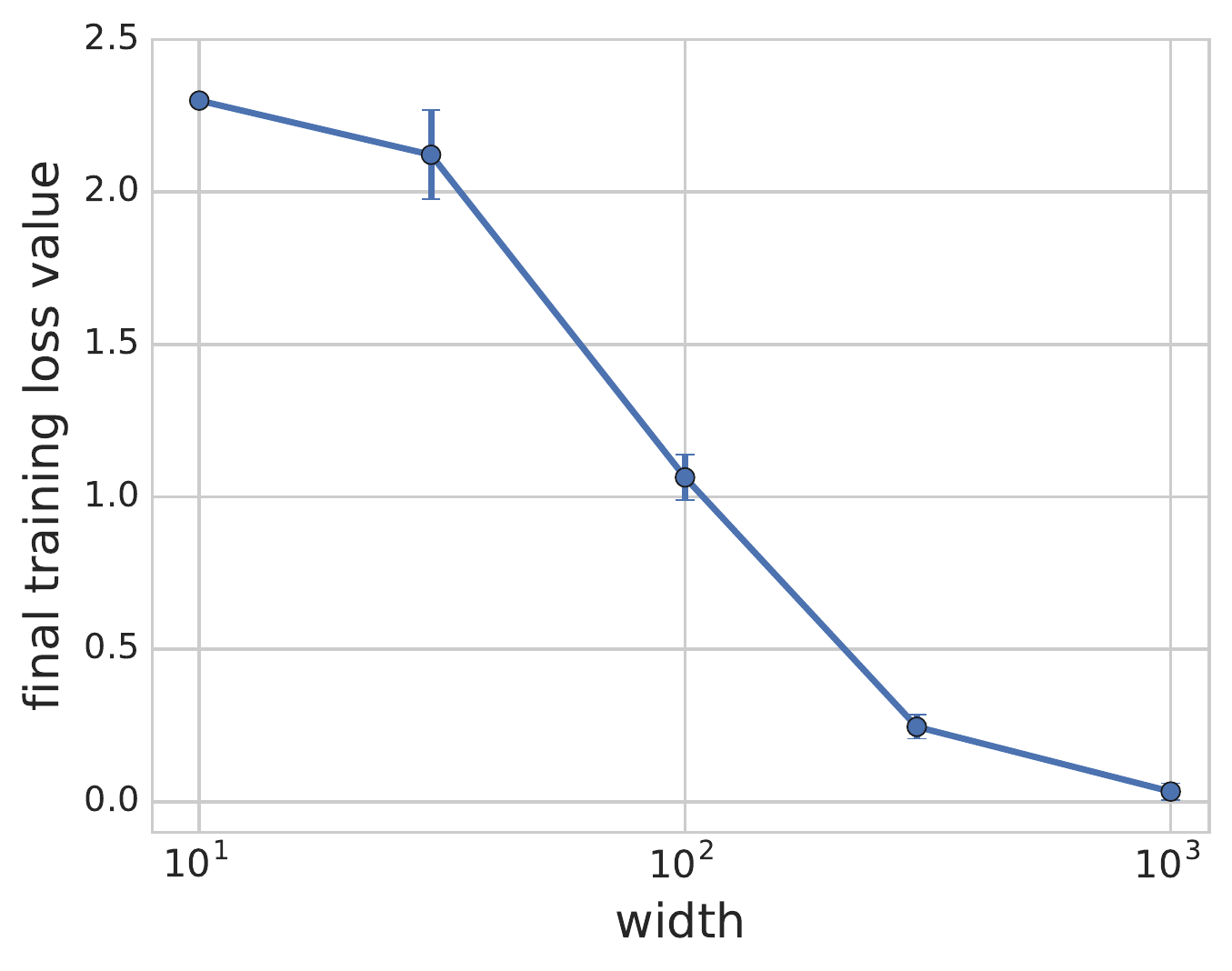}}
\subfigure[Tanh NN, LeCun init]{\includegraphics[width=0.24\textwidth]{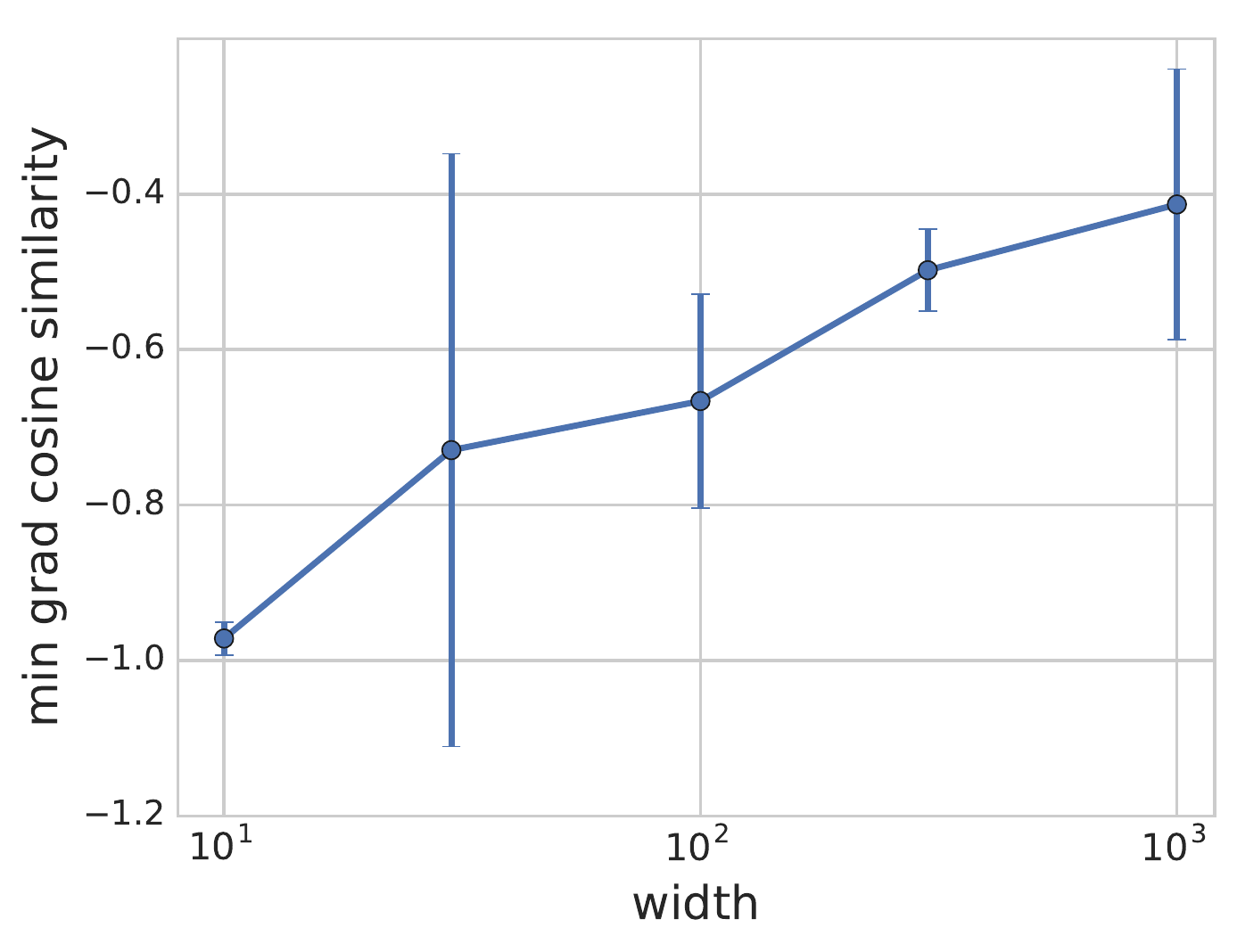}}
\caption{\small{Effect of varying width on MLP-300-$\ell$.}}
\label{fig:width_mnist}
\end{figure*}

\subsection{Additional experimental details for CNNs and WRNs}
\label{sec:supp_exp}

In this section, we review the details of our setup for the image classification experiments on CNNs and WRNs on the CIFAR-10 and CIFAR-100 datasets.

\subsubsection*{Wide residual networks}
The Wide ResNet (WRN) architecture \citep{zagoruyko2016wide} for CIFAR datasets is a stack of three groups of residual blocks. There is a downsampling layer between two blocks, and the number of channels (width of a convolutional layer) is doubled after downsampling.  In the three groups, the width of convolutional layers is $\{16\ell, 32\ell, 64\ell\}$, respectively. 
Each group contains $\beta_r$ residual blocks, and each residual block contains two $3 \times 3$ convolutional layers equipped with ReLU activation, batch normalization and dropout. 
There is a $3 \times 3$ convolutional layer with $16$ channels before the three groups of residual blocks.
And there is a global average pooling, a fully-connected layer and a softmax layer after the three groups. 
The depth of WRN is $\beta = 6 \beta_r + 4$.

For our experiments, we turned off dropout. Unless otherwise specified, we also turned off batch normalization. We added biases to the convolutional layers when not using batch normalization to maintain model expressivity. We used the MSRA initializer \citep{he2015delving} for the weights as is standard for this model, and used the same preprocessing steps for the CIFAR images as described in \citet{zagoruyko2016wide}. This preprocessing step involves normalizing the images and doing data augmentation \citep{zagoruyko2016wide}. We denote this network as WRN-$\beta$-$\ell$, where $\beta$ represents the depth and $\ell$ represents the width factor of the network. 

To study the effect of depth, we considered WRNs with width factor $\ell = 2$ and depth varying as:
$$\beta \in \{16, 22, 28, 34, 40, 52, 76, 100\}.$$
For cleaner figures, we sometimes plot a subset of these results: $\beta \in \{16, 28, 40, 52, 76, 100\}.$ 
To study the effect of width, we considered WRNs with depth $\beta = 16$ and width factor varying as:
$$\ell \in \{2, 3, 4, 5, 6\}.$$

\subsubsection*{Convolutional neural networks}
The WRN architecture contains skip connections that, as we show, help in training deep networks. To consider VGG-like convolutional networks, we consider a family of networks where we remove the skip connections from WRNs. Following the WRN convention, we denote these networks as CNN-$\beta$-$\ell$, where $\beta$ denotes the depth and $\ell$ denotes the width factor.

To study the effect of depth, we considered CNNs with width factor $\ell = 2$ and depth varying as:
$$\beta \in \{16, 22, 28, 34, 40\}.$$
To study the effect of width, we considered CNNs with depth $\beta = 16$ and width factor varying as:
$$\ell \in \{2, 3, 4, 5, 6\}.$$

\subsubsection*{Hyperparameter tuning and other details}
We used SGD as the optimizer without any momentum. Following \citet{zagoruyko2016wide}, we ran all experiments for 200 epochs with minibatches of size 128, and reduced the initial learning rate by a factor of 10 at epochs 80 and 160. We turned off weight decay for all our experiments.

We ran each individual experiment 5 times. We ignored any runs that were unable to decrease the loss from its initial value. We also made sure at least 4 out of the 5 independent runs ran till completion. When the learning rate is close to the threshold at which training is still possible, some runs may converge, while others may fail to converge. Thus, these checks ensure that we pick a learning rate that converges reliably in most cases on each problem. We show the standard deviation across runs in our plots. 

We tuned the optimal initial learning rate for each model over a logarithmically-spaced grid:
$$\alpha \in \{10^1, 3\times10^0, 10^0, 3\times10^{-1}, 10^{-1}, 3\times10^{-2}, 10^{-2}, 3\times10^{-3}, 10^{-3}, 3\times10^{-4}, 10^{-4}, 3\times10^{-5}\},$$
and selected the run that achieved the lowest final training loss value (averaged over the independent runs). Our grid search was such that the optimal learning rate never occurred at one of the extreme values tested. We used the standard train-valid-test splits of 40000-10000-10000 for CIFAR-10 and CIFAR-100.

To measure gradient confusion, at the end of every training epoch, we sampled 100 pairs of mini-batches each of size 128 (the same size as the training batch size). We calculated gradients on each mini-batch, and then computed pairwise cosine similarities. To measure the worse-case gradient confusion, we computed the lowest gradient cosine similarity among all pairs.
We also show the kernel density estimation of the pairwise gradient cosine similarities of the 100 minibatches sampled at the end of training (after 200 epochs), to see the concentration of the distribution. To do this, we combine together the 100 samples for each independent run and then perform kernel density estimation with a gaussian kernel on this data.

\subsection{Additional plots for CIFAR-10 on CNNs}
\label{app:additional_plots}

In section \ref{sec:experiments}, we showed results for image classification using CNNs on CIFAR-10. In this section, we show some additional plots for this experiment. 
Figure \ref{fig:c10_cnn_additional_depth} shows the effect of changing the depth, while figure \ref{fig:c10_cnn_additional_width} shows the effect of changing the width factor of the CNN.
We see that the final training loss and test set accuracy values show the same trends as in section \ref{sec:experiments}: deeper networks are harder to train, while wider networks are easier to train. As mentioned previously, theorems \ref{thm:cosine} and \ref{thm:nonconvex} indicate that we would expect the effect of gradient confusion to be more prominent near the end of training. From the plots we see that deeper networks have higher gradient confusion close to minimum, while wider networks have lower gradient confusion close to the minimum.

\begin{figure*}[t]
\centering
\subfigure[]{\includegraphics[width=0.32\textwidth]{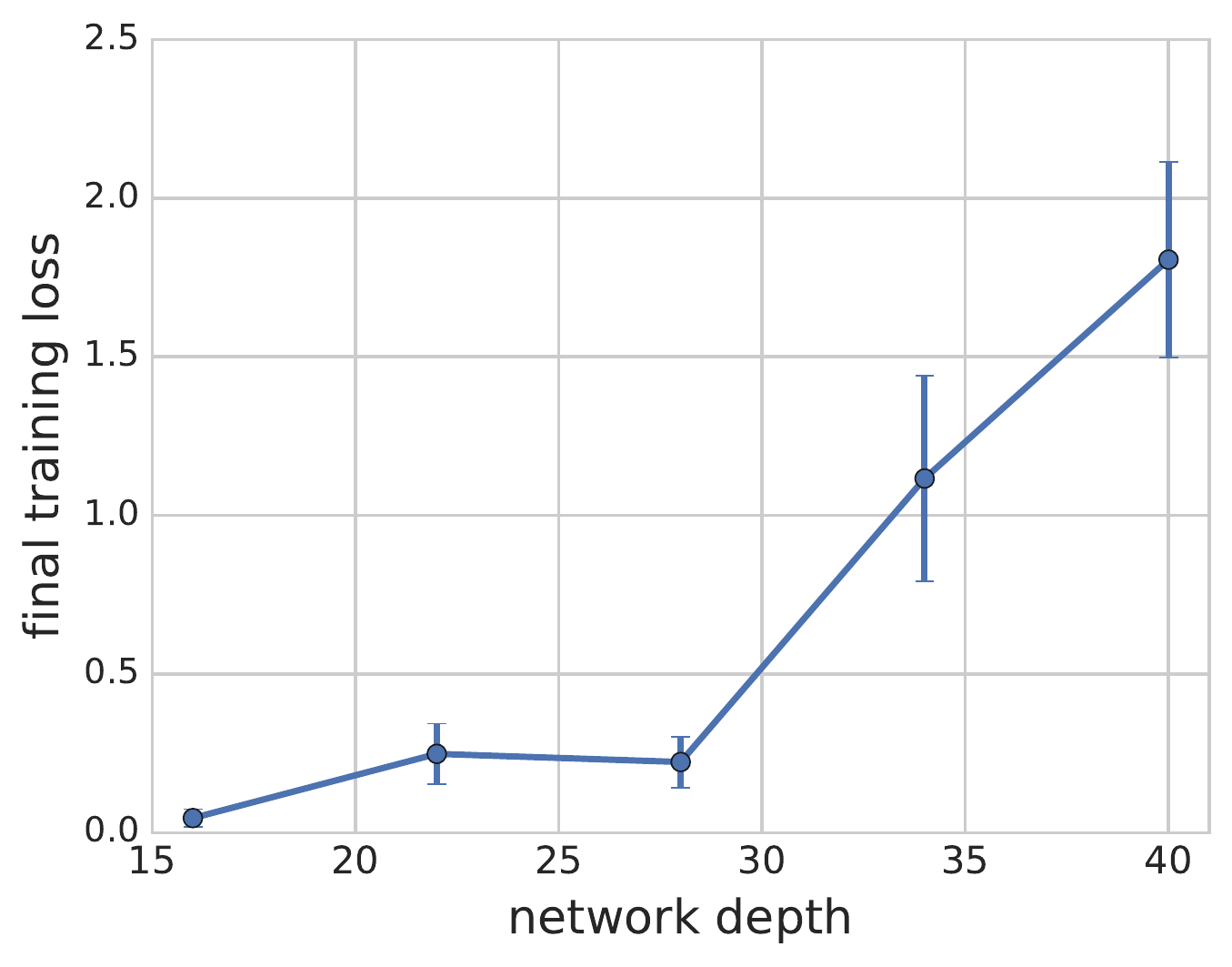}}
\subfigure[]{\includegraphics[width=0.32\textwidth]{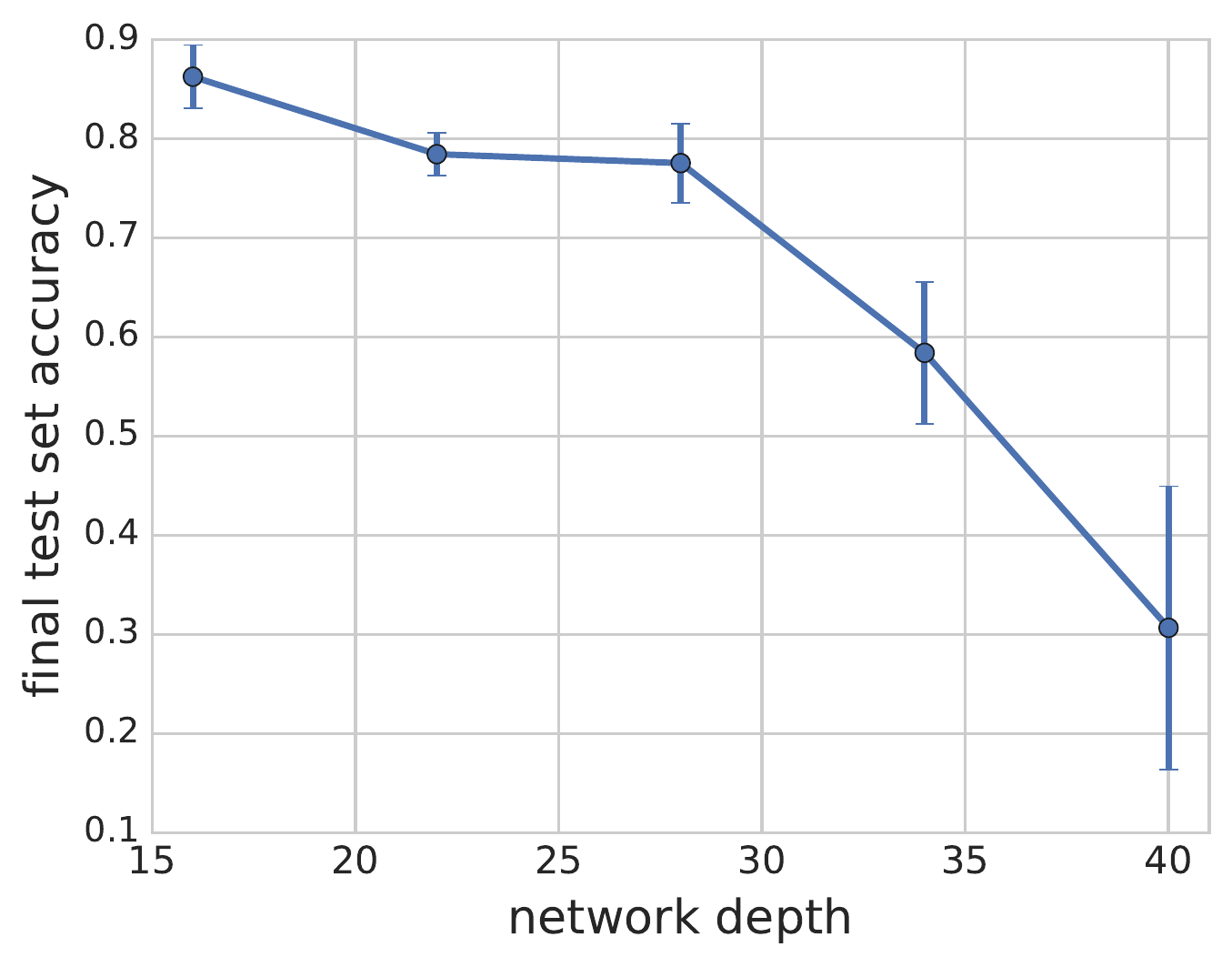}}
\subfigure[]{\includegraphics[width=0.32\textwidth]{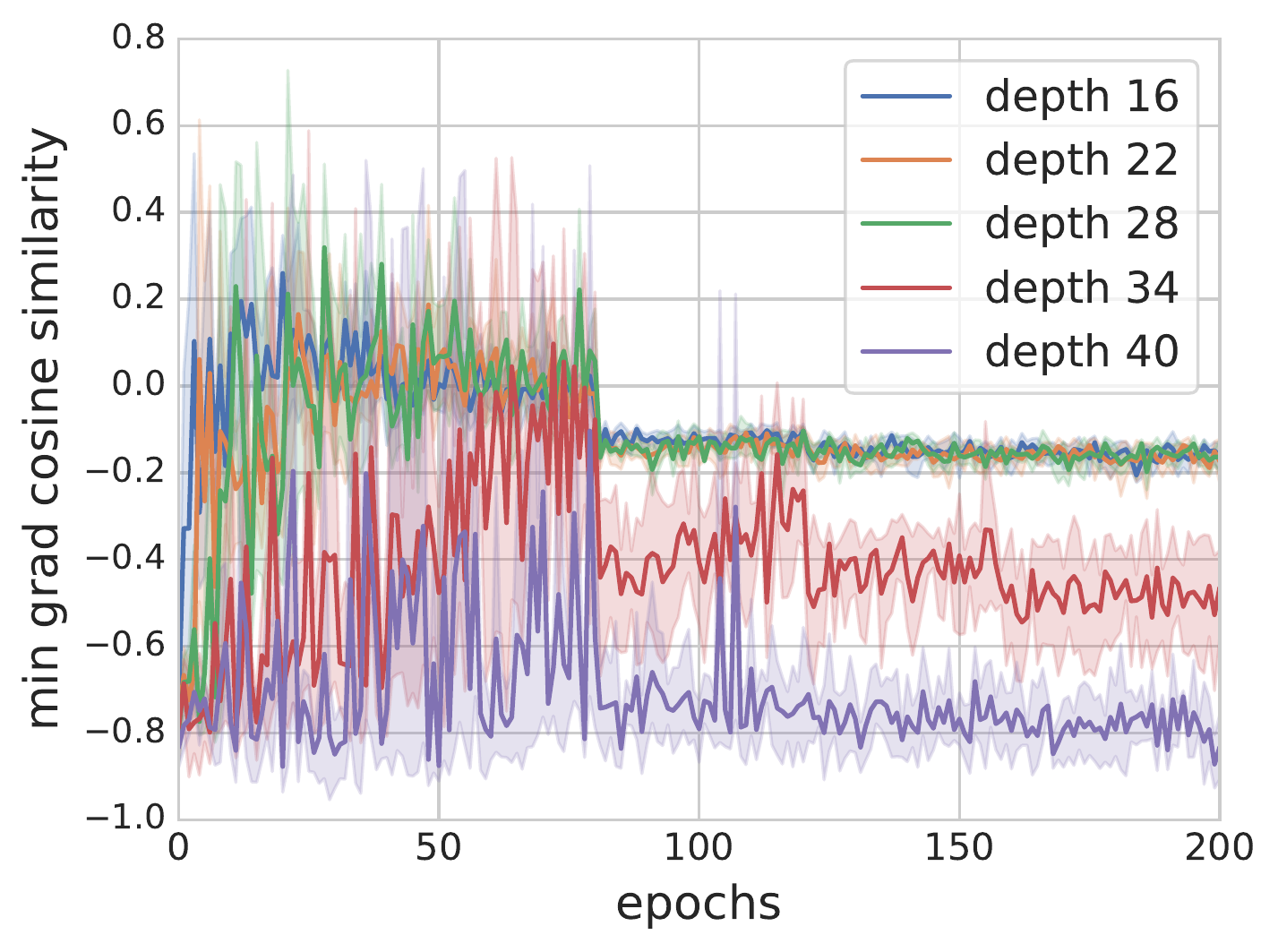}}
      \caption{The effect of network depth with CNN-$\beta$-2 on CIFAR-10. The plots show the
(a) final training loss values at the end of training, (b) final test set accuracy values at the end of training, and (c) the minimum of pairwise gradient cosine similarities during training.}
\label{fig:c10_cnn_additional_depth}
\end{figure*}

\begin{figure*}[t]
\centering
\subfigure[]{\includegraphics[width=0.32\textwidth]{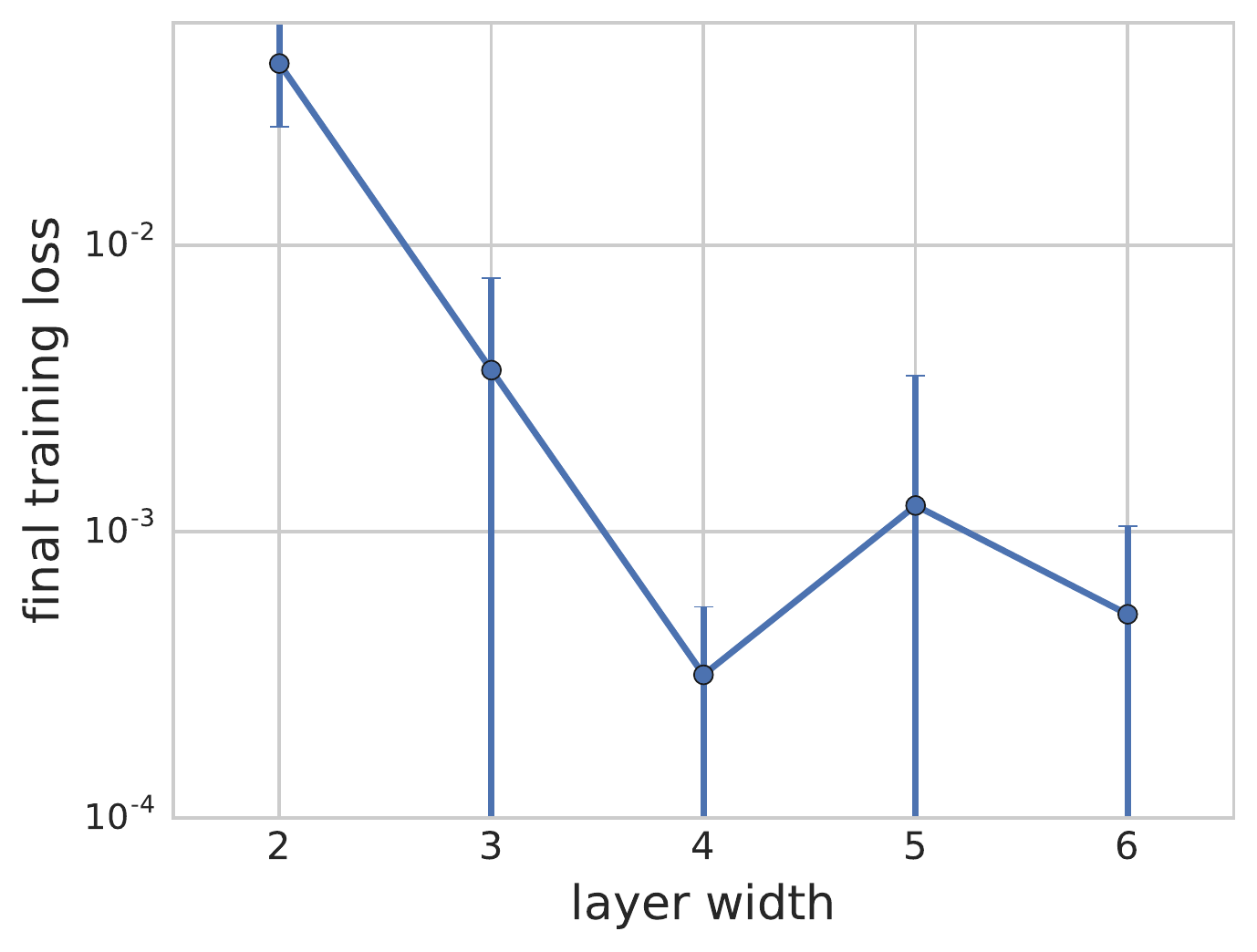}}
\subfigure[]{\includegraphics[width=0.32\textwidth]{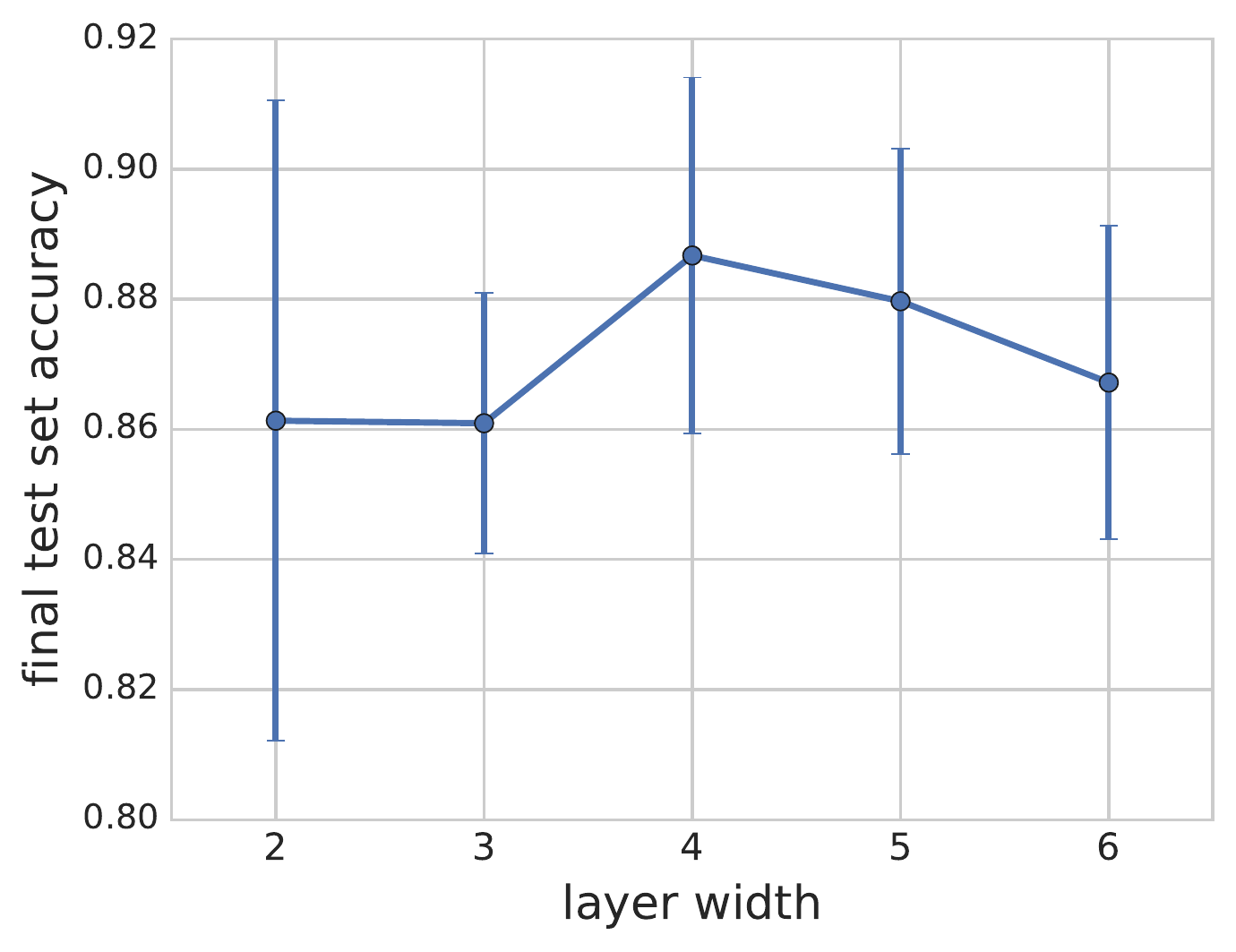}}
\subfigure[]{\includegraphics[width=0.32\textwidth]{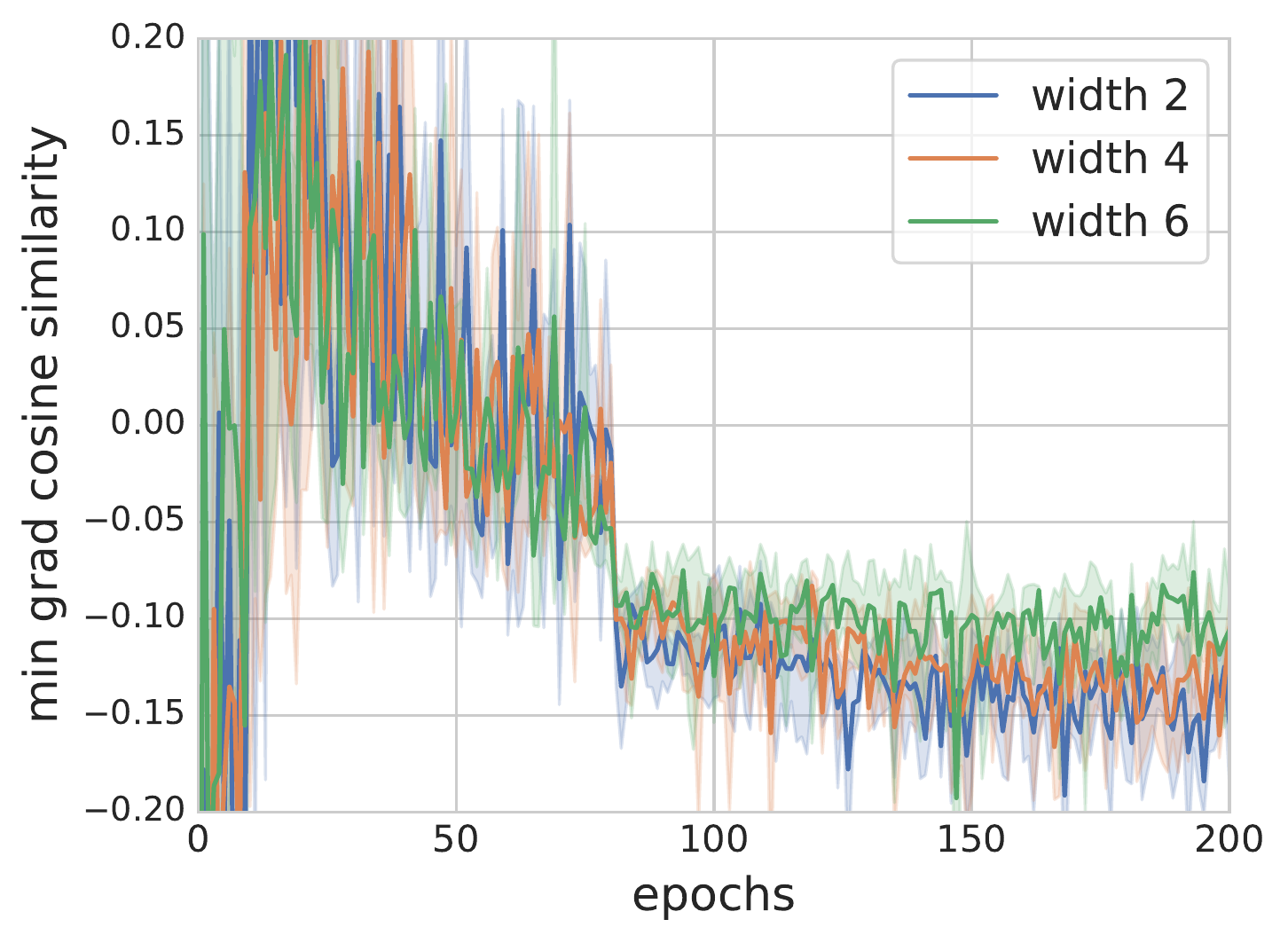}}
\caption{The effect of width with CNN-16-$\ell$ on CIFAR-10. The plots show
the (a) final training loss values at the end of training, (b) final test set accuracy values at the end of training, and the (c) minimum of pairwise gradient cosine similarities during training.}
\label{fig:c10_cnn_additional_width}
\end{figure*}

\subsection{CIFAR-100 on CNNs}
We now consider image classifications tasks with CNNs on the CIFAR-100 dataset. Figure \ref{fig:c100_cnn_depth} shows the effect of varying depth, while figure \ref{fig:c100_cnn_width} shows the effect of varying width. We notice the same trends as in our results with CNNs on CIFAR-10. Interestingly, from the width results in figure \ref{fig:c100_cnn_width}, we see that while there is no perceptible change to the minimum pairwise gradient cosine similarity, the distribution still sharply concentrates around 0 with increasing width. Thus more gradients become orthogonal to each other with increasing width, implying that SGD on very wide networks becomes closer to decoupling over the data samples.

\begin{figure*}[t]
\centering
\subfigure[]{\includegraphics[width=0.32\textwidth]{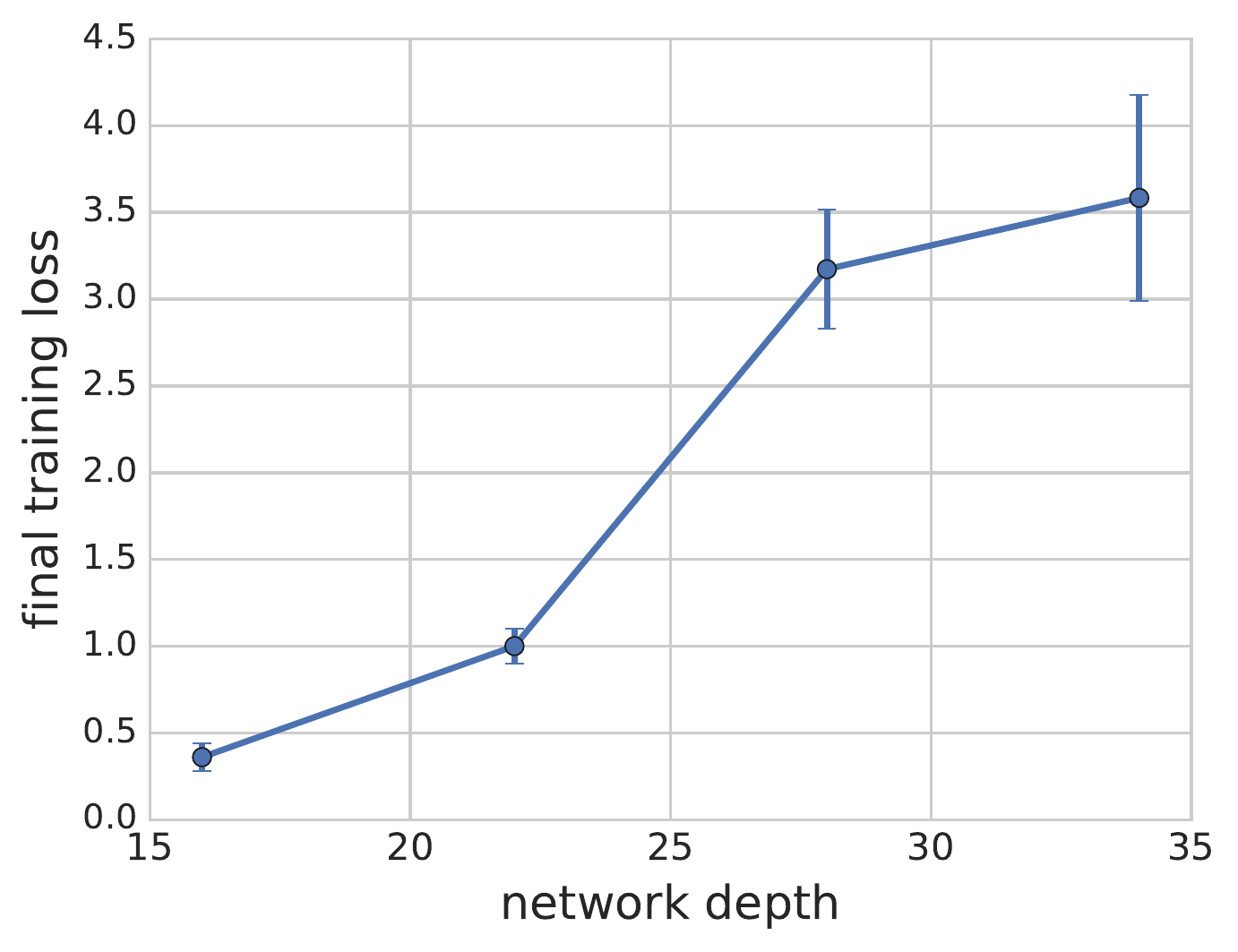}}
\subfigure[]{\includegraphics[width=0.32\textwidth]{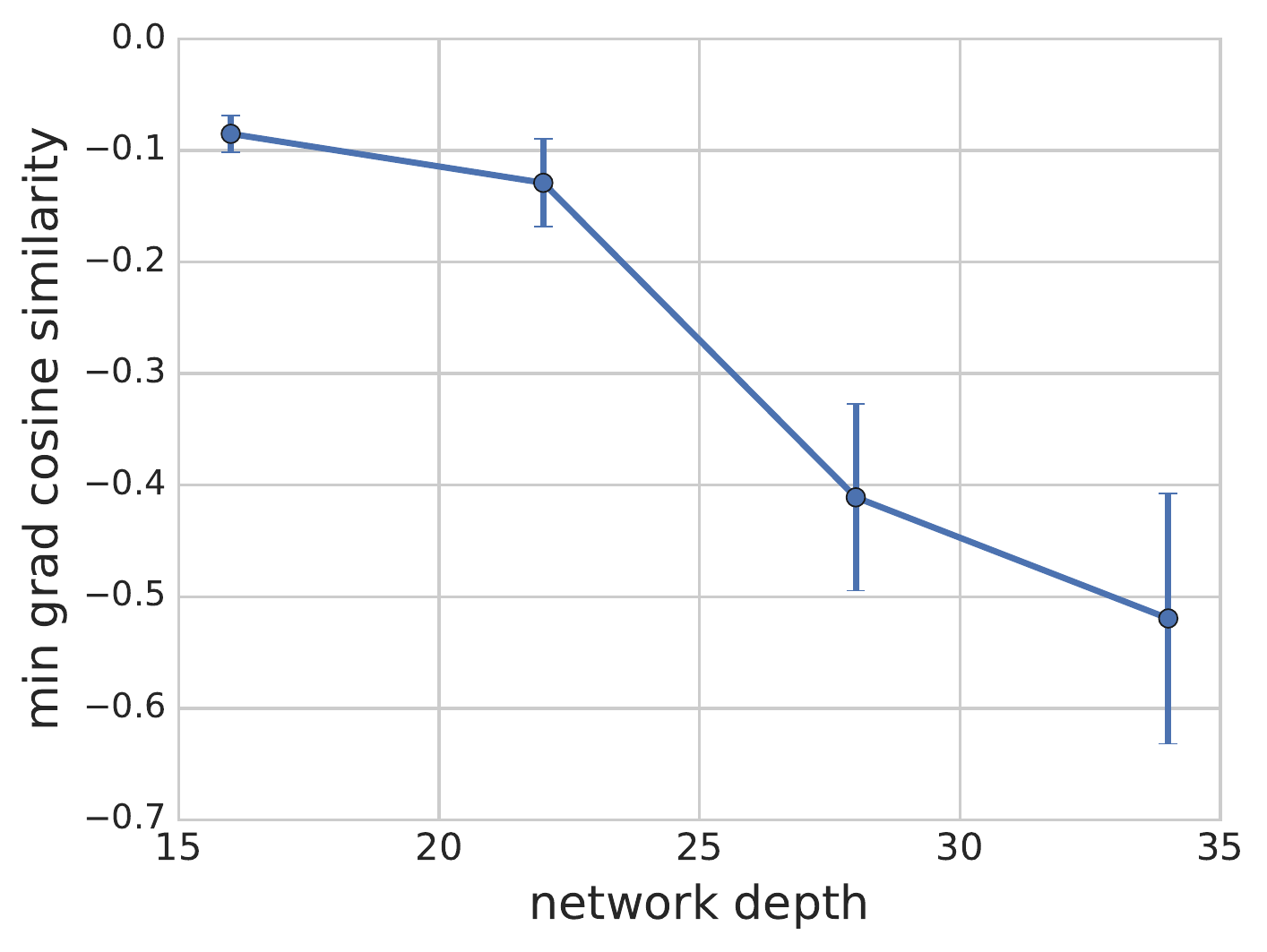}}
\subfigure[]{\includegraphics[width=0.32\textwidth]{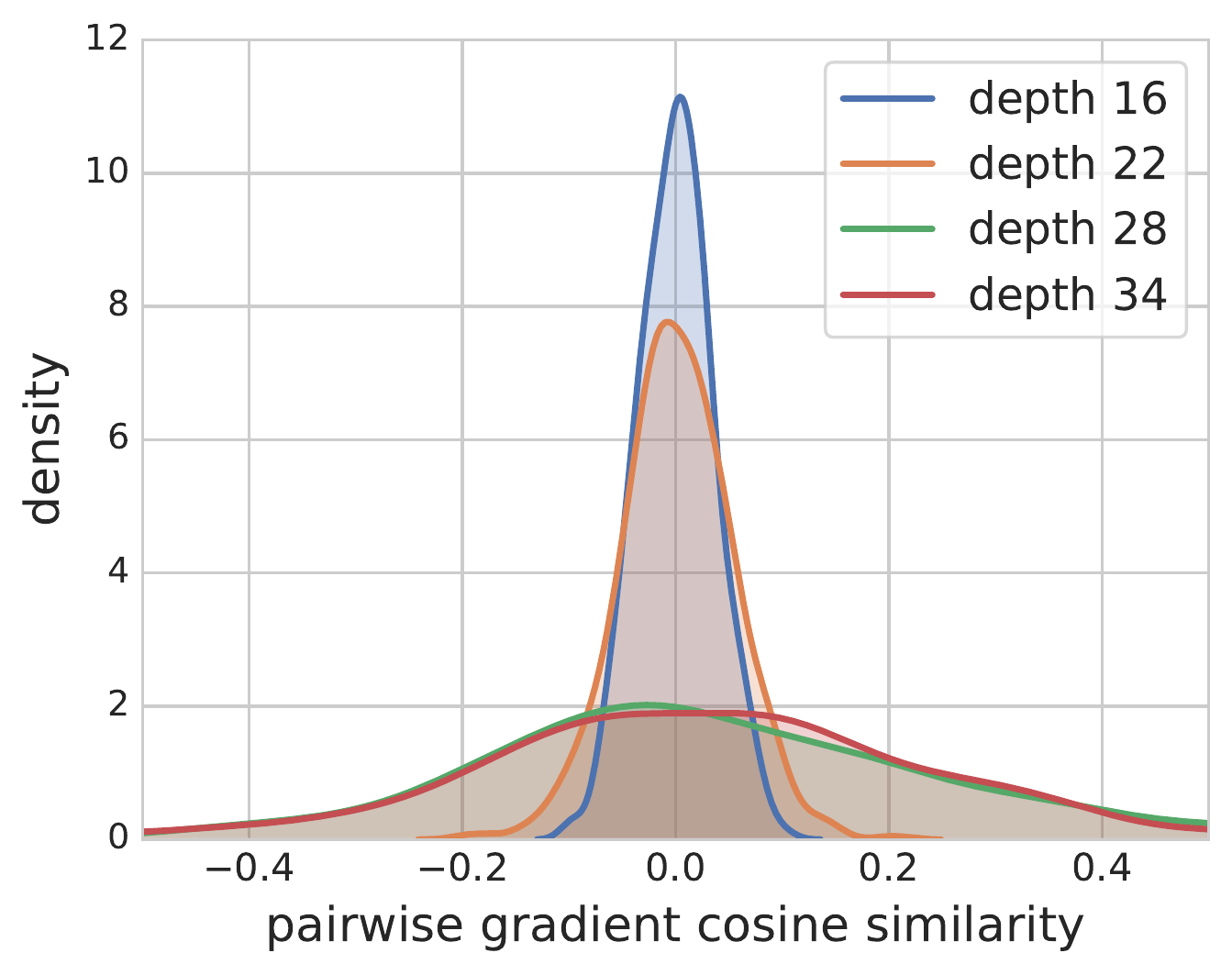}}
\caption{The effect of network depth with CNN-$\beta$-2 on CIFAR-100. The plots show the (a) training loss values at the end of training, (b) minimum of pairwise gradient cosine similarities at the end of training, and the (c) kernel density estimate of the pairwise gradient cosine similarities at the end of training.}
\label{fig:c100_cnn_depth}
\end{figure*}

\begin{figure*}[t]
\centering
\subfigure[]{\includegraphics[width=0.32\textwidth]{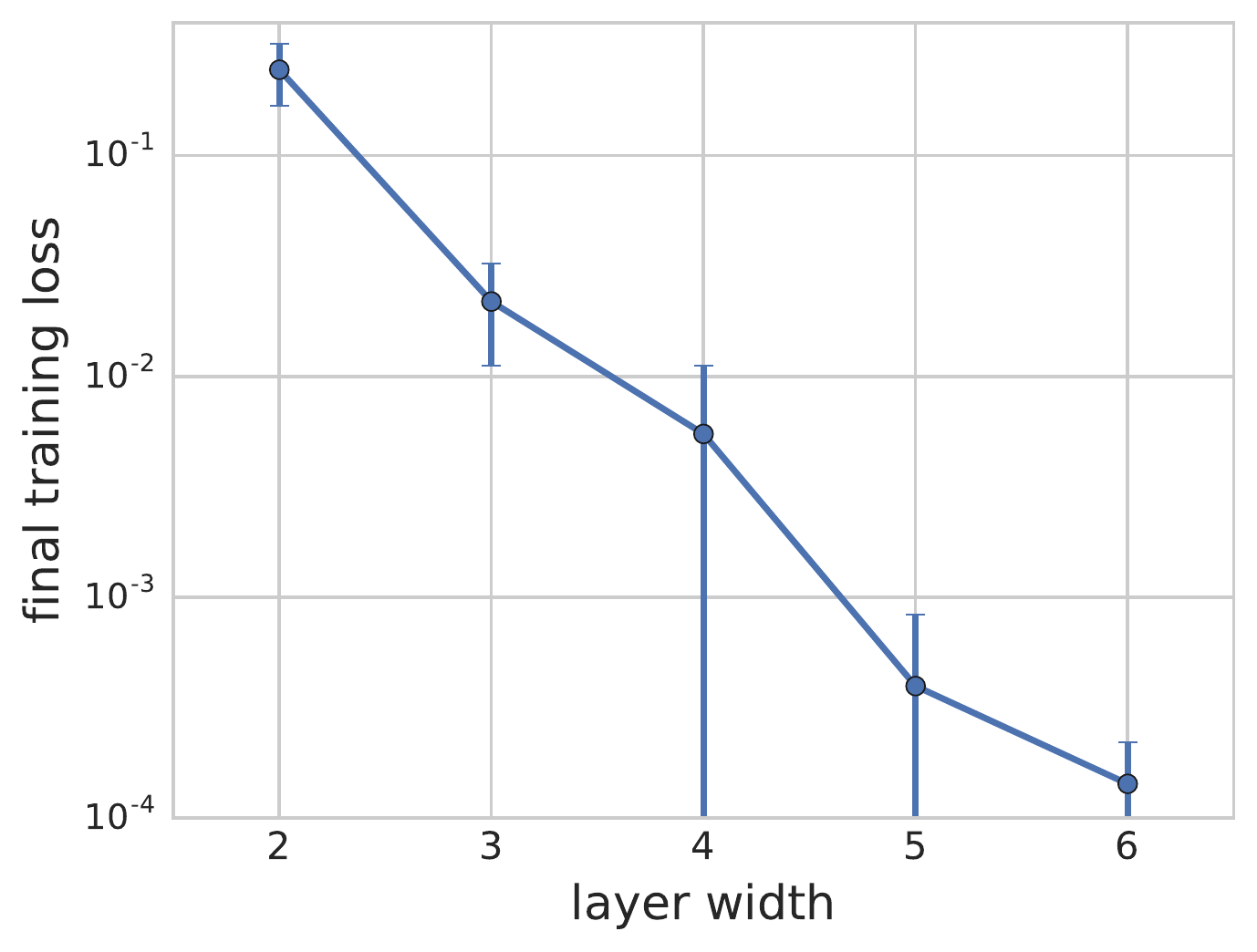}}
\subfigure[]{\includegraphics[width=0.32\textwidth]{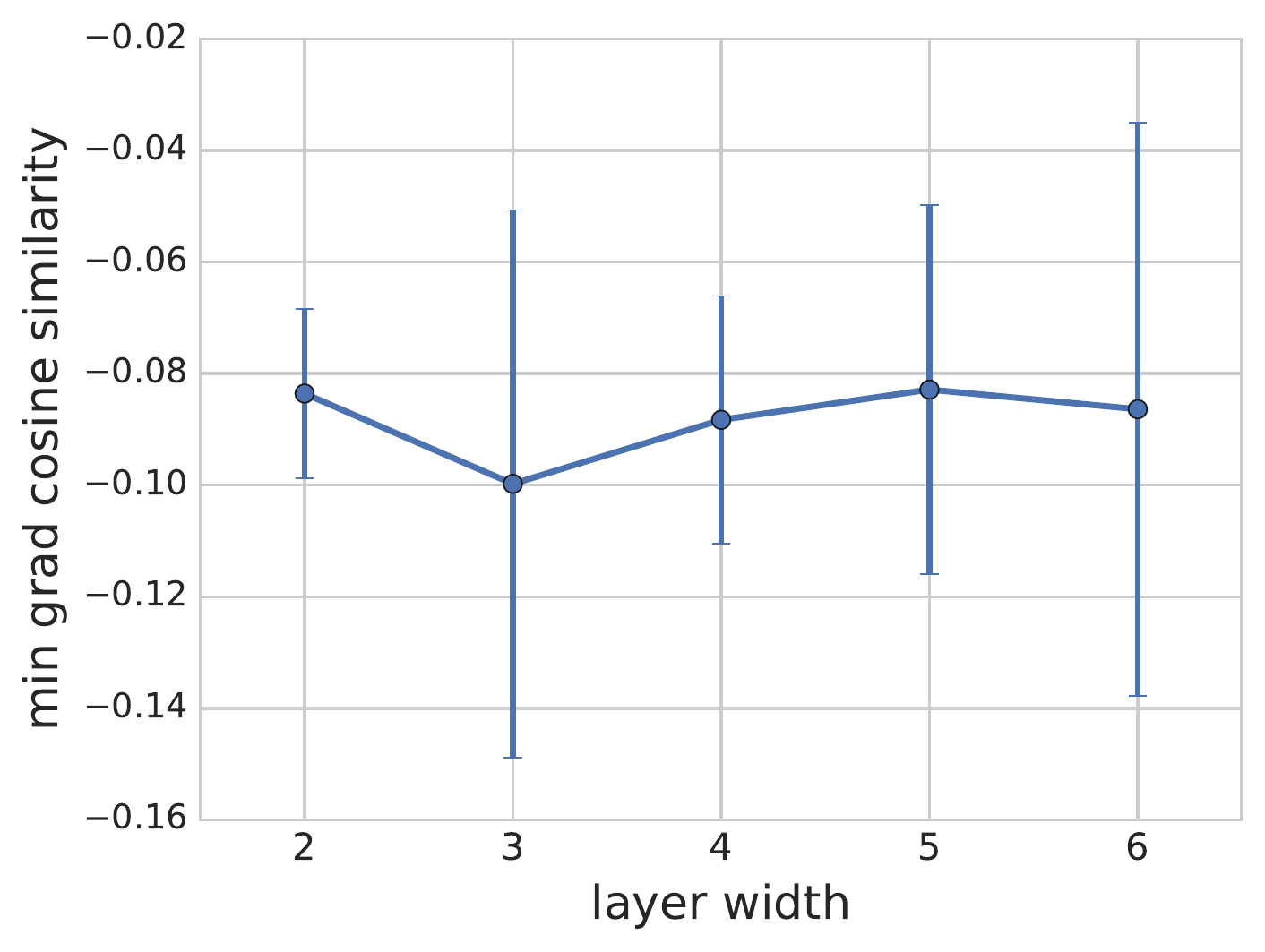}}
\subfigure[]{\includegraphics[width=0.32\textwidth]{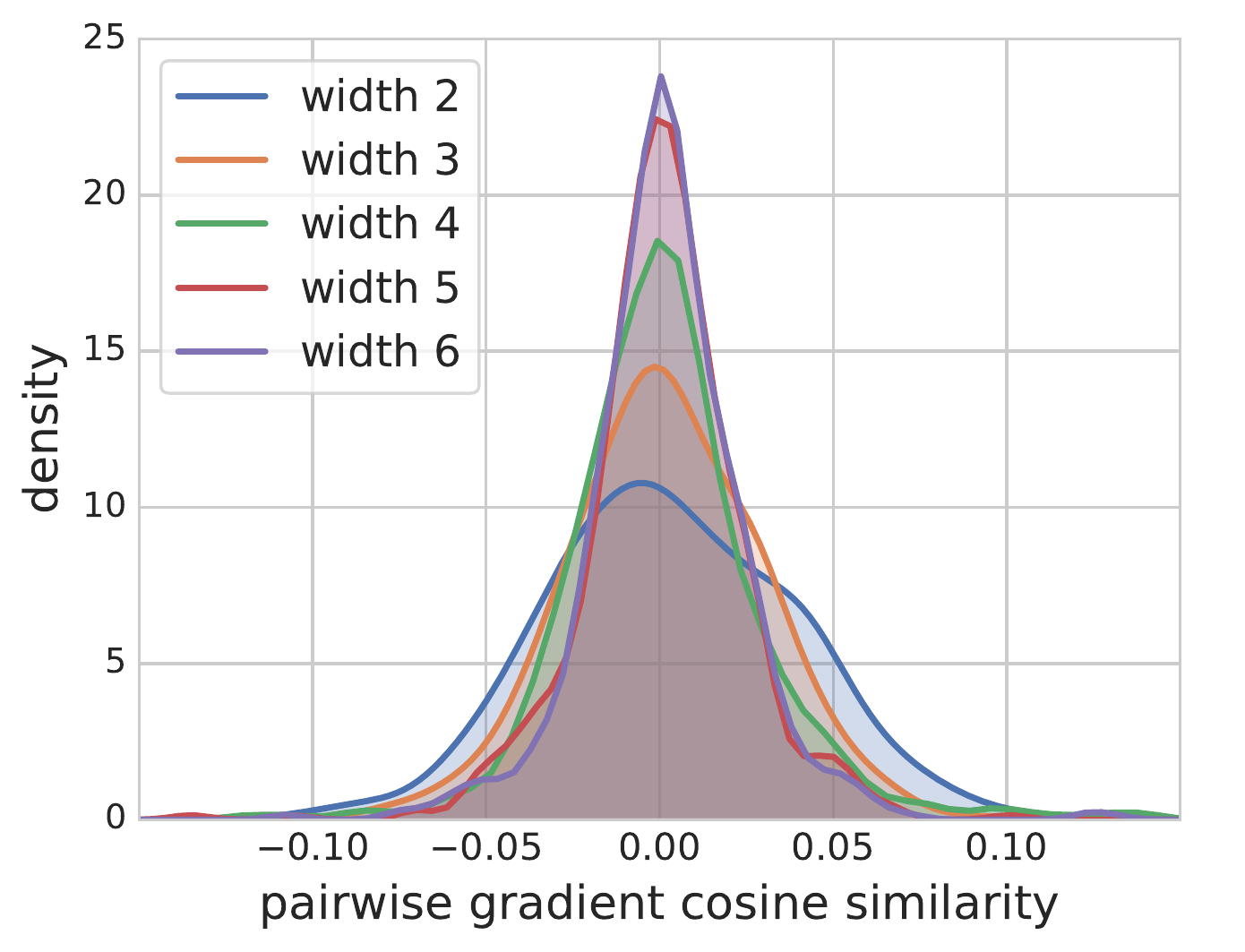}}
\caption{The effect of width with CNN-16-$\ell$ on CIFAR-100. The plots show the (a) training loss values at the end of training, (b) minimum of pairwise gradient cosine similarities at the end of training, and the (c) kernel density estimate of the pairwise gradient cosine similarities at the end of training.}
\label{fig:c100_cnn_width}
\end{figure*}

\subsection{Image classification with WRNs on CIFAR-10 and CIFAR-100}

We now show results for image classification problems using wide residual networks (WRNs) on CIFAR-10 and CIFAR-100. The WRNs we consider do not have any batch normalization. Later we show results on the effect of adding batch normalization to these networks.

Figures \ref{fig:c10_wrs_depth} and \ref{fig:c100_wrs_depth} show results on the effect of depth using WRNs on CIFAR-10 and CIFAR-100 respectively. We again see the consistent trend of deeper networks having higher gradient confusion, making them harder to train. We further see that increasing depth results in the pairwise gradient cosine similarities concentrating less around 0.

Figures \ref{fig:c10_wrs_width} and \ref{fig:c100_wrs_width} show results on the effect of width using WRNs on CIFAR-10 and CIFAR-100 respectively. We see that increasing width typically lowers gradient confusion and helps the network achieve lower loss values. The pairwise gradient cosine similarities also typically concentrate around 0 with higher width. We also notice from these figures that in some cases, increasing width might lead to diminishing returns, i.e., the benefits of increased width diminish after a certain point, as one would expect.

\begin{figure*}[t]
\centering
\subfigure[]{\includegraphics[width=0.32\textwidth]{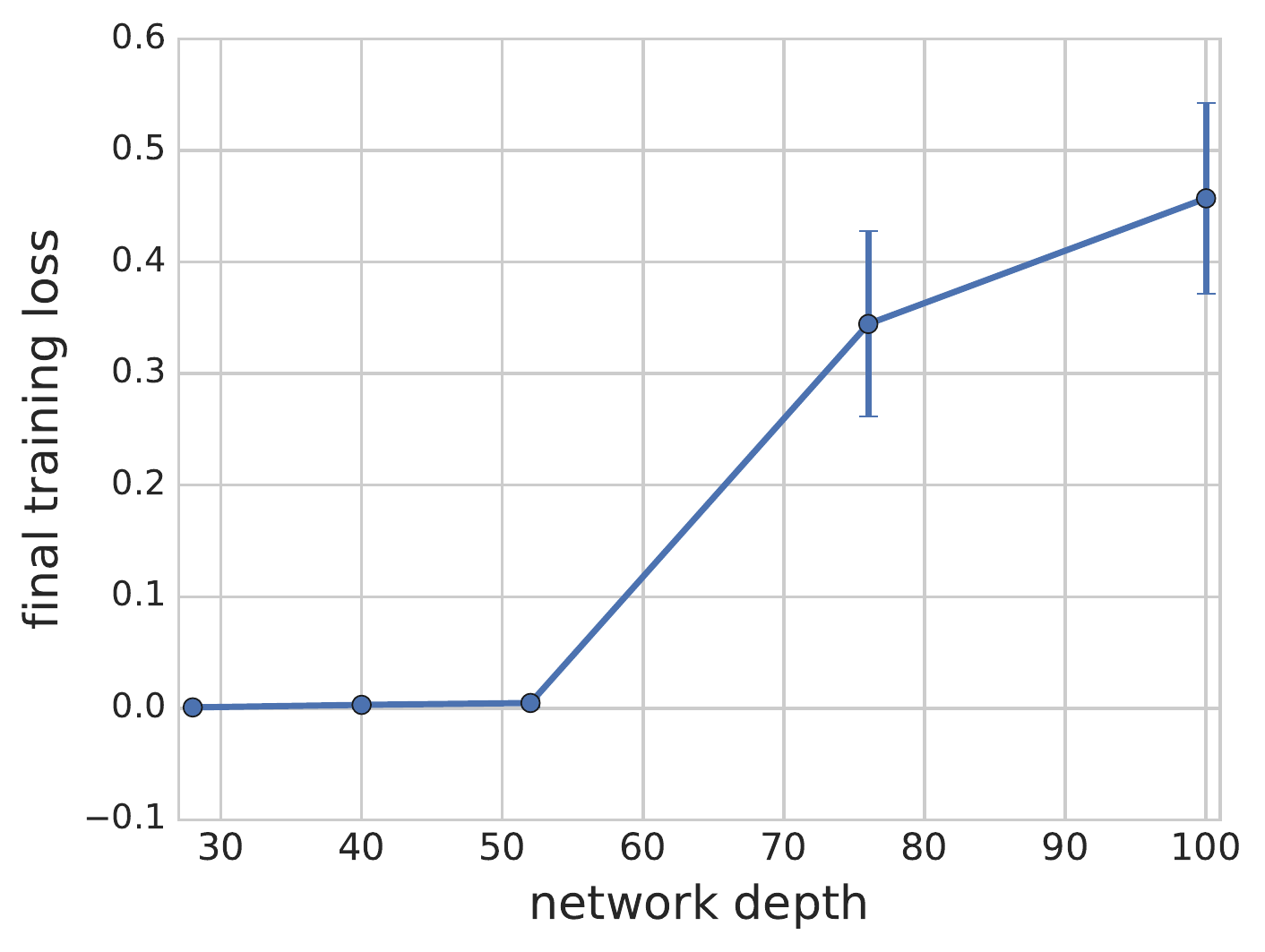}}
\subfigure[]{\includegraphics[width=0.32\textwidth]{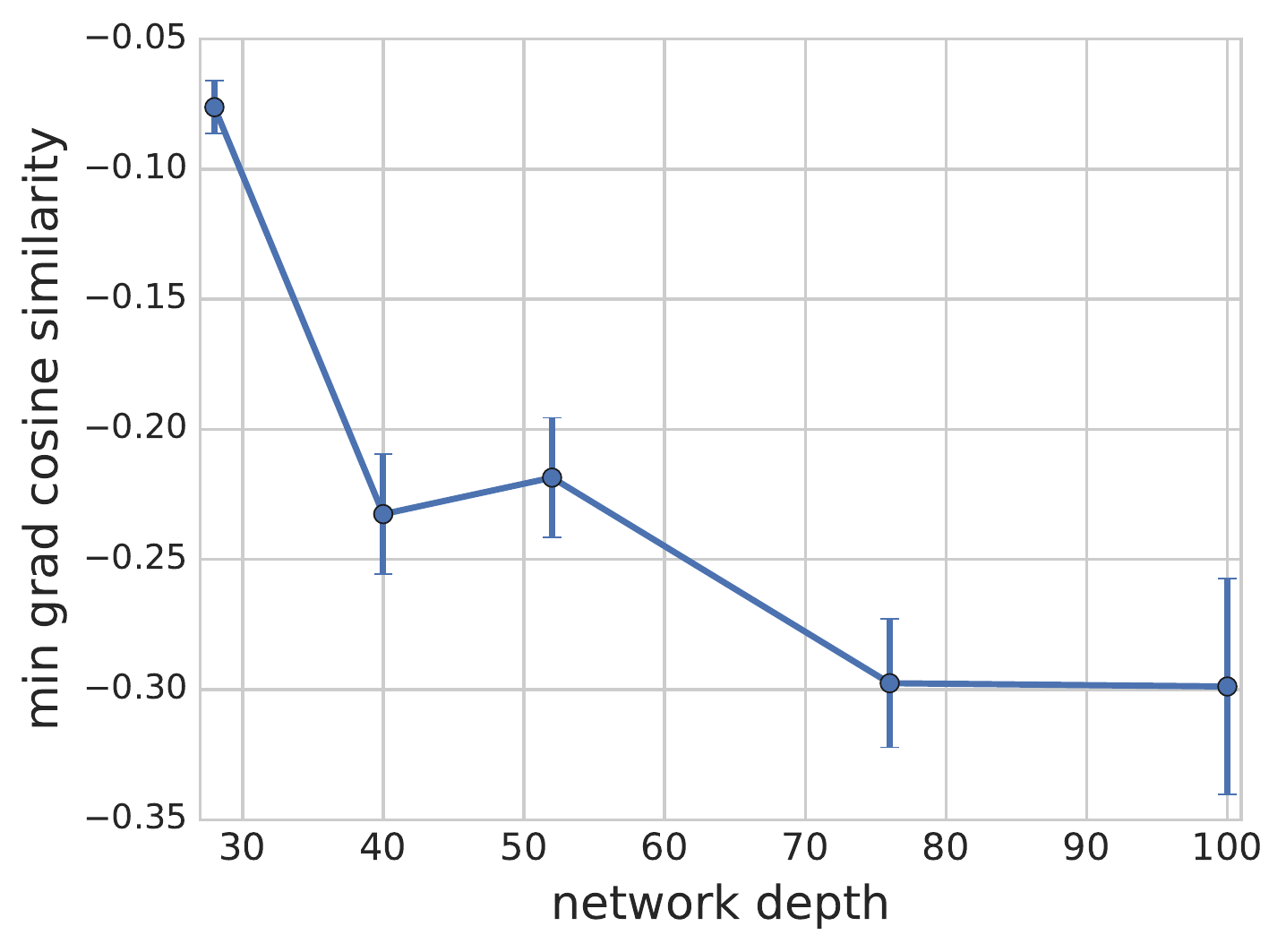}}
\subfigure[]{\includegraphics[width=0.32\textwidth]{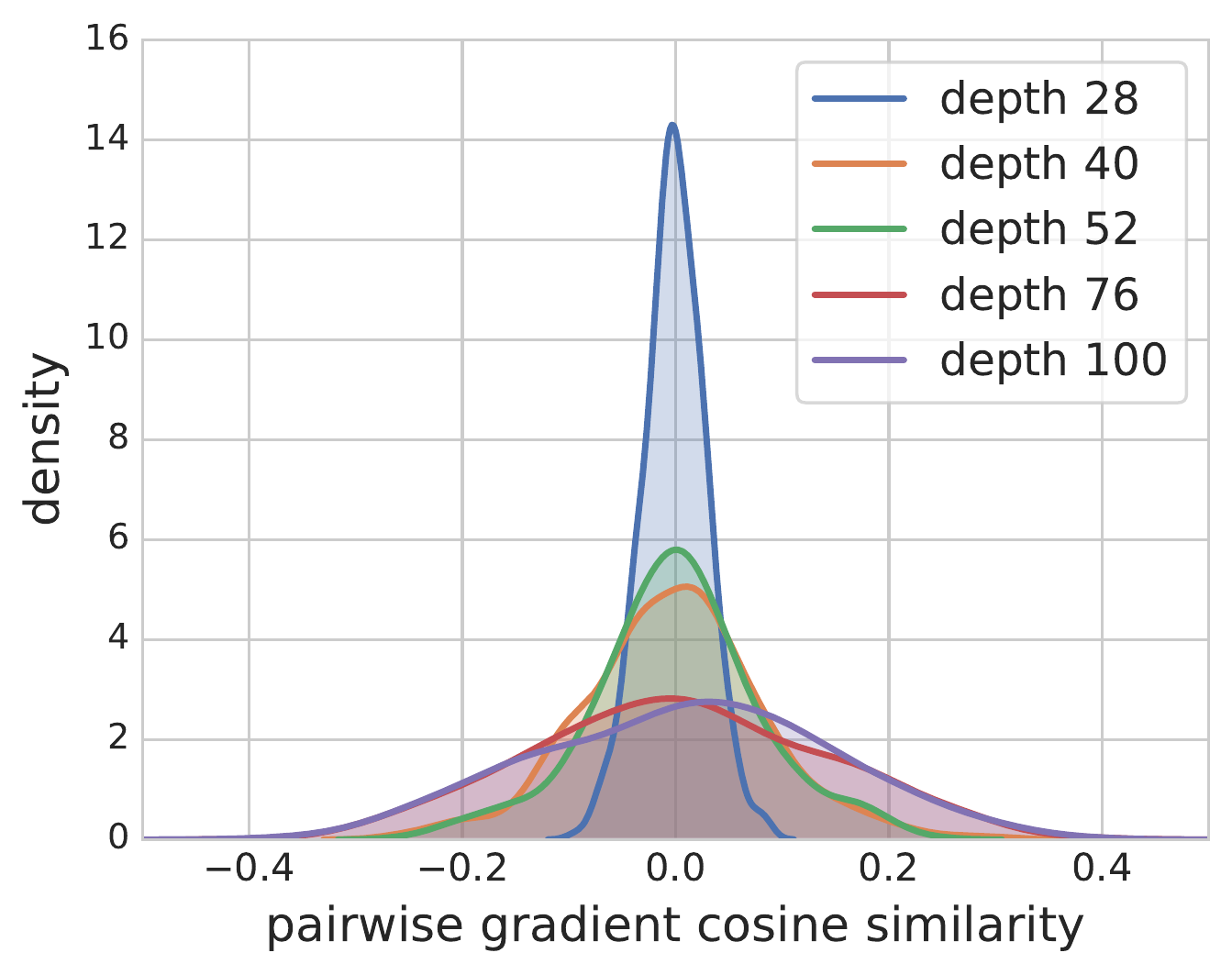}}
\caption{The effect of depth with WRN-$\beta$-2 (no batch normalization) on CIFAR-10. The plots show the (a) training loss values at the end of training, (b) minimum of pairwise gradient cosine similarities at the end of training, and the (c) kernel density estimate of the pairwise gradient cosine similarities at the end of training.}
\label{fig:c10_wrs_depth}
\end{figure*}

\begin{figure*}[t]
\centering
\subfigure[]{\includegraphics[width=0.32\textwidth]{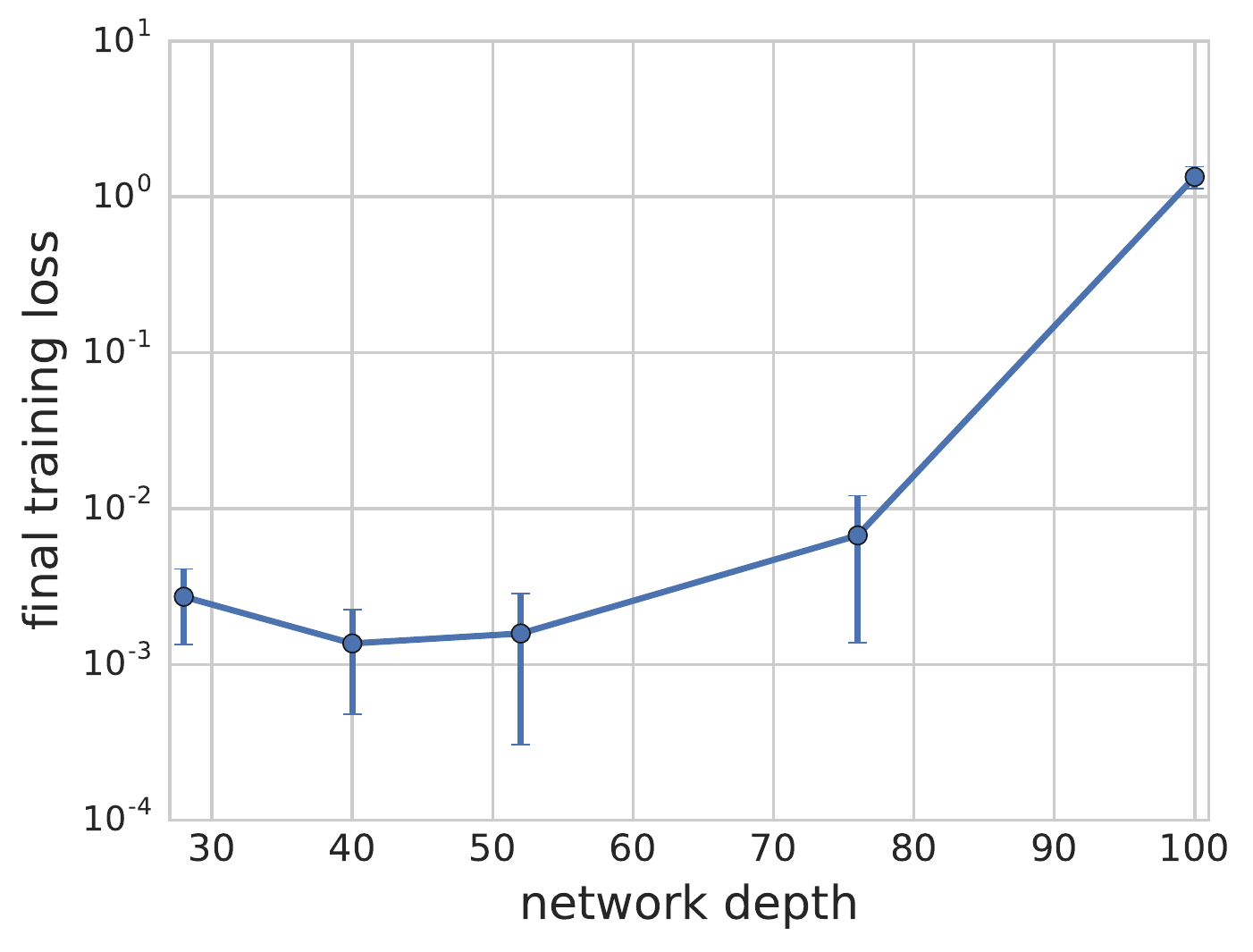}}
\subfigure[]{\includegraphics[width=0.32\textwidth]{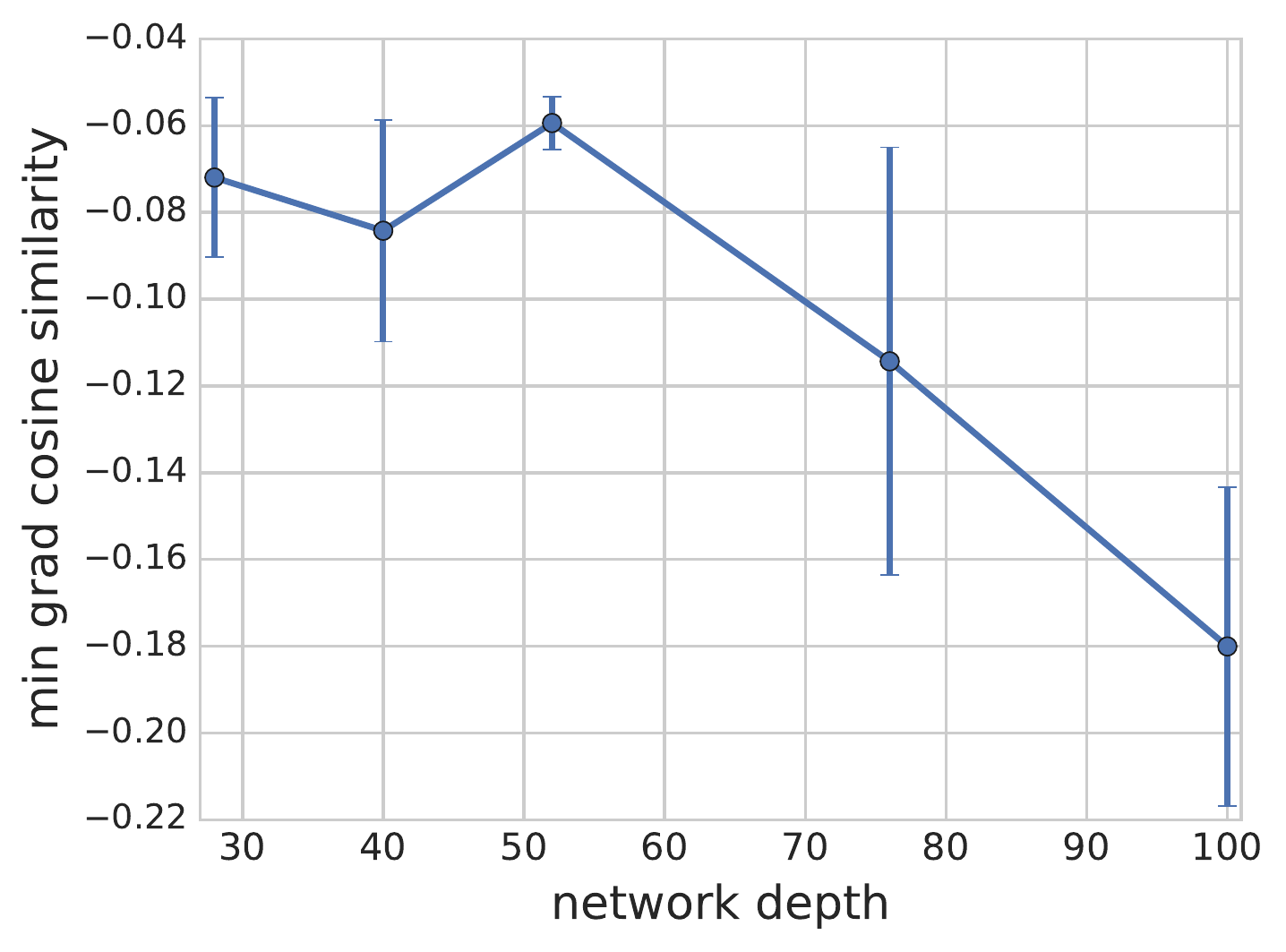}}
\subfigure[]{\includegraphics[width=0.32\textwidth]{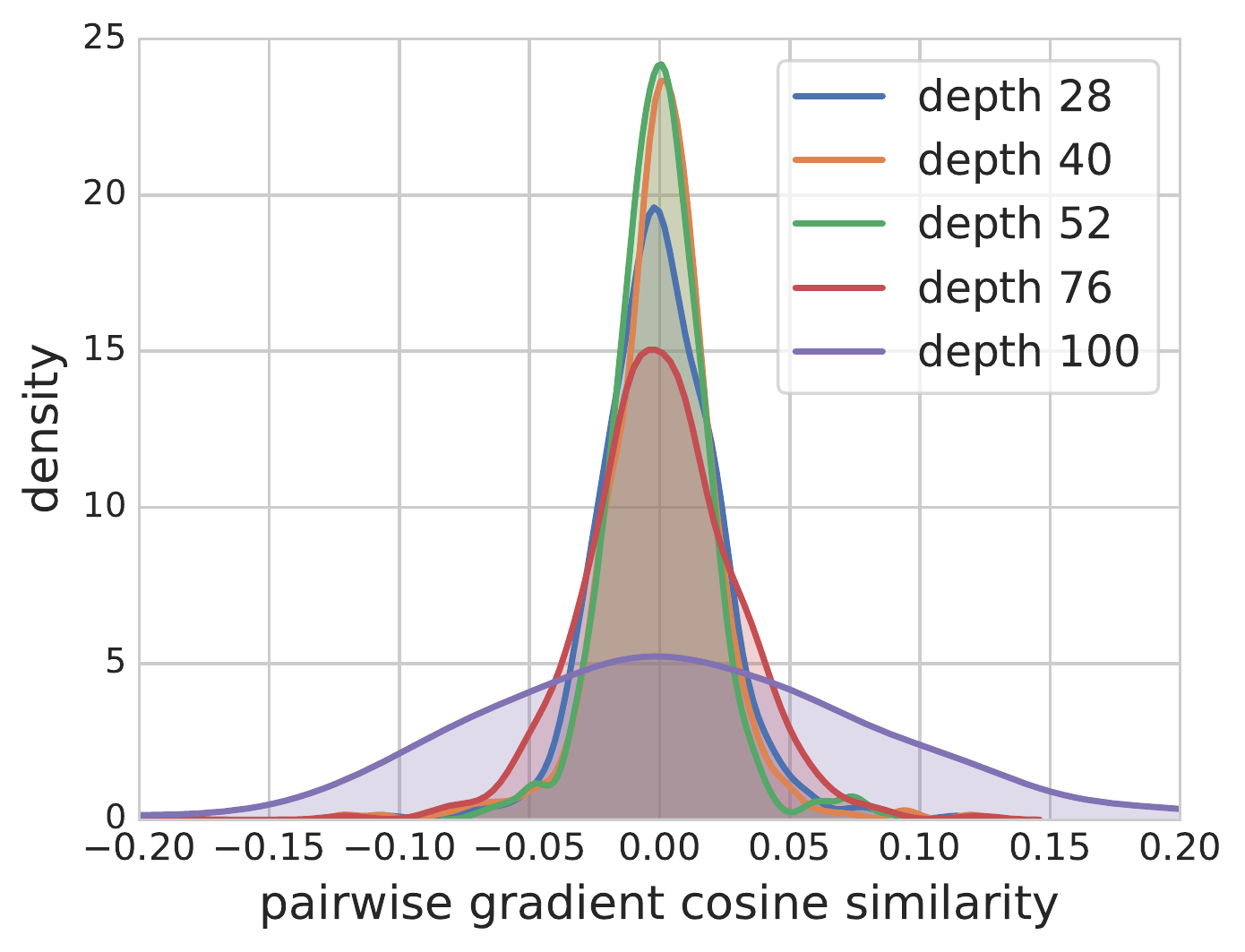}}
\caption{The effect of depth with WRN-$\beta$-2  (no batch normalization) on CIFAR-100. The plots show the (a) training loss values at the end of training, (b) minimum of pairwise gradient cosine similarities at the end of training, and the (c) kernel density estimate of the pairwise gradient cosine similarities at the end of training.}
\label{fig:c100_wrs_depth}
\end{figure*}

\begin{figure*}[t]
\centering
\subfigure[]{\includegraphics[width=0.32\textwidth]{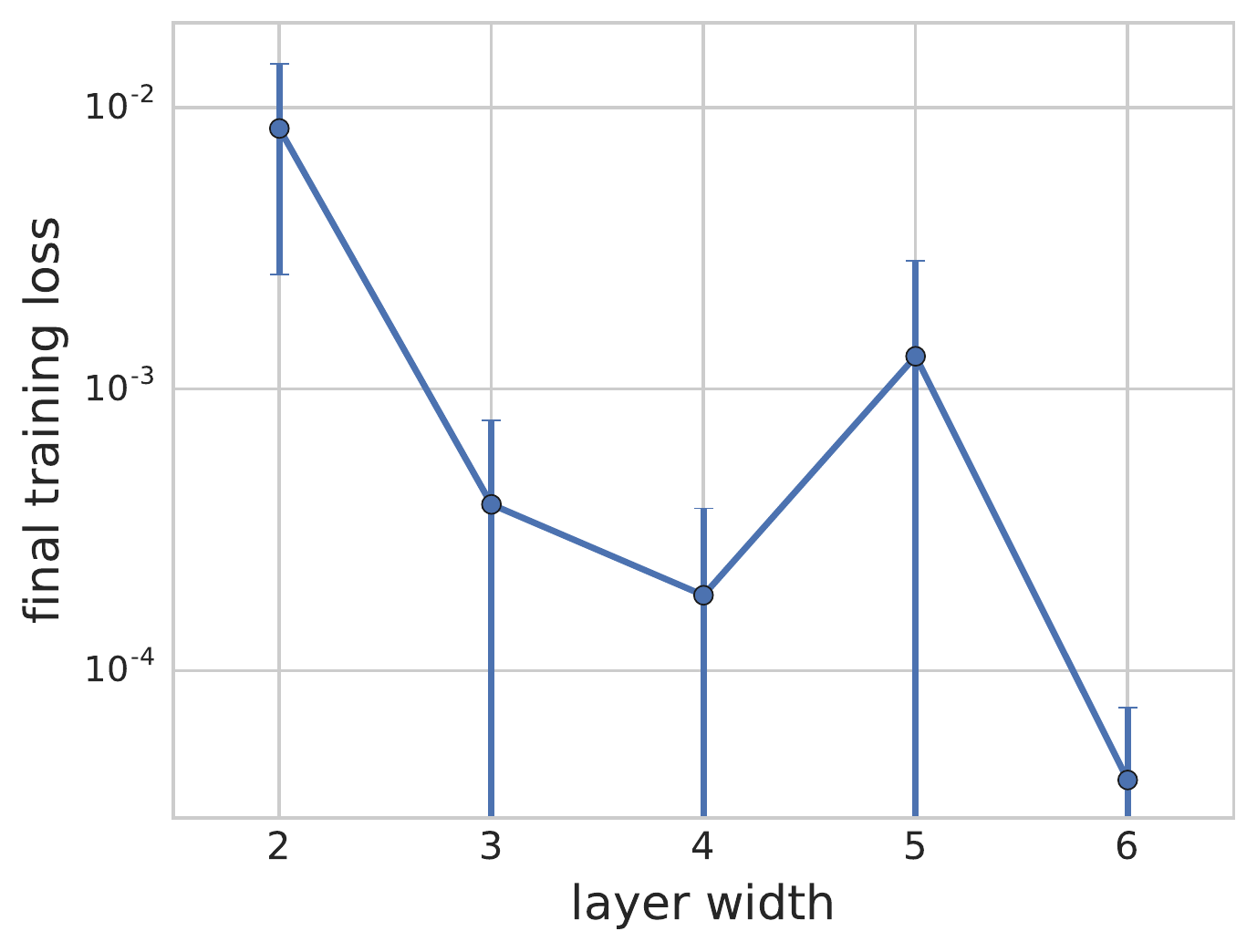}}
\subfigure[]{\includegraphics[width=0.32\textwidth]{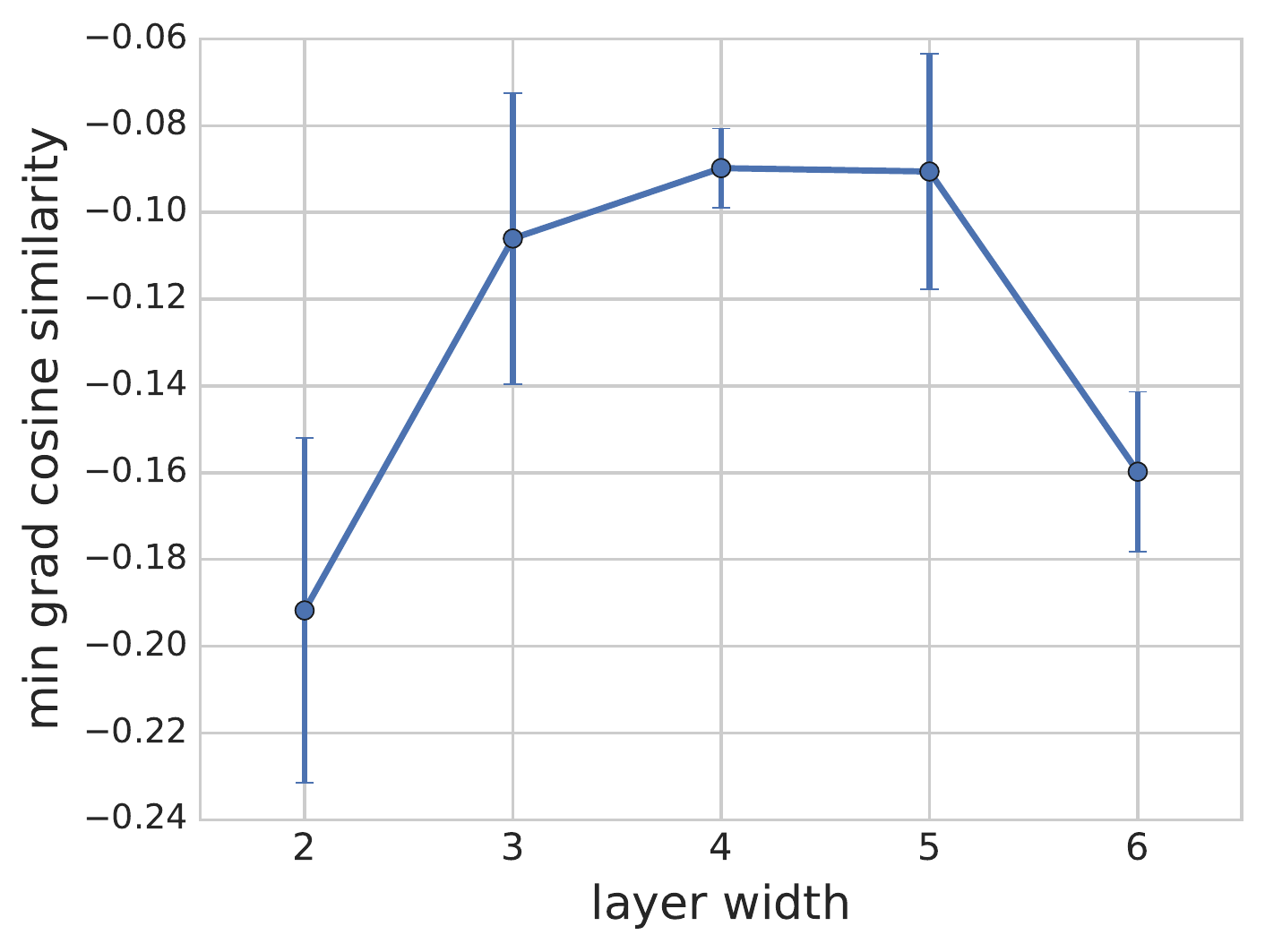}}
\subfigure[]{\includegraphics[width=0.32\textwidth]{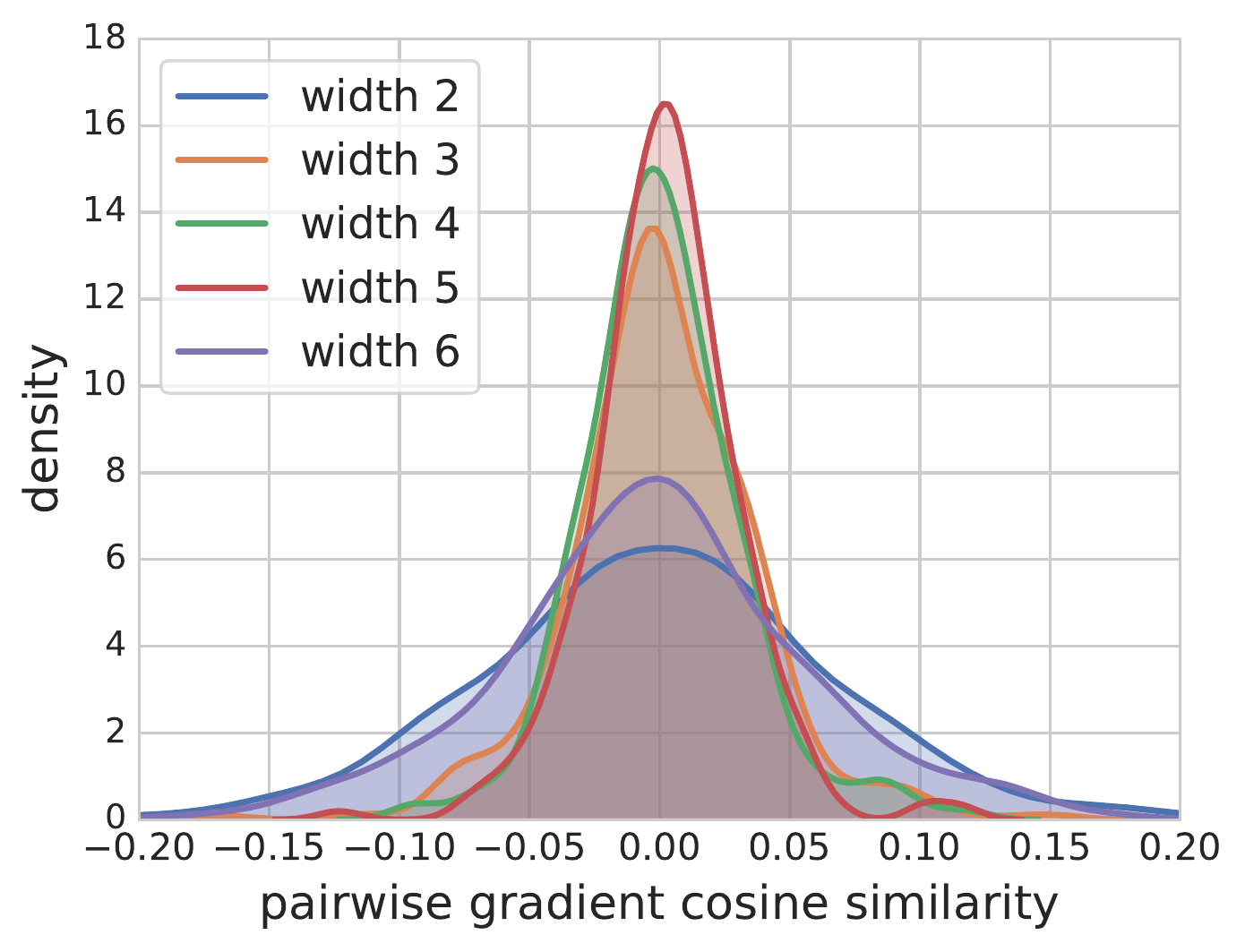}}
\caption{The effect of width with WRN-16-$\ell$  (no batch normalization) on CIFAR-10. The plots show the (a) training loss values at the end of training, (b) minimum of pairwise gradient cosine similarities at the end of training, and the (c) kernel density estimate of the pairwise gradient cosine similarities at the end of training.}
\label{fig:c10_wrs_width}
\end{figure*}

\begin{figure*}[t]
\centering
\subfigure[]{\includegraphics[width=0.32\textwidth]{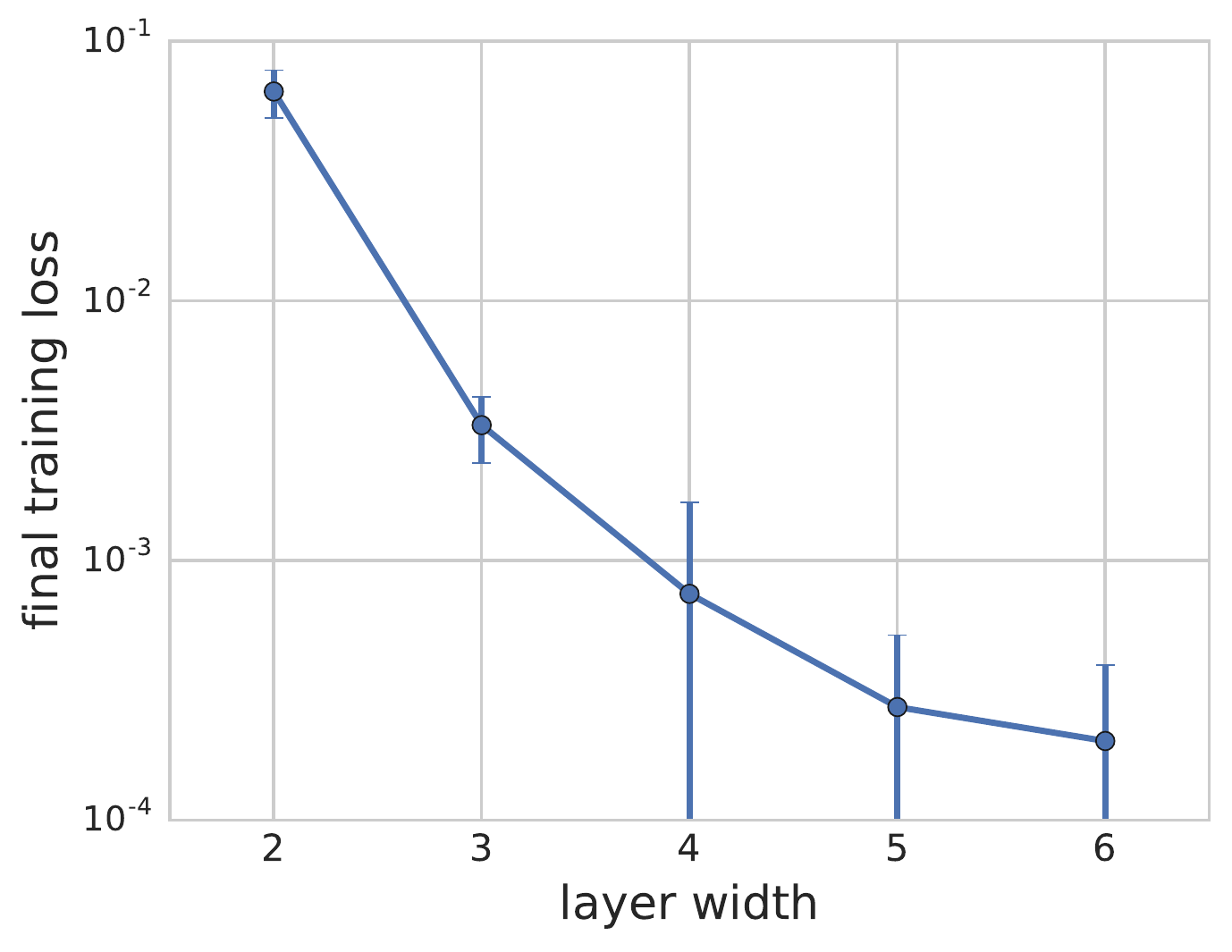}}
\subfigure[]{\includegraphics[width=0.32\textwidth]{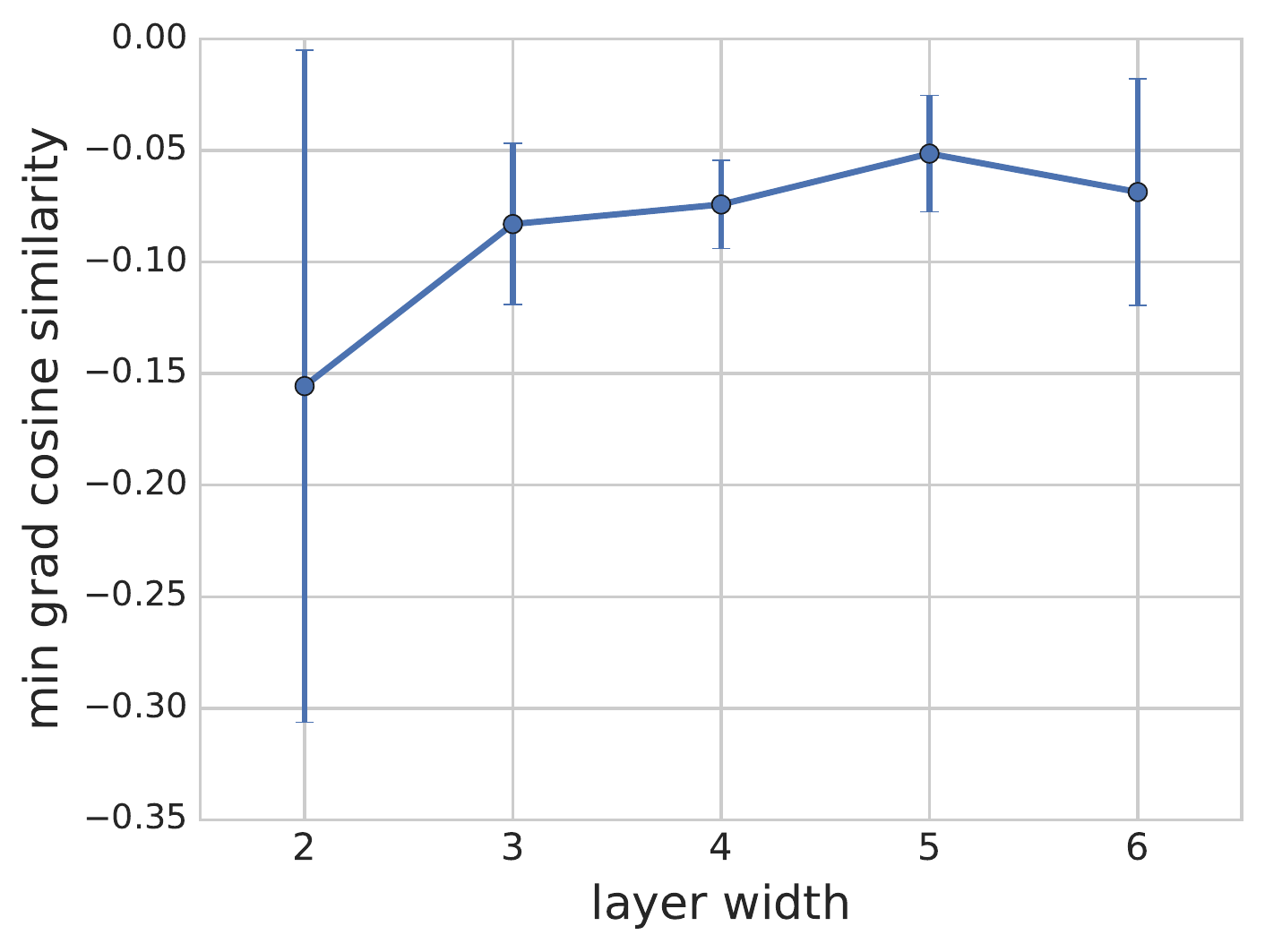}}
\subfigure[]{\includegraphics[width=0.32\textwidth]{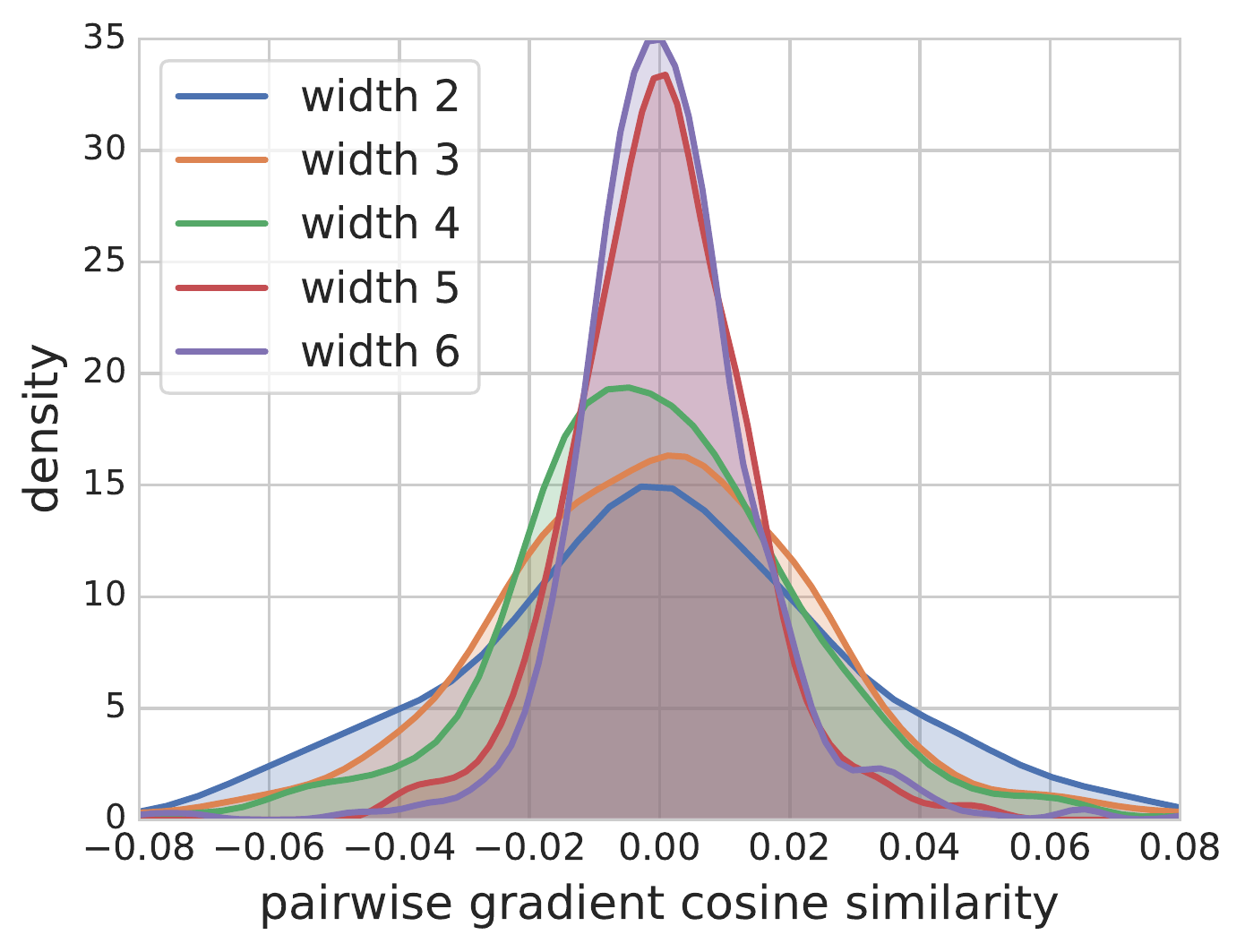}}
\caption{The effect of width with WRN-16-$\ell$  (no batch normalization) on CIFAR-100. The plots show the (a) training loss values at the end of training, (b) minimum of pairwise gradient cosine similarities at the end of training, and the (c) kernel density estimate of the pairwise gradient cosine similarities at the end of training.}
\label{fig:c100_wrs_width}
\end{figure*}
      
\subsection{Effect of batch normalization and skip connections}


In section \ref{sec:experiments} we showed results on the effect of adding batch normalization and skip connections to CNNs and WRNs on an image classification task on CIFAR-10. In this section, we present similar results for image classification on CIFAR-100. Similar to section \ref{sec:experiments}, figure~\ref{fig:c100_bnskip_depth} shows that adding skip connections or batch normalization individually help in training deeper models, but these models still suffer from worsening results and increasing gradient confusion as the network gets deeper. Both these techniques together keep the gradient confusion relatively low even for very deep networks, significantly improving trainability of deep models.

\begin{figure*}[t]
\centering
\subfigure[]{\includegraphics[width=0.32\textwidth]{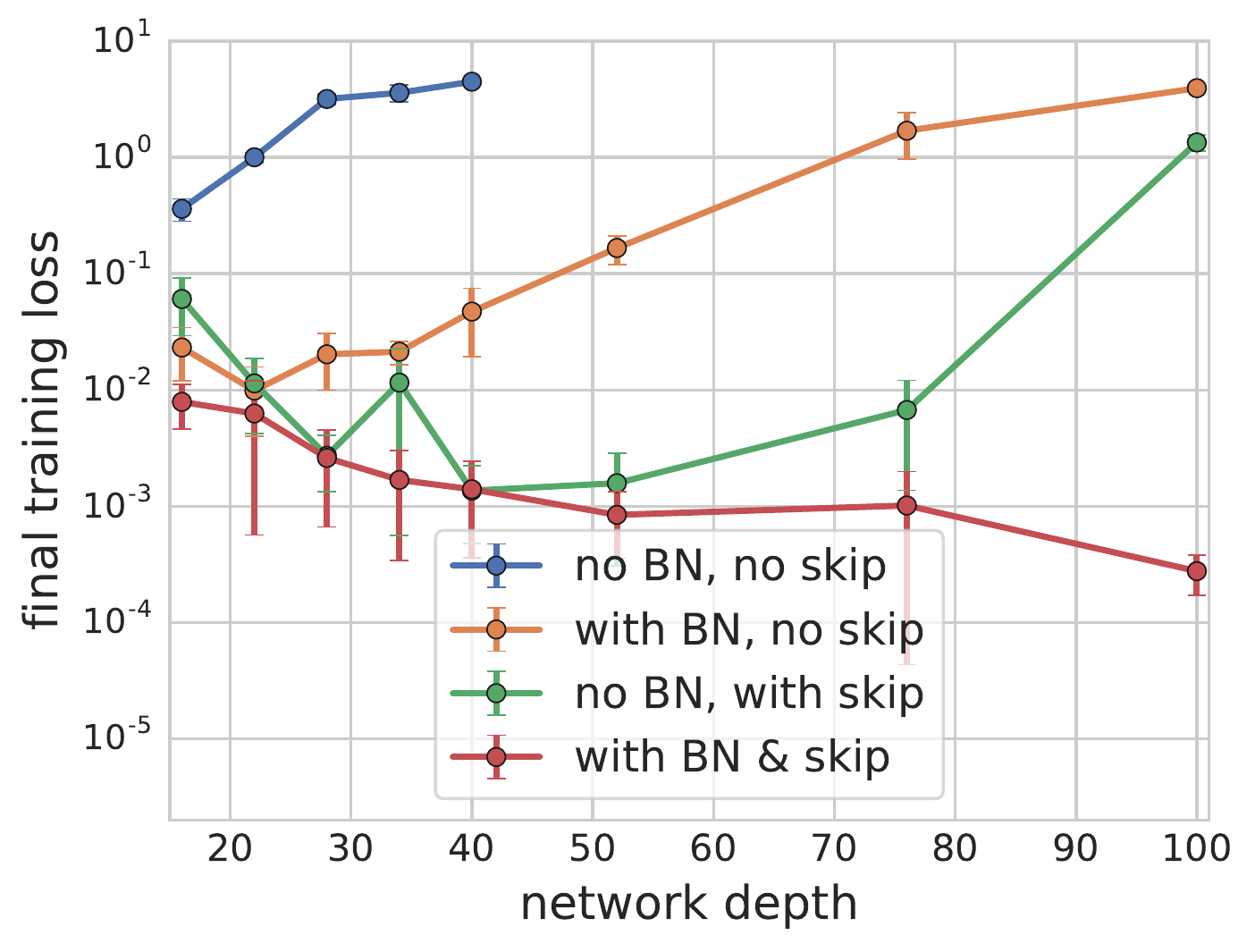}}
\subfigure[]{\includegraphics[width=0.32\textwidth]{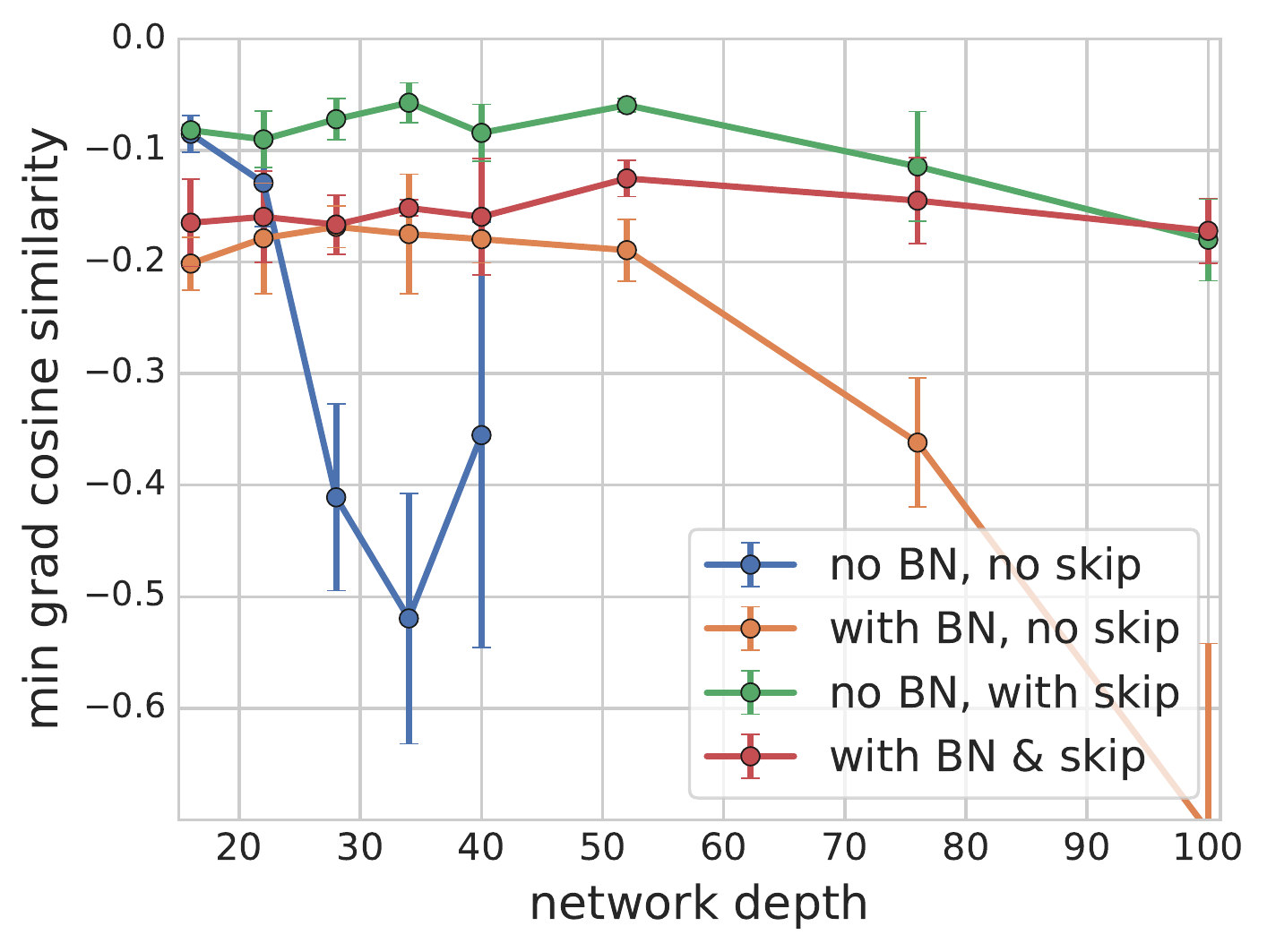}}
\subfigure[]{\includegraphics[width=0.32\textwidth]{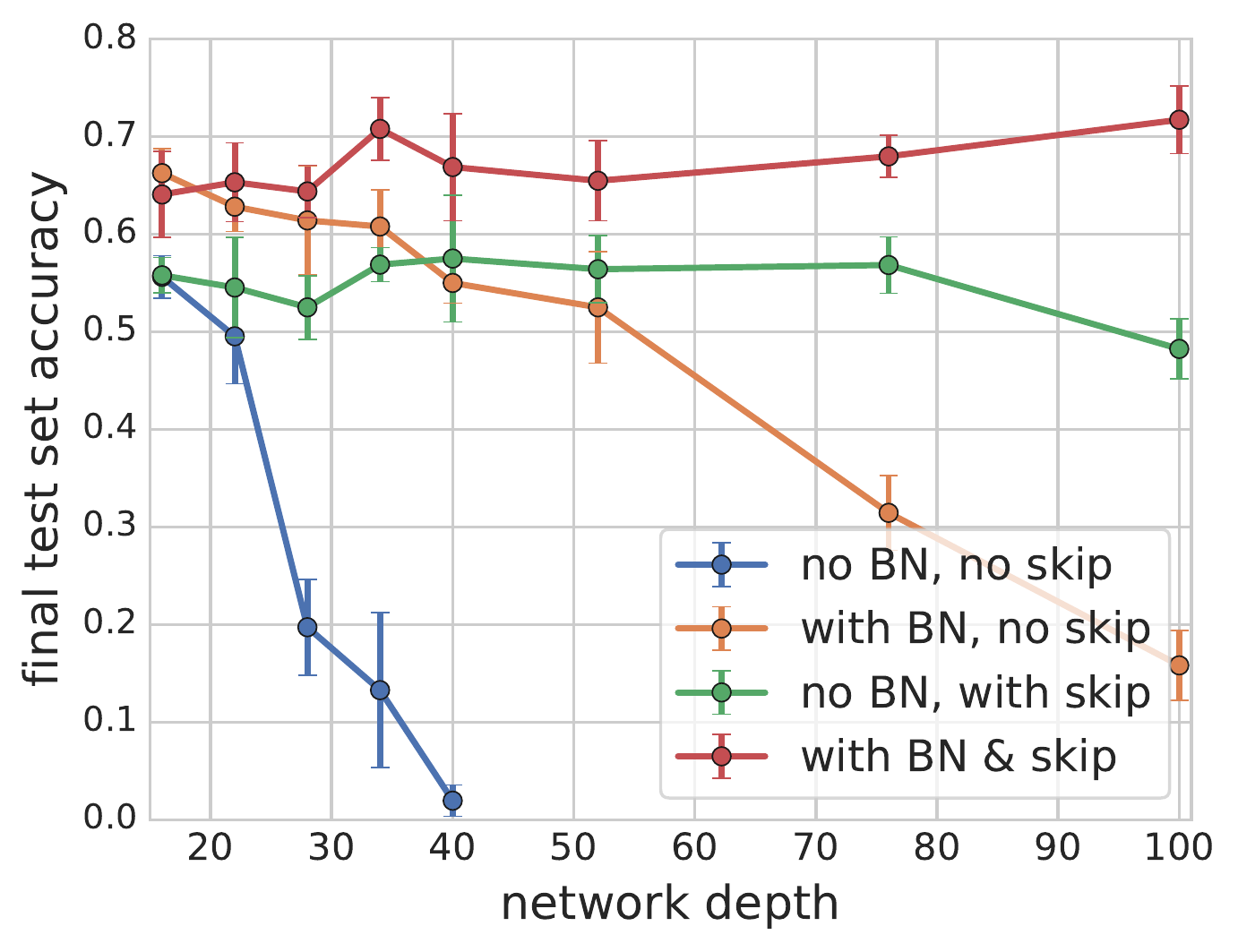}}
\caption{The effect of adding skip connections and batch normalization to CNN-$\beta$-2 on CIFAR-100. Plots show the (a) training loss, (b) minimum pairwise gradient cosine similarities, and the (c) test accuracies at the end of training.}
\label{fig:c100_bnskip_depth}
\end{figure*}

\addtocontents{toc}{\protect\setcounter{tocdepth}{2}}

\section{Near orthogonality of random vectors}
\label{app:orthovec}
For completeness, we state and prove below a lemma on the near orthogonality of random vectors. This result is often attributed to \citet{milman1986asymptotic}. 

\begin{restatable}[Near orthogonality of random vectors]{lem}{orthovec}
\label{lem:orthovec}
For vectors $\{\vec{x_i}\}_{i \in [N]}$ drawn uniformly from a unit sphere in $d$ dimensions, and $\nu>0,$ 
\begin{align*}
\textstyle
\Pr \big[ \exists i, j ~~ |\vec{x}_i^\top\vec{x}_j |  > \nu \big]  \le N^2 \sqrt{\frac{\pi}{8}} \exp \big( - \frac{d-1}{2} \nu^2 \big).
\end{align*}
\end{restatable}

\begin{proof}
Given a fixed vector $\vec{x},$ a uniform random vector $\vec{y}$ satisfies $|\vec{x}^\top\vec{y}| \ge \nu$ only if $\vec{y}$ lies in one of two spherical caps: one centered at $\vec{x}$ and the other at $-\vec{x},$ and both with angular radius $\cos^{-1}(\nu) \le \frac{\pi}{2} -\nu.$  A simple result often attributed to \citet{milman1986asymptotic} bounds the probability of lying in either of these caps as 
\begin{align} \label{capprob}
\Pr [ |\vec{x}^\top\vec{y}|  \ge \nu]  \le    \sqrt{\frac{\pi}{2}} \exp\left( -\frac{d-1}{2} \nu^2 \right).
\end{align}
Because of rotational symmetry, the bound \eqref{capprob} holds if both $\vec{x}$ and $\vec{y}$ are chosen uniformly at random.

We next apply a union bound to control the probability that $|\vec{x}_i^\top\vec{x}_j|  \ge \nu$ for some pair $(i,j).$  There are fewer than $N^2/2$ such pairs, and so the probability of this condition is
\begin{align*}
\Pr [ |\vec{x}_i^\top\vec{x}_j| \ge \nu, \text{for some } i,j]  \le   \frac{N^2}{2} \sqrt{\frac{\pi}{2}} \exp\left( -\frac{d-1}{2} \nu^2 \right).&\qedhere
\end{align*}
\end{proof}

\section{Low-rank Hessians lead to low gradient confusion}
\label{app:hessians}

In this section, we show that low-rank random Hessians result in low gradient confusion. For clarity in presentation, suppose each $f_i$ has a minimizer at the origin (the same argument can be easily extended to the more general case). Suppose also that there is a Lipschitz constant for the Hessian of each function $f_i$ that satisfies $\|\vec{H}_i(\vec{w}) - \vec{H}_i(\vec{w'}) \| \le L_H \|\vec{w}-\vec{w'}\|$ (note that this is a standard optimization assumption \citep{nesterov2018lectures}, with evidence that it is applicable for neural networks \citep{martens2016second}). Then
$ \nabla f_i(\vec{w}) = \vec{H}_i\vec{w}+\vec{e}$,
where $\vec{e}$ is an error term bounded as: $\|\vec{e}\| \le \half L_H\|\vec{w}\|^2,$
and we use the shorthand $\vec{H}_i$ to denote $\vec{H}_i(\vec{0}).$ Then we have:
\aln{
  | \la \nabla f_i(\vec{w}) , \nabla f_j(\vec{w}) \ra | 
&  = | \la  \vec{H}_i\vec{w} ,  \vec{H}_j\vec{w} \ra |  +  \la \vec{e}, \vec{H}_i \vec{w} + \vec{H}_j \vec{w}  \ra  + \| \vec{e} \|^2\nonumber \\ 
&  \le \|\vec{w}\|^2 \| \vec{H}_i \| \| \vec{H}_j\|  + \| \vec{e} \| \|\vec{w}\| (  \|\vec{H}_i\|  + \|\vec{H}_j\| ) + \| \vec{e} \|^2\nonumber \\
&   \le \|\vec{w}\|^2 \| \vec{H}_i \| \| \vec{H}_j\|   +  \frac{1}{2} L_H \|\vec{w}\|^3 (  \|\vec{H}_i\|  + \|\vec{H}_j\| ) + \frac{1}{4} L_H^2 \| \vec{w} \|^4. \nonumber
}
If the Hessians are sufficiently random and low-rank (e.g., of the form $\vec{H}_i = \vec{a}_i\vec{a}_i^\top$ where $\vec{a}_i\in \reals^{N\times r}$ are randomly sampled from a unit sphere), then one would expect the terms in this expression to be small for all $\vec{w}$ within a neighborhood of the minimizer. 

There is evidence that the Hessian at the minimizer is very low rank for many standard over-parameterized neural network models \citep{sagun2017empirical, cooper2018loss, chaudhari2016entropy, wu2017towards, ghorbani2019investigation}. While a bit non-rigorous, the above result nonetheless suggests that for many standard neural network models, the gradient confusion might be small for a large class of weights near the minimizer.

\section{Missing proofs}
\label{app:missing_proofs}

\subsection{Proofs of theorems \ref{thm:cosine} and \ref{thm:nonconvex}}
\label{appsec:rate}
This section presents proofs for the convergence theorems of SGD presented in section \ref{sec:convergence}, under the assumption of low gradient confusion.  For clarity of presentation, we re-state each theorem before its proof.

\linearConvergence*

\begin{proof}
Let $\tilde{i} \in [N]$ denote the index of the realized function $\tilde{f}_k$ in the uniform sampling from $\{ f_i \}_{i \in [N]}$ at step $k$. From assumption (A1), we have
\begin{align*}
F(\vec{w}_{k+1}) & \leq  F(\vec{w}_k) + \langle \nabla F (\vec{w}_k), \, \vec{w}_{k+1}  - \vec{w}_k \rangle + \frac{L}{2} \| \vec{w}_{k+1} - \vec{w}_k \|^2 \\
&=  F(\vec{w}_k) - \alpha \langle \nabla F  (\vec{w}_k), \, \nabla \tilde f_k(\vec{w}_k) \rangle + \frac{L\alpha^2}{2} \|  \nabla \tilde f_k(\vec{w}_k)  \|^2  \\
&= F(\vec{w}_k) - \Big( \frac{\alpha}{N} - \frac{L\alpha^2}{2} \Big)   \|  \nabla \tilde f_k(\vec{w}_k)  \|^2 - \frac{\alpha}{N} \sum_{\forall i: i \neq \tilde{i}}  \langle \nabla f_i  (\vec{w}_k), \, \nabla \tilde f_k(\vec{w}_k) \rangle \\
&\le F(\vec{w}_k) - \Big( \frac{\alpha}{N} - \frac{L\alpha^2}{2} \Big)   \|  \nabla \tilde f_k(\vec{w}_k)  \|^2 + \frac{\alpha (N-1) \eta}{N}, \\
&\le F(\vec{w}_k) -  \Big( \frac{\alpha}{N} - \frac{L\alpha^2}{2} \Big)   \|  \nabla \tilde f_k(\vec{w}_k)  \|^2 + \alpha \eta,
\end{align*}
where the second-last inequality follows from definition \ref{ass:cosine}. Let the learning rate $\alpha < 2/NL$. Then, using assumption (A2) and subtracting by $F\opt = \min_\vec{w} F(\vec{w})$ on both sides, we get
\begin{align*}
F(\vec{w}_{k+1}) - F\opt \le F(\vec{w}_k) - F\opt - 2 \mu\Big(\frac{\alpha}{N} - \frac{L\alpha^2}{2}  \Big) (\tilde f_k(\vec{w}_k) - \tilde f_k\opt) + \alpha \eta,
\end{align*}
where $\tilde f_k\opt = \min_\vec{w} \tilde f_k(\vec{w})$. It is easy to see that by definition we have, $\expect_i [ f_i\opt ] \le F\opt$. Moreover, from assumption that $\alpha < \frac{2}{NL}$, it implies that $\Big(\frac{\alpha}{N} - \frac{L\alpha^2}{2}  \Big) > 0$. Therefore, taking expectation on both sides we get,
\begin{align*}
\expect[F(\vec{w}_{k+1}) - F\opt] \le \Big( 1 - \frac{2\mu \alpha}{N} + \mu L\alpha^2 \Big) \expect[F(\vec{w}_k) - F\opt] + \alpha \eta.
\end{align*}
Writing $\rho = 1 -  \frac{2\mu\alpha}{N} + \mu L\alpha^2 $, and unrolling the iterations, we get
\begin{align*}
\expect[F(\vec{w}_{k+1}) - F\opt] &\le \rho^{k+1} (F(\vec{w}_0) - F\opt) + \sum_{i=0}^k \rho^i \alpha \eta \\
&\le \rho^{k+1} (F(\vec{w}_0) - F\opt) + \sum_{i=0}^\infty \rho^i \alpha \eta \\
&= \rho^{k+1} (F(\vec{w}_0) - F\opt) + \frac{\alpha \eta}{1 - \rho}.\qedhere
\end{align*}
\end{proof}


\boundedVariance*

\begin{proof}
From theorem \ref{thm:cosine}, we have:
\begin{align}
\label{cvg_lemma}
F(\vec{w}_{k+1})
\le F(\vec{w}_k) - \Big( \frac{\alpha}{N} - \frac{L\alpha^2}{2} \Big)   \|  \nabla \tilde f_k(\vec{w}_k)  \|^2 + \alpha  \eta.
\end{align}

Now we know that:
\begin{align*}
\expect \|  \nabla \tilde f_k(\vec{w}_k)  \|^2 = \expect \|  \nabla \tilde f_k(\vec{w}_k) - \nabla F(\vec{w}_k) \|^2 + \expect \| \nabla F(\vec{w}_k)  \|^2 \ge \expect \| \nabla F(\vec{w}_k)  \|^2.
\end{align*}

Thus, taking expectation and assuming the step size $\alpha < 2/(NL)$, we can rewrite equation \ref{cvg_lemma} as:
\begin{align*}
\expect \| \nabla F(\vec{w}_k)  \|^2 
&\le \frac{2N}{2 \alpha - NL\alpha^2 } \expect [F(\vec{w}_k) - F(\vec{w}_{k+1}) ] + \frac{2N  \eta}{2 - NL\alpha}.
\end{align*}
Taking an average over $T$ iterations, and using $F\opt = \min_{\vec{w}} F(\vec{w})$, we get:
\begin{align*}
\min_{k=1, \dots, T} \expect \| \nabla F(\vec{w}_k)  \|^2 \le \frac{1}{T} \sum_{k=1}^\top \expect \| \nabla F(\vec{w}_k)  \|^2 &\le \frac{2N}{2 \alpha - NL\alpha^2 } \frac{  F(\vec{w}_1) - F\opt }{T} + \frac{2N  \eta}{2 - NL\alpha}. \qedhere
\end{align*}
\end{proof}

\subsection{Proof of lemma \ref{lem:lossPropNeural}}
	\label{appsec:helperlemma}

  \begin{lemma}
\label{lem:lossPropNeural}
		Consider the set of loss-functions $\{f_i(\vec{W})\}_{i \in [N]}$ where all $f_i$ are either the square-loss function or the logistic-loss function. Recall that $f_i(\vec{W}) := f(\vec{W}, \vec{x}_i)$. Consider a feed-forward neural network as defined in equation \ref{eq:NNmodel} whose weights $\vec{W}$ satisfy assumption~\ref{ass:small_weight}. Consider the gradient $\nabla_{\vec{W}} f_i(\vec{W})$ of each function $f_i$. From definition we have that $\nabla_{\vec{W}} f_i(\vec{W}) = \zeta_{\vec{x}_i}(\vec{W}) \nabla_{\vec{W}}g_{\vec{W}}(\vec{x}_i)$, where we define $\zeta_{\vec{x}_i}(\vec{W}) = \partial f_i(\vec{W}) / \partial g_{\vec{W}}$. Then we have the following properties.
		\begin{enumerate}
			\item When $\| \vec{x} \| \leq 1$ for every $p \in [\beta]$ we have $\| \nabla_{\vec{W}_p}g_{\vec{W}}(\vec{x}_i) \| \leq 1$.
			\item There exists  $0 < \zeta_0 \leq 2 \sqrt{\beta}$, such that $|\zeta_{\vec{x}_i}(\vec{W}) | \leq 2 \,$, $\|\nabla_{\vec{x}_i} \zeta_{\vec{x}_i}(\vec{W}) \|_2 \leq \zeta_0 \,$, $\|\nabla_{\vW} \zeta_{\vec{x}_i}(\vec{W})\|_2 \leq \zeta_0$.
		\end{enumerate}
	\end{lemma}
	\begin{proof}
	    The first property is a direct consequence of assumption~\ref{ass:small_weight} and property (P2) of the activation function.
	
		Let $\vec{W}$ denote the tuple $(\vec{W}_p)_{p \in [\beta]_0}$. Consider $|\zeta_{\vec{x}_i}(\vec{W})| = | \partial f_i(\vec{W}) / \partial g_{\vec{W}} |$. In the case of square-loss function this evaluates to $|g_{\vec W}(\vec{x}) - \cC(\vec{x})| \leq 2$. In case of logistic regression, this evaluates to $|\frac{-1}{1+\exp(\cC(\vec{x}_i)g_{\vec W}(\vec{x}_i))}| \leq 1$. Now we consider $\| \nabla_{\vec{x}_i} \zeta_{\vec{x}_i}(\vec{W}) \|$. Consider the squared loss function. We then have the following.
			\begin{align*}
			\| \nabla_{\vec{x}_i} \zeta_{\vec{x}_i}(\vec{W}) \| &= \| \nabla_{\vec{x}_i} f'(\vec{W}) \| \\
			&= \| \nabla_{\vec{x}_i} g_{\vec{W}}(\vec{x}_i) - \cC(\vec{x}_i) \| \\
			& \leq \| \nabla_{\vec{x}_i} g_{\vec{W}}(\vec{x}_i) \| + 1.	
		\end{align*}
		
		Likewise, consider the logistic-loss function. We then have the following.
		\begin{align*}
			\| \nabla_{\vec{x}_i} \zeta_{\vec{x}_i}(\vec{W}) \| & \leq \norm{\frac{\cC(\vec{x}_i)^2}{(1+\exp(\cC(\vec{x}_i) g_{\vec{W}}(\vec{x}_i)))^2} \exp(\cC(\vec{x}_i) g_{\vec{W}}(\vec{x}_i))} \| \nabla_{\vec{x}_i} g_{\vec{W}}(\vec{x}_i) \| \\
			& \leq \| \nabla_{\vec{x}_i} g_{\vec{W}}(\vec{x}_i) \|.
		\end{align*}
		
		Thus, it suffices to bound $\| \nabla_{\vec{x}_i} g_{\vec{W}}(\vec{x}_i)	\|$. Using assumption~\ref{ass:small_weight} and the properties (P1), (P2) of $\sigma$, this can be upper-bounded by $1$.
		
	 Consider $\nabla_{\vec{W}_p} \zeta_{\vec{x}_i}(\vec{W})$ for some layer index $p \in [\beta]_0$. We will show that $\| \nabla_{\vec{W}_p} \zeta_{\vec{x}_i}(\vec{W}) \|_2 \leq 2$. Then it immediately follows that $\| \nabla_{\vec{W}} \zeta_{\vec{x}_i}(\vec{W}) \|_2 \leq 2 \sqrt{\beta}$.  In the case of a squared loss function. We have the following.
		\begin{align*}
			\| \nabla_{\vec{W}_p} \zeta_{\vec{x}_i}(\vec{W}) \| &= \| \nabla_{\vec{W}_p} f'(\vec{W}) \| \\
			&= \| \nabla_{\vec{W}_p} g_{\vec{W}}(\vec{x}_i) - \cC(\vec{x}_i) \| \\
			& \leq \| \nabla_{\vec{W}_p} g_{\vec{W}}(\vec{x}_i) \| + 1.	
		\end{align*}
		
		Likewise, consider the logistic-loss function. We then have the following.
		\begin{align*}
			\| \nabla_{\vec{W}_p} \zeta_{\vec{x}_i}(\vec{W}) \| & \leq \norm{\frac{\cC(\vec{x}_i)^2}{(1+\exp(\cC(\vec{x}_i) g_{\vec{W}}(\vec{x}_i)))^2} \exp(\cC(\vec{x}_i) g_{\vec{W}}(\vec{x}_i))} \| \nabla_{\vec{W}_p} g_{\vec{W}}(\vec{x}_i) \| \\
			& \leq \| \nabla_{\vec{W}_p} g_{\vec{W}}(\vec{x}_i) \|.
		\end{align*}
		
		Since $\|\nabla_{\vec{W}_p} g_{\vec{W}}(\vec{x}_i) \| \leq 1$, we have that $\| \nabla_{\vec{W}_p} \zeta_{\vec{x}_i}(\vec{W}) \| \leq 2$ in both the cases. Thus, $\zeta_0 = 2 \sqrt{\beta}$. 
	\end{proof}

\subsection{Proofs of theorem \ref{thm:arbitraryNN} and corollary \ref{thm:uniformNN}}
\label{appsec:nnconc}
		In this section, we will present the proofs of theorem \ref{thm:arbitraryNN} and corollary \ref{thm:uniformNN}.

		\NeuralNet*
		\begin{proof}
	We show two key properties, namely bounded gradient and non negative expectation. We will then use both these properties to complete the proof.
	
		\xhdr{Bounded gradient.}
		For every $i \in [n]$ define $\zeta_{\vec{x}_i}(\vec{W}) := f'(\vec{W})$. 
		For every $p \in [\beta]$ define $\vec{H}_p$ as follows.
		\[
				\vec{H}_p(\vec{x}) := \sigma(\vec{W}_p \cdot \sigma( \vec{W}_{p-1} \cdot \sigma(\ldots \cdot \sigma(\vec{W}_0 \cdot \vec{x}) \ldots).
		\]
		Fix an $i \in [N]$. Then we have the following recurrence
		\begin{align*}
			g_\beta(\vec{x}_i) &:= \sigma'(H_{\beta}(\vec{x}_i))& \\
			\vec{g}_p(\vec{x}_i) &:= (\vec{W}_{p+1}^\top \cdot \vec{g}_{p+1}(\vec{x}_i)) \cdot \Diag(\sigma'(\vec{H}_p(\vec{x}_i)))& \quad \forall p \in \{0, 1, \ldots, \beta-1\}.
		\end{align*}
		
		Then the gradients can be written in terms of the above quantities as follows.
		\begin{align*}
			\nabla_{\vec{W}_p} f_i (\vec{W}) &= \vec{g}_p(\vec{x}_i) \cdot \vec{H}_{p-1}(\vec{x}_i)^\top & \quad \forall p \in [\beta]_0.
		\end{align*}

		We can write, the gradient confusion denote by  $h_{\vec{W}}(\vec{x}_i, \vec{x}_j)$, as follows.
		\begin{equation}
			\label{eq:HfunctionNN}
			\zeta_{\vec{x}_i}(\vec{W}) \zeta_{\vec{x}_j}(\vec{W}) \left( \sum_{p \in [\beta]_0} \Tr[ \vec{H}_{p-1}(\vec{x}_i) \cdot \vec{g}_p(\vec{x}_i)^\top \cdot \vec{g}_p(\vec{x}_j) \cdot \vec{H}_{p-1}(\vec{x}_i)^\top]\right).
		\end{equation}
		We will now bound $\|\nabla_{(\vec{x}_i, \vec{x}_j)} h_{\vec{W}}(\vec{x}_i, \vec{x}_j) \|_2$. Consider 
		$\nabla_{\vec{x}_i} h_{\vec{W}}(\vec{x}_i, \vec{x}_j)$. This can be written as follows.
		\begin{multline}
			\label{eq:genNNMain}
			(\nabla_{\vec{x}_i} \zeta_{\vec{x}_i}(\vec{W})) \zeta_{\vec{x}_j}(\vec{W}) \left( \sum_{p \in [\beta]_0} \Tr[ \vec{H}_{p-1}(\vec{x}_i) \cdot \vec{g}_p(\vec{x}_i)^\top \cdot \vec{g}_p(\vec{x}_j) \cdot \vec{H}_{p-1}(\vec{x}_i)^\top]\right) +\\
		\zeta_{\vec{x}_i}(\vec{W}) \zeta_{\vec{x}_j}(\vec{W}) \sum_{p \in [\beta]_0} \left[ \nabla_{\vec{x}_i} \left( \vec{H}_{p-1}(\vec{x}_i) \cdot \vec{g}_p(\vec{x}_i)^\top \cdot \vec{g}_p(\vec{x}_j) \cdot \vec{H}_{p-1}(\vec{x}_i) \right) \right]^\top.
		\end{multline}
		
		Observe that each of the entries in the diagonal matrix $\Diag(\sigma'(\vec{H}_p(\vec{x}_i)))$ is at most $1$. Thus, we have that $\|\Diag(\sigma'(\vec{H}_p(\vec{x}_i)))\| \leq 1$.

		We have the following relationship.
		\begin{align*}
			\| g_\beta(\vec{x}_i) \| & \leq 1 &\\
			\| \vec{g}_p(\vec{x}_i) \| & \leq  \| \vec{W}_{p+1}^\top \| \| \vec{g}_{p+1}(\vec{x}_i)) \| \| \Diag(\sigma'(\vec{H}_p(\vec{x}_i))) \| \leq 1 & \quad \forall p \in \{0, 1, \ldots, \beta-1\}.
		\end{align*}
		
		Moreover we have, 
		\[
			\| \Tr[ \vec{H}_{p-1}(\vec{x}_i) \cdot \vec{g}_p(\vec{x}_i)^\top \cdot \vec{g}_p(\vec{x}_j) \cdot \vec{H}_{p-1}(\vec{x}_i)^\top] \| \leq \| \vec{H}_{p-1}(\vec{x}_i) \| \|\vec{g}_p(\vec{x}_i)^\top\| \| \vec{g}_p(\vec{x}_j) \| \|\vec{H}_{p-1}(\vec{x}_i)^\top\| \leq 1.
		\]
		
		Consider $\| \nabla_{\vec{x}_i} \left( \vec{H}_{p-1}(\vec{x}_i) \cdot \vec{g}_p(\vec{x}_i)^\top \cdot \vec{g}_p(\vec{x}_j) \cdot \vec{H}_{p-1}(\vec{x}_i) \right) \|$ for every $p \in [\beta]_0$.

		This can be upper-bounded by,
		\[
				\| \nabla_{\vec{x}_i} \vec{H}_{p-1}(\vec{x}_i) \| \| \vec{g}_p(\vec{x}_i)^\top \| \| \vec{g}_p(\vec{x}_j) \|  \| \vec{H}_{p-1}(\vec{x}_i) \| + \| \vec{H}_{p-1}(\vec{x}_i) \| \| \nabla_{\vec{x}_i} \vec{g}_p(\vec{x}_i)^\top\| \| \vec{g}_p(\vec{x}_j) \| \| \vec{H}_{p-1}(\vec{x}_i) \|.
		\]
		Note that $\nabla_{\vec{x}_i} \vec{H}_{p-1}(\vec{x}_i) = \vec{g}_1(\vec{x}_i) \cdot \Diag(\sigma'(\vec{W}_0 \cdot \vec{x}_i)) \cdot \vec{W}_0^\top \cdot \vec{g}_p(\vec{x}_i)^\top$. Thus, $\| \nabla_{\vec{x}_i} \vec{H}_{p-1}(\vec{x}_i) \| \leq 1$. We will now show that $\| \nabla_{\vec{x}_i} \vec{g}_p(\vec{x}_i) \| \leq \beta - p + 1$. We prove this inductively.
		Consider the base case when $p=\beta$.
		\[
				\| \nabla_{\vec{x}_i} \vec{g}_{\beta}(\vec{x}_i) \| = \| \nabla_{\vec{x}_i} \sigma'(\vec{H}_{\beta}(\vec{x}_i)) \| \leq 1 = \beta-\beta+1.
		\]
		Now, the inductive step.
		\[
				\| \nabla_{\vec{x}_i} \vec{g}_p(\vec{x}_i) \| \leq \| \nabla_{\vec{x}_i} \vec{g}_{p+1}(\vec{x}_i) \| + \| \nabla_{\vec{x}_i} \Diag(\sigma'(\vec{H}_p(\vec{x}_i))) \| \leq \beta-p \leq \beta-p+1.
		\]
		Thus, using equation \ref{eq:genNNMain} and the above arguments, we obtain,
		$\| \nabla_{\vec{x}_i} h_{\vec{W}}(\vec{x}_i, \vec{x}_j) \|_2 \leq \zeta_0^2(\beta+1) + \zeta_0^2(\beta+1)(\beta+2) \leq 2 \zeta_0^2(\beta+2)^2$ and thus, $\|\nabla_{(\vec{x}_i, \vec{x}_j)} h_{\vec{W}}(\vec{x}_i, \vec{x}_j)\|_2 \leq 4 \zeta_0^2(\beta+2)^2$.
	
	\xhdr{Non-negative expectation.}
	\begin{align}
		\mathbb{E}_{\vxi,\vxj}[h(\vec{x}_i, \vec{x}_j)] =& \mathbb{E}_{\vxi,\vxj}[\langle \nabla f_i (\vec{W}) , \nabla f_j (\vec{W}) \rangle] \nonumber \\
		&= \langle \mathbb{E}_{\vxi}[ \nabla f_i (\vec{W})] , \mathbb{E}_{\vxj}[\nabla f_j (\vec{W}) ] \rangle \nonumber \\
		&= \|\mathbb{E}_{\vxi}[ \nabla f_i (\vec{W})]\|^2\ge 0. \label{eq:nonNegativeIID}
		\end{align}
We have used the fact that $\nabla f_i (\vec{W})$ and $ \nabla f_j (\vec{W})$ are identically distributed and independent.

	 \xhdr{Concentration of Measure.}	
     We combine the two properties as follows. From \textbf{Non-negative Expectation} property and equation \ref{lem:conc:eq2}, we have that 
			\begin{align} \label{linregprobbound}
				{ \Pr[h_{\vec{W}}(\vec{x}_i, \vec{x}_j) \leq -\eta] \leq   \Pr[h_{\vec{W}}(\vec{x}_i, \vec{x}_j) \leq \mathbb{E}_{(\vec{x}_i, \vec{x}_j)}[h_{\vec{W}}(\vec{x}_i, \vec{x}_j)]-\eta] \leq  \exp\left( \frac{-c d \eta^2}{16 \zeta_0^4 (\beta+2)^4} \right) .}
			\end{align}
To obtain the probability that {\em some} value of $h_{\vec{w}}(\nabla_{\vec{w}} f_i ,\nabla_{\vec{w}} f_j)$ lies below $-\eta,$ we use a union bound.  There are $N(N-1)/2<N^2/2$ possible pairs of data points to consider, and so this probability is bounded above by $N^2 \exp\left( \frac{-c d \eta^2}{16 \zeta_0^4 (\beta+2)^4} \right)$. 
\end{proof}

\subsubsection{Proof of corollary~\ref{thm:uniformNN}}
Before we prove corollary~\ref{thm:uniformNN} we first prove the following helper lemma.

\begin{lemma} \label{lipschitz_product}
Suppose $\max_\vW \|\nabla_\vW f_i (\vW)\| \le M,$ and both  $\nabla_{\vW} f_i(\vw)$ and  $\nabla_{\vW} f_j(\vW)$ are Lipschitz in $\vW$ with constant $L$. Then $h_\vW(\vxi,\vxj)$ is Lipschitz in $\vW$ with constant $2LM.$
\end{lemma}
\begin{proof}
We view $\vW$ as flattened vector. We now prove the above result for these two vectors. For two vectors $\vw,\vw',$
\begin{align*}
&|h_\vw(\vxi,\vxj)-h_{\vw'}(\vxi,\vxj)| \\&= | \la  \nabla_{\vw} f_i(\vw), \nabla_{\vw} f_j(\vw) \ra-\la  \nabla_{\vw'} f_i(\vw'), \nabla_{\vw'} f_j(\vw') \ra |\\
&= | \la  \nabla_{\vw} f_i(\vw) - \nabla_{\vw'} f_i(\vw')+ \nabla_{\vw'} f_i(\vw')   , \nabla_{\vw} f_j(\vw) \ra \\ & \quad -\la  \nabla_{\vw'} f_i(\vw'), \nabla_{\vw'} f_j(\vw')-\nabla_{\vw} f_j(\vw)+\nabla_{\vw} f_j(\vw) \ra |\\
 &= | \la  \nabla_{\vw} f_i(\vw) - \nabla_{\vw'} f_i(\vw')  , \nabla_{\vw} f_j(\vw) \ra  -\la  \nabla_{\vw'} f_i(\vw'), \nabla_{\vw'} f_j(\vw')-\nabla_{\vw} f_j(\vw) \ra| \\
  &\le  | \la  \nabla_{\vw} f_i(\vw) - \nabla_{\vw'} f_i(\vw')  , \nabla_{\vw} f_j(\vw) \ra |  +|\la  \nabla_{\vw'} f_i(\vw'), \nabla_{\vw'} f_j(\vw')-\nabla_{\vw} f_j(\vw) \ra| \\
    & \le  \|   \nabla_{\vw} f_i(\vw) - \nabla_{\vw'} f_i(\vw') \|  \| \nabla_{\vw} f_j(\vw)  \|  +\| \nabla_{\vw'} f_i(\vw')\| \| \nabla_{\vw'} f_j(\vw')-\nabla_{\vw} f_j(\vw) \| \\
         &\le  L \|\vw - \vw' \|  \| \nabla_{\vw} f_j(\vw)  \|  +\| \nabla_{\vw'} f_i(\vw')\|  L \| \vw'- \vw \| \\
               & \le  2 LM \|\vw - \vw' \|.  
\end{align*}
Here the first inequality uses the triangle inequality, the second inequality uses the Cauchy-Schwartz inequality, and the third and fourth inequalities use the assumptions that $\nabla_\vw f_i (\vw)$ and $\nabla_\vw f_j (\vw)$ are Lipschitz in $\vw$ and have bounded norm.
\end{proof}
 
 We are now ready to prove the corollary, which we restate here.  The proof uses a standard "epsilon-net" argument; we identify a fine net of points within the ball $\mathcal{B}_r.$ If the gradient confusion is small at every point in this discrete set, and the gradient confusion varies slowly enough with $\vW,$ when we can guarantee small gradient confusion at every point in $\mathcal{B}_r.$ \\
 
	\NeuralNetUniform*

\begin{proof}
Define the function $h^+(\vW) = \max_{ij} h_\vW(\vxi,\vxj).$ Our goal is to find conditions under which $h^+(\vW) > -\eta$ for all $\vW$ in a large set.
To derive such conditions, we will need a Lipschitz constant for  $h^+(\vW),$ which is no larger than the maximal Lipschitz constant of $h_\vW(\vxi,\vxj)$ for all $i,j.$  We have that $\|\nabla_{\vW} f_i\| = \|\zeta_{\vec{x}_i}(\vec{W}) \vec{x}_i\| \le \zeta_0 .$ Now we need to get a $\vW$-Lipschitz constants for  $\nabla_\vxi f_i =\zeta_{\vec{x}_i}(\vec{W}) \vxi.$ By lemma \ref{lem:lossPropNeural}, we have $ \| \nabla_\vW (\zeta_{\vec{x}_i}(\vec{W}) \vec{x}_i)\|  =  \|(\nabla_\vW \zeta_{\vec{x}_i}(\vec{W})) \vec{x}_i \| \le \zeta_0.$ Using lemma \ref{lipschitz_product}, we see that $2\zeta_0^2$ is a Lipschitz constant for $h_\vW(\vxi,\vxj),$ and thus also  $h^+(\vW).$ 

Now, consider a minimizer $\vW$ of the objective, and a ball $\mathcal{B}_r$ around this point of radius $r$.  Define the constant $\epsilon = \frac{\eta}{4\zeta_0^2},$ and create an $\epsilon$-net of points $\mathcal{N}_\epsilon=\{\vW_i\}$ inside the ball.  This net is sufficiently dense that any $\vW' \in\mathcal{B}_r$ is at most $\epsilon$ units away from some $\vW_i \in \mathcal{N}_\epsilon.$ Furthermore, because $h^+(\vW)$ is Lipschitz in $\vW,$  $|h^+(\vW') - h^+(\vWi) |  \le 2\zeta_0^2 \epsilon = \eta/2.$

We now know the following: if we can guarantee that 
\begin{align} \label{linearnet}
h^+(\vWi) \ge -\eta/2, \text{ for all } \vWi \in  \mathcal{N}_\epsilon,
\end{align}
then we also know that $h^+(\vW') \ge -\eta$ for all $\vW'\in \mathcal{B}_r$.  For this reason, we prove the result by bounding the probability that \eqref{linearnet} holds.  It is known that $\mathcal{N}_\epsilon$ can be constructed so that $|\mathcal{N}_\epsilon | \le (2r/\epsilon+1)^d = (8\zeta_0^2 r/\eta +1)^d$ (see \citet{vershynin2016high}, corollary 4.1.13). Theorem~\ref{thm:arbitraryNN} provides a bound on the probability that each individual point in the net satisfies condition \eqref{linearnet}.  Using a union bound, we see that all points in the net satisfy this condition with probability at least
\begin{align}
&1 - N^2\left(\frac{8\zeta_0^2 r}{\eta} +1\right)^d \exp\left(-\frac{cd(\eta/2)^2}{16\zeta_0^4}\right)\\
& =1- N^2\exp( d \log(8\zeta_0^2 r/\eta +1)) \exp\left(-\frac{cd\eta^2}{64\zeta_0^4}\right) \\
& \ge 1- N^2\exp( 8d\zeta_0^2 r/\eta) \exp\left(-\frac{cd\eta^2}{64\zeta_0^4}\right)\\
& = 1- N^2 \exp\left(-\frac{cd\eta^2}{64\zeta_0^4} + \frac{8d\zeta_0^2 r}{\eta}\right).
\end{align}

Finally, note that, if $r<\epsilon,$ then we can form a net with $|\mathcal{N}_\epsilon | =1$.  In this case, the probability of satisfying \eqref{linearnet} is at least 
 \[1 - N^2\exp\left(-\frac{cd(\eta/2)^2}{64\zeta_0^4}\right). \qedhere\]
\end{proof} 

\subsection{Proof of theorem \ref{thm:fixedData}}
\label{appsec:randomw}
\NeuralNetsFixedData*

Both parts in theorem~\ref{thm:fixedData} depend on the following argument. From theorem 2.3.8 and Proposition 2.3.10 in \citet{tao2012topics} with appropriate scaling\footnote{In particular, each entry has to be scaled by $\frac{1}{\ell}$ for matrices $\{ \vec{W}_p \}_{p \in [\beta]}$ and $\frac{1}{d}$ for the matrix $\vec{W}_0$.}, we have for every $p = 1 ,\ldots, \beta$ we have that the matrix norm $\| \vec{W}_p \| \leq 1$ with probability at least $1- \beta \exp\left( -c_1 \kappa^2 \ell^2 \right)$ and $\| \vec{W}_0 \| \leq 1$ with probability at least $1- \exp\left( -c_1 \kappa^2 d^2 \right)$ when the weight matrices are initialized according to strategy~\ref{strat:weights}. Thus, conditioning on this event it implies that these matrices satisfy assumption~\ref{ass:small_weight}. The proof strategy is similar to that of theorem~\ref{thm:arbitraryNN}. We will first show that the gradient of the function $h(., .)$ as defined in equation \eqref{eq:HfunctionNN} \emph{with respect to the weights} is bounded. Note that in part (1) the random variable is the set of weight matrices $\{ W_p \}_{p \in [\beta]}$. Thus, the dimension used to invoke theorem~\ref{lem:HDPBook} is at most $\ell^2 \beta$. In part (2) along with the weights, the data $\vec{x} \in \mathbb{R}^d$ is also random. Thus, the dimension used to invoke theorem~\ref{lem:HDPBook} is at most $\ell d + \ell^2 \beta$. Combining this with theorem~\ref{lem:HDPBook}, the bound on the gradient of $h(., .)$ and taking a union bound, we get the respective parts of the theorem. Thus, all it remains to prove is the bound on the gradient of the function $h(., .)$ as defined in equation \eqref{eq:HfunctionNN} \emph{with respect to the weights} conditioning on the event that $\| \vec{W}_p \| \leq 1$ for every $p \in \{ 0, 1, \ldots, \beta \}$.

	We obtain the following analogue of equation \eqref{eq:genNNMain}.
		\begin{multline}
			\label{eq:genNNExtension}
			(\nabla_{\vec{W}} \zeta_{\vec{x}_i}(\vec{W})) \zeta_{\vec{x}_j}(\vec{W}) \left( \sum_{p \in [\beta]_0} \Tr[ \vec{H}_{p-1}(\vec{x}_i) \cdot \vec{g}_p(\vec{x}_i)^\top \cdot \vec{g}_p(\vec{x}_j) \cdot \vec{H}_{p-1}(\vec{x}_i)^\top]\right) +\\
			(\nabla_{\vec{W}} \zeta_{\vec{x}_j}(\vec{W})) \zeta_{\vec{x}_i}(\vec{W}) \left( \sum_{p \in [\beta]_0} \Tr[ \vec{H}_{p-1}(\vec{x}_i) \cdot \vec{g}_p(\vec{x}_i)^\top \cdot \vec{g}_p(\vec{x}_j) \cdot \vec{H}_{p-1}(\vec{x}_i)^\top]\right) +\\
		\zeta_{\vec{x}_i}(\vec{W}) \zeta_{\vec{x}_j}(\vec{W}) \sum_{p \in [\beta]_0} \left[ \nabla_{\vec{W}} \left( \vec{H}_{p-1}(\vec{x}_i) \cdot \vec{g}_p(\vec{x}_i)^\top \cdot \vec{g}_p(\vec{x}_j) \cdot \vec{H}_{p-1}(\vec{x}_i) \right) \right]^\top.
		\end{multline}
	As in the case of the proof for theorem~\ref{thm:arbitraryNN}, we will upper-bound the $\ell_2$-norm of the above expression. In particular, we show the following.
	\begin{align}
		&  \Big\| (\nabla_{\vec{W}} \zeta_{\vec{x}_i}(\vec{W})) \zeta_{\vec{x}_j}(\vec{W}) \left( \sum_{p \in [\beta]_0} \Tr[ \vec{H}_{p-1}(\vec{x}_i) \cdot \vec{g}_p(\vec{x}_i)^\top \cdot \vec{g}_p(\vec{x}_j) \cdot \vec{H}_{p-1}(\vec{x}_i)^\top]\right) \Big\|_2 \leq 2\zeta_0^2(\beta+2)^2.\label{eq:randomWNNextension1} \\
		&  \Big\| (\nabla_{\vec{W}} \zeta_{\vec{x}_j}(\vec{W})) \zeta_{\vec{x}_i}(\vec{W}) \left( \sum_{p \in [\beta]_0} \Tr[ \vec{H}_{p-1}(\vec{x}_i) \cdot \vec{g}_p(\vec{x}_i)^\top \cdot \vec{g}_p(\vec{x}_j) \cdot \vec{H}_{p-1}(\vec{x}_i)^\top]\right) \Big\|_2 \leq 2\zeta_0^2(\beta+2)^2. \label{eq:randomWNNextension2} \\
		& \Big\| \zeta_{\vec{x}_i}(\vec{W}) \zeta_{\vec{x}_j}(\vec{W}) \sum_{p \in [\beta]_0} \left[ \nabla_{\vec{W}} \left( \vec{H}_{p-1}(\vec{x}_i) \cdot \vec{g}_p(\vec{x}_i)^\top \cdot \vec{g}_p(\vec{x}_j) \cdot \vec{H}_{p-1}(\vec{x}_i) \right) \right]^\top \Big\|_2 \leq 4\zeta_0^2(\beta+2)^2. \label{eq:randomWNNextension3}
	\end{align}
	
	Equations \eqref{eq:randomWNNextension1} and \ref{eq:randomWNNextension2} follow from the the fact that $\|(\nabla_{\vec{W}} \zeta_{\vec{x}_i}(\vec{W}))\|_2 \leq \zeta_0$ and the arguments in the proof for theorem~\ref{thm:arbitraryNN}. We will now show the proof sketch for equation \eqref{eq:randomWNNextension3}. For every $p \in [\beta]_0$, consider $\| \nabla_{\vec{W}} \left( \vec{H}_{p-1}(\vec{x}_i) \cdot \vec{g}_p(\vec{x}_i)^\top \cdot \vec{g}_p(\vec{x}_j) \cdot \vec{H}_{p-1}(\vec{x}_i) \right) \|$. Using the symmetry between $\vec{x}_i$ and $\vec{x}_j$, the expression can be upper-bounded by,
		\[
				2\| \nabla_{\vec{W}} \vec{H}_{p-1}(\vec{x}_i) \| \| \vec{g}_p(\vec{x}_i)^\top \| \| \vec{g}_p(\vec{x}_j) \|  \| \vec{H}_{p-1}(\vec{x}_i) \| + 2\| \vec{H}_{p-1}(\vec{x}_i) \| \| \nabla_{\vec{W}} \vec{g}_p(\vec{x}_i)^\top\| \| \vec{g}_p(\vec{x}_j) \| \| \vec{H}_{p-1}(\vec{x}_i) \|.
		\]
		As before we can use an inductive argument to find the upper-bound and thus, we obtain the following which implies equation \eqref{eq:randomWNNextension3}.
		\[
				\| \nabla_{\vec{W}} \left( \vec{H}_{p-1}(\vec{x}_i) \cdot \vec{g}_p(\vec{x}_i)^\top \cdot \vec{g}_p(\vec{x}_j) \cdot \vec{H}_{p-1}(\vec{x}_i) \right) \| \leq 4(\beta+2)^2.
		\]
	
Next, we show that the expected value can be lower-bounded by $-4$ as in the case of theorem \ref{thm:fixedData} above. Combining these two gives us the desired result. Consider $\mathbb{E}_{\vec{W}}[h_{\vec{W}}(\vec{x}_i, \vec{x}_j)]$. We compute this expectation iteratively as follows.
	\begin{align*}
		& \mathbb{E}_{\vec{W}}[h_{\vec{W}}(\vec{x}_i, \vec{x}_j)] \\ &= \mathbb{E}_{\vec{W}_0}[\mathbb{E}_{\vec{W}_1}[ \ldots \mathbb{E}_{\vec{W}_\beta}[h_{\vec{W}}(\vec{x}_i, \vec{x}_j)] \\
		&  \geq - 4 \mathbb{E}_{\vec{W}_0}\left[\mathbb{E}_{\vec{W}_1}\left[ \ldots \mathbb{E}_{\vec{W}_\beta}\left[\sum_{p \in [\beta]_0} \Tr( \vec{H}_{p-1}(\vec{x}_i) \cdot \vec{g}_p(\vec{x}_i)^\top \cdot \vec{g}_p(\vec{x}_j) \cdot \vec{H}_{p-1}(\vec{x}_i)^\top) \right] \right] \right].
	\end{align*}
	The inequality combines equation \ref{eq:HfunctionNN} with Lemma~\ref{lem:lossPropNeural}.
	We now prove the following inequality.
	\begin{equation}
		\label{eq:ineqExpradnW}
		 \mathbb{E}_{\vec{W}_0}\left[\mathbb{E}_{\vec{W}_1}\left[ \ldots \mathbb{E}_{\vec{W}_\beta}\left[\sum_{p \in [\beta]_0} \Tr( \vec{H}_{p-1}(\vec{x}_i) \cdot \vec{g}_p(\vec{x}_i)^\top \cdot \vec{g}_p(\vec{x}_j) \cdot \vec{H}_{p-1}(\vec{x}_i)^\top) \right] \right] \right] \leq 1.	
	\end{equation}
	Consider the inner-most expectation. Note that the only random variable is $\vec{W}_\beta$. Moreover, the term inside the trace is \emph{scalar}. Note that the activation function $\sigma$ satisfies $|\sigma'(x)| \leq 1$. Using the linearity of expectation, the LHS in equation \eqref{eq:ineqExpradnW} can be upper-bounded by the following.
	\begin{align}
	 & \mathbb{E}_{\vec{W}_0} \Big[ \mathbb{E}_{\vec{W}_1} \Big[ \ldots \mathbb{E}_{\vec{W}_{\beta-1}} \Big[ \Tr( \vec{H}_{\beta-1}(\vec{x}_i) \cdot \vec{H}_{\beta-1}(\vec{x}_i)^\top)\Big] \Big] \Big]\label{eq:ineqExpradnW21} \\ 
	& + \mathbb{E}_{\vec{W}_0} \Big[ \mathbb{E}_{\vec{W}_1} \Big[ \ldots \mathbb{E}_{\vec{W}_{\beta}} \Big[ \sum_{p \in [\beta]_0 \setminus \{\beta\} } \Tr( \vec{H}_{p-1}(\vec{x}_i) \cdot \vec{g}_p(\vec{x}_i)^\top \cdot \vec{g}_p(\vec{x}_j) \cdot \vec{H}_{p-1}(\vec{x}_i)^\top) \Big] \Big] \Big].\label{eq:ineqExpradnW22}
	\end{align}
	The first sum in the above expression can be upper-bounded by $1$, since $|\sigma(x)| \leq 1$. We will now show that the second sum is $0$. Consider the inner-most expectation. The weights $\vec{W}_{\beta}$ appears only in the expression $\vec{g}_p(\vec{x}_i)^\top \cdot  \vec{g}_p(\vec{x}_j)$. Moreover, note that every entry in $\vec{W}_{\beta}$ is an i.i.d. normal random variable with mean $0$. Thus, the second summand simplifies to,
	\[
			 \mathbb{E}_{\vec{W}_0} \Big[ \mathbb{E}_{\vec{W}_1} \Big[ \ldots \mathbb{E}_{\vec{W}_{\beta-1}} \Big[ \sum_{p \in [\beta]_0 \setminus \{\beta, \beta-1\} } \Tr( \vec{H}_{p-1}(\vec{x}_i) \cdot \vec{g}_p(\vec{x}_i)^\top \cdot \vec{g}_p(\vec{x}_j) \cdot \vec{H}_{p-1}(\vec{x}_i)^\top) \Big] \Big] \Big].
	\]
	Applying the above argument repeatedly we obtain that the second summand (equation \eqref{eq:ineqExpradnW22}) is $0$.
	
Thus, we obtain the inequality in equation \eqref{eq:ineqExpradnW} which implies that $\mathbb{E}_{\vec{W}}[h_{\vec{W}}(\vec{x}_i, \vec{x}_j)] \geq -4$. 

\subsection{Proof of Theorem~\ref{thm:OrthInit}}
\label{appsec:orthoinit}
    In this section, we prove Theorem~\ref{thm:OrthInit}. The proof follows similar to those in previous sub-sections; we prove a bound on the gradient of the gradient inner-product and show that the expectation is non-negative. Combining these two with an argument similar to equation \ref{linregprobbound} we get the theorem.

    Note that the dataset is obtained by considering i.i.d. samples from a $d$-dimensional unit sphere. Thus, the lower-bound on the expectation (\ie non-negative expectation of the gradient inner-product) follows from equation \ref{eq:nonNegativeIID}. Thus, it remains to prove an upper-bound on the norm of the gradient of the gradient inner-product term.
    
    Throughout this proof, we will use $g(\vec{x})$ as a short-hand to denote $g_{\vec{W}}(\vec{x})$. Consider the gradient $\nabla_{\vec{W}} g(\vec{x})$. The the $i^{th}$ component of this can be written as follows.
    \begin{equation}
        \label{eq:OrthInitGradW}
        \left[ \nabla_{\vec{W}} g(\vec{x}) \right]_i = \gamma^2 \zeta_{\vec{x}}(\vec{W}) \left( \vec{W}_{\beta}^T \cdot \ldots \vec{W}_{i+1}^T \cdot \vec{x}^T \cdot \vec{W}_1^T \cdot \ldots \vec{W}_{i-1}^T \right).
    \end{equation}
    Now consider, the gradient inner-product $h_{\vec{W}}(\vec{x}_i, \vec{x}_j)$. We want to upper-bound the quantity $\| \nabla_{(\vec{x}_i, \vec{x}_j)} h_{\vec{W}}(\vec{x}_i, \vec{x}_j) \|$. From symmetry, this can be upper-bounded by $2 \| \nabla_{\vec{x}_i} h_{\vec{W}}(\vec{x}_i, \vec{x}_j) \|$. Consider the $k^{th}$ coordinate of $\nabla_{\vec{x}_i} h_{\vec{W}}(\vec{x}_i, \vec{x}_j)$. Using equation \ref{eq:OrthInitGradW}, the assumption that $\{ \vec{W}_i \}_{i \in [\beta]}$ are  orthogonal matrices and taking the gradient, this can be written as,
    \begin{equation}
        \label{eq:OrthInitGradX}
        \left[ \nabla_{\vec{x}_i} h_{\vec{W}}(\vec{x}_i, \vec{x}_j) \right]_k = \gamma^2 \zeta_{\vec{x}_i}(\vec{W}) \vec{x}_j + \alpha^2 \left( \vec{W}_{\beta}^T \cdot \ldots \vec{W}_{i+1}^T \cdot \vec{x}^T \cdot \vec{W}_1^T \cdot \ldots \vec{W}_{i-1}^T \right)\left( \nabla_{\vec{x}_i} \zeta_{\vec{x}_i}(\vec{W}) \right).
    \end{equation}
    Combining assumption~\ref{ass:small_weight} with equation \ref{eq:OrthInitGradX} we have that $\| \nabla_{\vec{x}_i} h_{\vec{W}}(\vec{x}_i, \vec{x}_j) \|$ is at most $2 \gamma^2 \beta \| \vec{x} _j \| \leq 2 \gamma^2 \beta$. For the definition of the scaling factor $\gamma = \frac{1}{\sqrt{2 \beta}}$, we have that $2 \gamma^2 \beta = 1$. Thus, $\| \nabla_{(\vec{x}_i, \vec{x}_j)} h_{\vec{W}}(\vec{x}_i, \vec{x}_j) \| \leq 2$.

\section{Technical lemmas}
\label{app:technical_lemmas}
	We will briefly describe some technical lemmas we require in our analysis. The following Chernoff-style concentration bound is proved in Chapter 5 of \citet{vershynin2016high}.
	
	\begin{lemma}[Concentration of Lipshitz function over a sphere]
			\label{lem:HDPBook}
				Let $\vec{x} \in \mathbb{R}^d$ be sampled uniformly from the surface of a $d$-dimensional sphere. Consider a Lipshitz function $\ell: \mathbb{R}^d \rightarrow \mathbb{R}$ which is differentiable everywhere. Let $\|\nabla \ell\|_2$ denote $\sup_{\vec{x} \in \mathbb{R}^d} \|\nabla \ell(\vec{x})\|_2$. Then for any $t \geq 0$ and some fixed constant $c \geq 0$, we have the following.
					\begin{equation}
						\label{lem:conc:eq}	
\Pr\left[ \Bigl\lvert \ell(\vec{x}) - \mathbb{E}[\ell(\vec{x})] \Bigr\rvert \geq t\right] \leq 2 \exp\left(-\frac{cdt^2}{\rho^2}\right),
					\end{equation}
					where $\rho \ge \| \nabla \ell\|_2$. 
		\end{lemma}
We will rely on the following generalization of lemma \ref{lem:HDPBook}. We would like to point out that the underlying metric is the Euclidean metric and thus we use the $\| . \|_2$-norm.
		\begin{corollary}
			\label{lem:HDPCorr}
			Let $\vec{x}, \vec{y} \in \mathbb{R}^d$ be two mutually independent vectors sampled uniformly from the surface of a $d$-dimensional sphere. Consider a Lipshitz function $\ell: \mathbb{R}^d \times \mathbb{R}^d \rightarrow \mathbb{R}$ which is differentiable everywhere. Let $\|\nabla \ell\|_2$ denote $\sup_{(\vec{x}, \vec{y}) \in \mathbb{R}^d \times \mathbb{R}^d} \| \nabla \ell(\vec{x, y})\|_2$. Then for any $t \geq 0$ and some fixed constant $c \geq 0$, we have the following.
				\begin{equation}
						\label{lem:conc:eq2}	
\Pr\left[ \Bigl\lvert \ell(\vec{x}, \vec{y}) - \mathbb{E}[\ell(\vec{x}, \vec{y})] \Bigr\rvert \geq t\right] \leq 2 \exp\left(-\frac{cdt^2}{\rho^2}\right),
				\end{equation}
				where $\rho \ge \| \nabla \ell\|_2$.
		\end{corollary}
		
			\begin{proof}
			This corollary can be derived from lemma~\ref{lem:HDPBook} as follows. Note that for every fixed $\tilde{\vec{y}} \in \mathbb{R}^d$, equation \ref{lem:conc:eq} holds. Additionally, we have that the vectors $\vec{x}$ and $\vec{y}$ are mutually independent. Hence we can write the LHS of equation \ref{lem:conc:eq2} as the following.
			\[
			 	\bigintsss_{(\tilde{\vec{y}})_1=-\infty}^{(\tilde{\vec{y}})_1=\infty} \ldots \bigintsss_{(\tilde{\vec{y}})_d=-\infty}^{(\tilde{\vec{y}})_d=\infty} \Pr\left[ \Bigl\lvert \ell(\vec{x}, \vec{y}) - \mathbb{E}[\ell(\vec{x}, \vec{y})] \Bigr\rvert \geq t ~ \Biggl\lvert~\vec{y}=\tilde{\vec{y}} \Biggr\rvert \right] \phi(\tilde{\vec{y}}) d(\tilde{\vec{y}})_1 \ldots d(\tilde{\vec{y}})_d.
			 \]
			 Here $\phi(\tilde{\vec{y}})$ refers to the pdf of the distribution of $\vec{y}$. From independence, the inner term in the integral evaluates to $\Pr\left[ \Bigl\lvert \ell(\vec{x}, \tilde{\vec{y}}) - \mathbb{E}[\ell(\vec{x}, \tilde{\vec{y}})] \Bigr\rvert \geq t \right]$. We know this is less than or equal to $ 2 \exp\left(-\frac{cdt^2}{\|\vec{\nabla} \ell\|_2^2}\right)$. Therefore, the integral can be upper bounded by the following.
			 \[
			 	\bigintsss_{(\tilde{\vec{y}})_1=-\infty}^{(\tilde{\vec{y}})_1=\infty} \ldots \bigintsss_{(\tilde{\vec{y}})_d=-\infty}^{(\tilde{\vec{y}})_d=\infty}2 \exp\left(-\frac{cdt^2}{\|\vec{\nabla} \ell\|_2^2}\right) \phi(\tilde{\vec{y}}) d(\tilde{\vec{y}})_1 \ldots d(\tilde{\vec{y}})_d.
			 \]
			 
			 Since $\phi(\tilde{\vec{y}})$ is a valid pdf, we get the required equation \ref{lem:conc:eq2}.
		\end{proof}

		Additionally, we will use the following facts about a normalized Gaussian random variable.
	\begin{lemma}
			\label{lem:gauss_supp}
			For a normalized Gaussian $\vec{x}$ (\emph{i.e.,} an $\vec{x}$ sampled uniformly from the surface of a unit $d$-dimensional sphere)  the following statements are true.
		\begin{enumerate}
			\item $\forall p \in [d]$ we have that $\mathbb{E}[(\vec{x})_p] = 0$.
			\item $\forall p \in [d]$ we have that $\mathbb{E}[(\vec{x})_p^2] = 1/d$.
		\end{enumerate}
		\end{lemma}
	\begin{proof}
			Part (1) can be proved by observing that the normalized Gaussian random variable is spherically symmetric about the origin. In other words, for every $p \in [d]$ the vectors $(x_1, x_2, \ldots, x_p, \ldots, x_d)$ and $(x_1, x_2, \ldots, -x_p, \ldots, x_d)$ are identically distributed. Hence $\mathbb{E}[x_p] = \mathbb{E}[-x_p]$ which implies that $\mathbb{E}[x_p]=0$.
			
			Part (2) can be proved by observing that for any $p, p' \in [d]$, $x_p$ and $x_{p'}$ are identically distributed. Fix any $p \in [d]$. We have that $\sum_{p' \in [d]} \mathbb{E}[x_{p'}^2] = d \times \mathbb{E}[x^2_p]$. Note that we have
			
				\[
						\sum_{p' \in [d]} \mathbb{E}[x_{p'}^2] = \bigintsss_{(\vec{x})_1=-\infty}^{(\vec{x})_1=\infty} \ldots \bigintsss_{(\vec{x})_d=-\infty}^{(\vec{x})_d=\infty}  \frac{\sum_{p' \in [d]} x_{p'}^2}{\sum_{p'' \in [d]} x^2_{p''}}
						\phi(\vec{x}) d(\vec{x})_1 \ldots d(\vec{x})_d = 1.
			 \]
			Therefore $\mathbb{E}[x^2_p] = 1/d$.
		\end{proof}

	We use the following well-known Gaussian concentration inequality in our proofs (\emph{e.g.,} Chapter 5 in \citet{boucheron2013concentration}).
	
	\begin{lemma}[Gaussian Concentration]
		\label{lem:gaussianConcenteration}
		Let $\vec{x} = (x_1, x_2, \ldots, x_d)$ be \emph{i.i.d.} $\mathcal{N}(0, \nu^2)$ random variables. 
		Consider a Lipshitz function $\ell: \mathbb{R}^d \rightarrow \mathbb{R}$ which is differentiable everywhere. Let $\|\nabla \ell\|_2$ denote $\sup_{\vec{x} \in \mathbb{R}^d} \|\nabla \ell(\vec{x})\|_2$. Then for any $t \geq 0$, we have the following.
					\begin{equation}
						\label{lem:concGaussian:eq}	
	\Pr\left[ \Bigl\lvert \ell(\vec{x}) - \mathbb{E}[\ell(\vec{x})] \Bigr\rvert \geq t\right] \leq 2 \exp\left(-\frac{t^2}{2 \nu^2 \rho^2}\right),
					\end{equation}
					where $\rho \ge \| \nabla \ell\|_2$. 
	\end{lemma}

\section{Additional discussion of the small weights assumption (assumption \ref{ass:small_weight})}
\label{app:small_weights}

 Without the small-weights assumption, the signal propagated forward or the gradients  $\nabla_{\vec{W}} f_i$ could potentially blow up in magnitude, making the network untrainable. Proving non-vacuous bounds in case of such blow-ups in magnitude of the signal or the gradient is not possible in general, and thus, we assume this restricted class of weights. 
 	

Note that the small-weights assumption is not just a theoretical concern, but also usually holds in practice. Neural networks are often trained with \emph{weight decay} regularizers of the form $\sum_i \|W_i\|_F^2$, which keep the weights small during optimization.
The operator norm of convolutional layers have also recently been used as an effective regularizer for image classification tasks by \citet{sedghi2018singular}. 
		
In the proof of theorem~\ref{thm:fixedData} we showed that assumption~\ref{ass:small_weight} holds with high probability at standard Gaussian initializations used in practice. While, in general, there is no reason to believe that such a small-weights assumption would continue to hold during optimization without explicit regularizers like weight decay, some recent work has shown evidence that the weights do not move too far away during training from the random initialization point for overparameterized neural networks \citep{neyshabur2018towards, dziugaite2017computing, nagarajan2019generalization, zou2018stochastic, allen2018convergence, du2018gradient, oymak2018overparameterized}. It is worth noting though that all these results have been shown under some restrictive assumptions, such as the width requiring to be much larger than generally used by practitioners.

\end{document}